\documentclass{article}
\usepackage[final]{neurips_2024}
\usepackage{enumitem}
\usepackage{lipsum}
\usepackage{subcaption}
\usepackage{url}
\usepackage[utf8]{inputenc} 
\usepackage[T1]{fontenc}    

\usepackage{booktabs}       
\usepackage{amsfonts}       
\usepackage{nicefrac}       
\usepackage{microtype}      
\usepackage{xcolor,colortbl} 
\usepackage{algorithm}
\usepackage{algorithmic}
\usepackage{hhline}
\usepackage{smile}
\usepackage{multicol}
\usepackage{multirow}
\usepackage{graphicx}
\usepackage{wrapfig}
\usepackage[compact]{titlesec}
\usepackage{sidecap}
\usepackage{setspace}
\allowdisplaybreaks

\theoremstyle{plain}
\def\BB{\textcolor{black}}
\ifdefined\final
\usepackage[disable]{todonotes}
\else
\usepackage[textsize=tiny]{todonotes}
\fi
\usepackage{colortbl}
\definecolor{LightCyan}{rgb}{0.8, 0.9, 1}

\newcommand{\method}{DNDM}
\def\blue#1{\textcolor{black}{#1}}
\usepackage{microtype}
\usepackage{graphicx}
\usepackage{booktabs} 
\def \CC {}

\usepackage{hyperref}

\title{
Fast Sampling via Discrete Non-Markov Diffusion Models with Predetermined Transition Time
}

\usepackage{amsmath}
\usepackage{amssymb}
\usepackage{mathtools}
\usepackage{amsthm}

\usepackage[capitalize,noabbrev]{cleveref}

\theoremstyle{plain}
\theoremstyle{definition}
\theoremstyle{remark}

\usepackage[textsize=tiny]{todonotes}

\author{%
 Zixiang Chen~~Huizhuo Yuan~~Yongqian Li~~Yiwen Kou~~Junkai Zhang ~~Quanquan Gu\\
 Department of Computer Science\\ University of California, Los Angeles \\
Los Angeles, CA 90095 \\
\texttt{\{chenzx19,hzyuan,yongqianl,evankou,jkzhang,qgu\}@cs.ucla.edu} 
}

\begin{document}

\maketitle

\begin{abstract}
Discrete diffusion models have emerged as powerful tools for high-quality data generation. Despite their success in discrete spaces, such as text generation tasks, the acceleration of discrete diffusion models remains under-explored. In this paper, we propose discrete non-Markov diffusion models (DNDM), which naturally induce the predetermined transition time set. This enables a training-free sampling algorithm that significantly reduces the number of function evaluations (i.e., calls to the neural network), making the sampling process much faster. Furthermore, we study the transition from finite to infinite step sampling, offering new insights into bridging the gap between discrete and continuous-time processes for discrete diffusion models. Extensive experiments on natural language generation and machine translation tasks demonstrate the superior performance of our method in terms of both generation speed and sample quality compared to existing methods for discrete diffusion models. Codes are available at \url{https://github.com/uclaml/DNDM}.

\end{abstract}

\section{Introduction}
Diffusion-based generative models, as first introduced by~\citet{sohl2015deep}, have shown remarkable capabilities in generating high-quality samples across various domains, including images~\citep{ho2020denoising, song2020improved}, audio \citep{chen2020wavegrad, kong2020diffwave}, and videos \citep{ho2022imagen}.
The diffusion model utilizes an innovative approach comprising a forward process that gradually transforms training data into pure noise and a reverse process that reconstructs clean data from the noise. Throughout the training phase, the model optimizes a neural network by minimizing an objective derived from maximum likelihood estimation. Once trained, the model can generate samples using various decoding strategies, including implicit dynamics \citep{song2020denoising}, analytical processes \citep{bao2022analytic}, or differential equation solvers \citep{song2020score, liu2022pseudo, lu2022dpm}. In particular, \citet{song2020denoising} introduced the denoising diffusion implicit model (DDIM), providing a non-Markov and de-randomized version of the Denoising Diffusion Probabilistic Model (DDPM)~\citep{sohl2015deep, ho2020denoising}, which enables faster generation of high-quality samples.

Although diffusion models were initially introduced for both discrete and continuous-state spaces \citep{sohl2015deep}, these studies have largely focused on Gaussian diffusion processes in continuous-state spaces. Recently, Discrete Denoising Diffusion Probabilistic Models (D3PMs) \citep{austin2021structured} working in discrete-state spaces have gained increasing interest due to their applications in diverse areas such as text generation \citep{hoogeboom2021argmax}, medical record generation \citep{ceritli2023synthesizing}, and protein design \citep{gruver2024protein}. These models, which are distinct from their Gaussian counterparts, employ discrete noises, such as the multinomial distribution, for diffusion processes. Very recently, \cite{zheng2023reparameterized} introduced a reparameterized diffusion model (RDM) that can improve sampling speed and sample quality in text generation tasks. However, their proposed algorithm is a training-based approach. Compared with diffusion models using Gaussian noise, discrete diffusion models remain under-studied, especially regarding training-free sampling acceleration.

In this work, we introduce a training-free approach aiming at enhancing the sampling speed of discrete diffusion models. This approach stems from a unique characteristic of discrete diffusion models: unlike continuous diffusion models, which typically employ Gaussian noise for data corruption~\citep{ho2020denoising, song2020improved,song2020score,song2020denoising}, discrete diffusion models often use categorical white noises~\citep{hoogeboom2021argmax, austin2021structured, zheng2023reparameterized}. 

\begin{wraptable}{t}{0.7\textwidth}
\vspace{-10pt} 
\centering
\caption{Cross Comparison of Diffusion Models.}
\begin{tabular}{c|c|c}
\toprule
 & \textbf{Continuous} & \textbf{Discrete}\\ \midrule
\multirow{2}{*}{\textbf{Markov}} & DDPM & D3PM \\
& \citep{sohl2015deep} & \cite{austin2021structured}  \\ \midrule
\multirow{2}{*}{\textbf{Non-Markov}} & DDIM & \cellcolor{cyan!20}DNDM \\
& \citep{song2020denoising} & \cellcolor{cyan!20}(Ours) \\
\bottomrule
\end{tabular}
\label{tab:comparison}
 \vspace{-10pt} 
\end{wraptable} 
By delving into this special property, we develop a discrete non-Markov diffusion model, together with a design of accelerated algorithm.  Notably, this new sampling technique does not require any modifications to the training objective of diffusion models and is, therefore, training-free. Our contributions are summarized as follows:

\begin{itemize}[leftmargin=*,nosep]
\item We propose discrete non-Markov diffusion models (DNDM), which naturally induces a set of latent variables $\mathcal{T}$, termed as the \textit{transition time set}. This key feature enables us to develop a training-free sampling algorithm that can accelerate a large family of discrete diffusion models. Importantly, DNDM preserves the essential properties of the original discrete diffusion model: for any diffusion trajectory $\{\xb_t\}$ starting from real data $\xb_0$, it provably maintains both the marginal distribution $q(\xb_t)$ and the conditional distribution $q(\xb_0|\xb_t)$. Our method can accelerate the two most widely used discrete diffusion models: multinomial diffusion~\citep{hoogeboom2021argmax} and absorbing diffusions~\citep{austin2021structured}. Similar to how DDIM introduces a de-randomized, faster sampling algorithm compared to DDPM in continuous space, DNDM achieves acceleration through a predetermined transition time set in discrete space (See Table~\ref{tab:comparison}).

\item Based on the predetermined transition time set $\cT$ in DNDM, we design an accelerated sampling algorithm that reduces the required number of neural network function evaluations. In a standard $T$ time-step discrete diffusion process, while D3PM, including Multinomial \citep{ho2020denoising} and absorbing state discrete sampling \citep{austin2021structured}, requires evaluating the neural network function $T$ times, our approach only requires $|\cT|$ function evaluations, where $|\cT|$ is the cardinality of the transition set $\cT$. Moreover, $|\cT|$ is provably less than $T$ and approaches $O(1)$ as $T$ goes to infinity.
We provide both theoretical analysis and empirical experiments showing that the improvement in the number of function evaluations (NFE) is significant. Notably, our algorithm is about $3\times$ faster than baselines for $T = 50$ and about $30\times$ faster for $T = 1000$ while preserving the sample quality.

\item To further illustrate the effectiveness of DNDM, we explore the limit as $T\rightarrow \infty$ and introduce an infinite-step sampling algorithm.
With a pretrained neural network, we can generate an initial noise $\xb_{T}$ and a transition time set $\cT\subseteq [0,1]$ with infinitesimal spacing, such that $|\cT| = O(1)$. This enables the generation of the real data distribution with only $|\cT|$ neural network evaluations. This study offers new insights into bridging the gap between discrete and continuous-time processes for discrete diffusion models. 
\end{itemize}

\noindent\textbf{Notation.}
We use $|\cT|$ to denote the cardinality of the set $\cT$ (excluding repeated elements). We use lowercase letters to denote scalars, boldface lowercase letters to denote vectors, and boldface uppercase letters to denote matrices. The notation $1:N$ indicates the sequence from $1$ through $N$. The symbol $\qb$ designates the real distribution in a diffusion process, while $\pb$ represents the distribution during sampling. With its success probability inside the parentheses, the Bernoulli distribution is denoted by $\mathrm{Bernoulli}(\cdot)$. We further use $\mathrm{Cat}(\xb;\pb)$ to denote a categorical distribution over a one-hot row vector $\xb$ with probabilities given by the row vector $\pb$.

\section{Background}\label{sec:markov}
In this section, we provide the background of discrete diffusion models. We begin by introducing the discrete Markov diffusion model, designed for handling categorical random variables. Specifically, consider a diffusion model trying to generate distributions over a discrete random variable $\mathbf{x} \in \mathbb{R}^{K}$ that is one-hot encoded with $K$ categories, i.e., $\xb$ can be chosen as one of $K$ categories, and for any $k \in [K]$, $\xb$ is categorized as $k$ if $\xb$ aligns with the standard basis vector $\eb_k$. The sequence $\{\xb_t\}_{t=0}^{T}$ represents how this random variable changes over time $0\leq t \leq T$, starting from an $\xb_0 \in \mathbb{R}^{K}$ drawn from the real distribution $\mathbf{q}_{\mathrm{data}}$. In this paper, we focus on the two most widely used D3PMs: multinomial diffusion \citep{hoogeboom2021argmax} and absorbing diffusions \citep{austin2021structured}.


\noindent\textbf{Forward Process.} During the forward process, the real distribution $\qb_{\mathrm{data}}$ is gradually transformed into a noise distribution named $\qb_{\mathrm{noise}}$. The transformation occurs through $T$ steps, with $T$ intermediate latent variables $\xb_{1}, \ldots \xb_{T}$ and update rules given by:
\begin{align}
\xb_{t} = b_{t}\xb_{t-1} + (1-b_{t})\wb_{t}, \qquad t=1, \ldots, T \label{eq:1} 
\end{align}
Here $b_t$ is randomly drawn from a Bernoulli distribution with parameter $\beta_t$, denoted by $b_t \sim \mathrm{Bernoulli}(\beta_t)$, and
$\wb_{t}$ is randomly drawn from the noise distribution $\qb_{\mathrm{noise}}$, while for different $t$ the samples are independent. In this work, we focus on cases where the noise $\qb_{\mathrm{noise}}$ can be either a uniform distribution over the vocabulary $\{1,2, \ldots, K\}$ \citep{hoogeboom2021argmax}, or a point mass with all of the probability mass lying on an absorbing state \citep{austin2021structured}. 
Following this notation, the process in~\eqref{eq:1} defines a Markov process characterized by the transition kernel
\begin{align}
q(\xb_{t}|\xb_{t-1}) = \mathrm{Cat}\big(\xb_{t}; \pb = \beta_{t}\xb_{t-1} + (1- \beta_{t})\qb_{\mathrm{noise}}\big).  
\end{align}
Moreover, the Markov chain property allows us to get samples $\xb_{0:t}$ from $\xb_0$ by multiplying 
 the transition probabilities at each step as $p(\xb_{1:t} | \xb_0) = \prod_{i=1}^t q(\xb_t | \xb_{t-1})$. It further leads to the following marginal distribution.
\begin{align}
q(\xb_{t}|\xb_{0}) = \mathrm{Cat}\big( \xb_t ; \pb = \alpha_{t}\xb_{0} + (1- \alpha_{t})\qb_{\mathrm{noise}}\big), \label{eq:markovt0}
\end{align}
where $\alpha_t: = \Pi_{s=1}^{t}\beta_s$ is determined by the sequence of $\beta_t$ of our choice and decreases from $1$ to $0$.

\noindent\textbf{Reverse Process.} 
Given the forward Markov process, the reverse process can be derived by Bayes' rule \citep{hoogeboom2021argmax,austin2021structured,zheng2023reparameterized}. The conditional probability $q(\xb_{t-1}|\xb_{0}, \xb_{t})$ can be determined by $q(\xb_{t-1}|\xb_{0}, \xb_{t}) = q(\xb_{t}|\xb_{t-1})q(\xb_{t-1}|\xb_0)/q(\xb_t|\xb_0)$. The reverse process can be used for synthetic data generation by sampling from the noise distribution $q_{\mathrm{noise}}$ and repeatedly applying a learned predictor (neural network)  $p_{\btheta}(\cdot|\xb_t)$ parameterized by $\btheta$:
\begin{align}
p_{\btheta}(\xb_{T}) = q_{\mathrm{noise}}(\xb_{T}), \qquad q_{\btheta}(\xb_{t-1}|\xb_t) = \int_{\hat{\xb}_0}q(\xb_{t-1} | \xb_t, \hat{\xb}_0) p_{\btheta}(\hat{\xb}_0|\xb_t)d \hat{\xb}_0. \label{eq:generation}
\end{align}
We note that the reverse process $q(\xb_{t-1} | \xb_t, \hat{\xb}_0)$ is stochastic and thus requires function evaluation at every step.

\noindent\textbf{Training the Neural Network.} The neural network $p_{\btheta}(\cdot|\xb_t)$ that predicts $\hat{\xb}_0$ is trained by  maximizing the evidence lower bound (ELBO) \citep{sohl2015deep}, 
\begin{align}
\log p_\theta(\xb_0) &\geq  \mathbb{E}_{q(\xb_{1: T} | \xb_0)}\Big[\log \frac{p_{\btheta}(\xb_{0: T})}{q(\xb_{1: T} | \xb_0)}\Big] d \xb_{1: T} \notag \\
&=  \mathbb{E}_{q(\xb_1 | \xb_0)}[\log p_{\btheta}(\xb_0 | \xb_1)]  - \sum_{t=2}^{T}\mathbb{E}_{q(\xb_t | \xb_0)}[\mathrm{KL}(q(\xb_{t-1} | \xb_t, \xb_0) \| p_{\btheta}(\xb_{t-1} | \xb_t)) \notag\\
&\qquad - \EE_{q(\xb_T|\xb_0)}\mathrm{KL}(q(\xb_T | \xb_0) \| p_{\btheta}(\xb_T)), \label{eq:ELBOmain}
\end{align}
Here KL denotes Kullback-Liebler divergence and the last term $\EE_{q(\xb_T|\xb_0)}\mathrm{KL}(q(\xb_T | \xb_0) \| q_{\mathrm{noise}}(\xb_{T}))$ equals zero. Building on this foundation, \citet{austin2021structured} introduced an auxiliary denoising objective, which refines the data predictions $\xb_0$ at each time step. Since this paper primarily focuses on reverse sampling, we leave detailed discussions of these losses to Appendix~\ref{sec:morebackground}.

\section{Discrete Non-Markov Diffusion Models (DNDM)} 

\subsection{Forward and Reverse Process}\label{sec: De}
In this section, we introduce a non-Markov process such that the joint distribution of $(\xb_{0}, \xb_{t})$ remains the same as the one defined with Markov process in Section~\ref{sec:markov}. The new process aims to gradually transform input data $\qb_{\mathrm{data}}$ to the noise distribution $\qb_{\mathrm{noise}}$ through $T$ intermediate latent variables $\xb_{1}, \ldots \xb_{T}$ with the following process:
\begin{align}
\xb_{t} = b_{t}\xb_{t-1} + (1-b_{t})\wb, \label{eq:2}
\end{align}
where $b_{t}$ is independently drawn from the Bernoulli  distribution $\mathrm{Bernoulli }(\beta_t)$ and $\wb$ is drawn from the noise distribution $\qb_{\mathrm{noise}}$. The only difference between~\eqref{eq:2} and \eqref{eq:1} is that we replace $\wb_t$ in~\eqref{eq:1} by $\wb$, which is time-invariant during the diffusion. Therefore, the process in~\eqref{eq:2} becomes non-Markov since $q(\xb_{t}|\xb_{t-1},\ldots,\xb_{0})$ doesn't necessarily equals $q(\xb_{t}|\xb_{t-1})$. The following theorem shows that the conditional distribution $q(\xb_{t}|\xb_{0})$ remains unchanged.
\begin{theorem}\label{thm:same distribution}
For the non-Markov process in~\eqref{eq:2}, we have 
\begin{align*}
q(\xb_{t}|\xb_{0}) = \mathrm{Cat}\big(\xb_t; \pb = \alpha_{t}\xb_{0} + (1- \alpha_{t})\qb_{\mathrm{noise}}\big),   
\end{align*}
where $\alpha_t: = \Pi_{i=1}^{s}\beta_s$ is specified to decrease from $1$ to $0$. 
\end{theorem}

Using the Bayes' rule, we have $q(\xb_{0}|\xb_t) \propto q(\xb_{t}| \xb_{0})q(\xb_{0})$. Consequently, the condtional distribution $q(\xb_{0}|\xb_t)$ remains consistent with the one induced by the process process in~\eqref{eq:1}. Therefore, neural network $p_{\btheta}(\cdot|\xb_{t})$ trained by the Markov process in~\eqref{eq:1}, remains applicable to our non-Markov process~\eqref{eq:2} (see Appendix~\ref{sec:morebackground} for detail). 

Based on the discrete non-Markov diffusion model, we can give a simple characterization of the reverse process by introducing the transition time.  
\begin{definition}\label{Def:transition time}
Transition time $\tau$ is the time that the token $\xb_{t}$ transition from $\xb_0$ to noise, i.e., $\tau:= \min_{t}\{t|b_{t} = 0\}$. 
\end{definition}
{\color{black}
\begin{remark}
    The concept of transition time has also been introduced in \citet{hoogeboom2021autoregressive}. However, \citet{hoogeboom2021autoregressive} restricts the transition time to be the first time of entering the absorbing state, which is only applicable to absorbing diffusion. Our definition is more general and applicable to discrete diffusion with various noise including multinomial diffusion.
\end{remark}
}
Given the transition time $\tau$, the forward process reduces to:
\begin{align}
\xb_{t} = \ind(\tau>t)\xb_{0} + \ind(\tau\leq t)\wb,\label{eq:deterministic}
\end{align}
which shows that the token will be a real token $\xb_{0}$ before the time $\tau$ and will be the noise $\wb$ after the transition time. Since token only get changed at the transition time $\tau$, we can derive a reverse process based on~\eqref{eq:deterministic},
\begin{align}
\xb_{t-1} = \ind(\tau=t)\xb_{0} + \ind(\tau\not= t)\xb_{t}.\label{eq:version1}
\end{align}
Therefore, the process in \eqref{eq:version1} is de-randomized given transition time $\tau$. Specifically, after independently sampled transition times $\tau$,
$\xb_{t-1}$ becomes deterministically known and fixed if we observe $\xb_{0}$ and $\xb_{t}$.
It is also worth noting that given $\xb_0$ and $\tau$, the exact reverse process \eqref{eq:version1} is Markovian, since $\xb_{t-1}$ solely depends on  $\xb_0, \tau, \xb_t$. Plugging \eqref{eq:version1} into \eqref{eq:generation} gives the generation process. We can prove the ELBO of the DNDM is equivalent to the ELBO of the original process \eqref{eq:ELBOmain} up to some constant, which further supports the neural network $p_{\btheta}(\cdot|\xb_{t})$ trained by the Markov process in~\eqref{eq:1}, remains applicable to \method{}. (See Appendix~\ref{subsubsec: finite time ELBO} for details).
\begin{remark}
\eqref{eq:deterministic} and \eqref{eq:version1} suggest that even though there are $T$ distinct time steps, not every time in the range $1:T$ is crucial for capturing the process. Therefore, our primary focus should be on the most significant time step, i.e., the transition time $\tau$, enabling faster reverse sampling. We further note that although transition happens only at time $\tau$, the transition time is random, differs across runs, and covers the full range from $1$ to $T$ on average.
\end{remark}

\begin{remark}
While \citet{song2020denoising} proposed a non-Markov multinomial diffusion model in Appendix A, DDIM and DNDM are fundamentally different models when specialized to multinomial diffusion. DDIM's discrete process remains stochastic at every step, even with deterministic noise scheduling. In contrast, DNDM achieves full de-randomization by pre-determined transition time $\tau$ (Equation 8 in our paper). By sampling these transition times upfront, DNDM establishes a predetermined transition time set that guides the sampling process, enabling deterministic evolution and faster sampling speed even under the same number of sampling steps, which is not reported under DDIM framework. For detailed technical comparison, see Appendix~\ref{app:ddim}.
\end{remark}

\subsection{Accelerated Reverse Sampling}\label{sec:fast}
In this section,  we demonstrate that sampling from \method{} can lead to accelerated reverse sampling. Although our algorithm is quite general, we focus on text generation in the presentation.

In Section~\ref{sec: De}, we only consider the case of a single token $\xb \in \RR^{K}$ being one hot encoding of $K$ categories. In real applications, we are interested in generating a sentence with multiple tokens. So, we extend the terminology in Section~\ref{sec: De}, and we denote the sequence of tokens at $t$-th time step to be $\xb_{t, 1:N} = [\xb_{t,1}, \ldots, \xb_{t,N}]$ where $\xb_{t,n}$ is the $n$-th token and $N$ is the sequence length. The noise will be added to each token in a sequence independently. Therefore, each token will have its own transition time defined in Definition~\ref{Def:transition time}. We denote the transition time for each token $\xb_{n}$ to be $\tau_{n}$ and further denote the transition time set $\cT := \{\tau_{n}\}_{n=1}^{N}$. Given the transition times $\tau_{n} \in \cT$, our \method{} can now be extended to the sequence with multiple tokens
\begin{align}
\xb_{t-1,n} = \ind(\tau_n=t)\xb_{0,n} + \ind(\tau_{n}\not= t)\xb_{t,n}, \forall n\in [N]\label{eq:multi-token version1}.
\end{align}
\blue{
\textbf{Learning the Reverse Process.}  We first generate the transition times $\tau_{n}$ for $n \in [N]$, then we follow \eqref{eq:multi-token version1} to generate the learned reverse process. Since $\xb_{0,n}$ is unknown in the process, we use the neural network evaluation $p_{\btheta}(\cdot| \xb_t)$ obtained in Section~\ref{sec: De} to predict $\xb_{0,n}$.} In detail, the noisy sequence $\xb_{t,1:N}$ is fed into $p_{\btheta}(\cdot| \xb_{t,1:N})$ and
the prediction tokens $\hat{\xb}_{0, 1:N} \sim p_{\btheta}(\cdot| \xb_{t, 1:N})$ are collected.

\noindent\textbf{Transition time.} Transition time, denoted by $\tau$, is crucial in our reverse process. This is because the reverse sampling becomes deterministic upon using~\eqref{eq:multi-token version1}. Each instance of transition time $\tau$ is a random variable within the set $\{1, 2, \ldots, T\}$. Let's assume it follows the distribution $\cD_{\tau}$. Given the schedule $\{\alpha_{t}\}_{t=0}^{T}$, we can derive the distribution for $\cD_{\tau}$.

\begin{theorem}\label{thm:transition probability}
Each specific transition time $\tau_n$ in Definition~\ref{Def:transition time} is independent. Furthermore, they collectively adhere to the distribution $\cD_{\tau}$, which obeys the rule $\mathbb{P}(\tau_n = t) = \alpha_{t-1} - \alpha_{t}$. 
\end{theorem}

From Theorem~\ref{thm:transition probability}, we discern that the nature of the diffusion model scheduler, ${\alpha_t}$, clarifies the distribution of $\tau$. Take the linear schedule as an example, as given by \citet{austin2021structured}, the relationship is $\alpha_t = 1 - t/T$. This translates to $\mathbb{P}(\tau_n = t) = 1/T$ for every $t$ in the range $1$ to $T$. As a result, transition time distributes uniformly across each moment in the set $\{1, \ldots, T\}$. Generally, if we express $\alpha_t$ as $g(t/T)$, then we can simplify to $\mathbb{P}(\tau_n = t) = g((t-1)/T) - g(t/T)$, which further refines to $(1/T)|g'(t/T)| + o(1/T)$. This indicates that transitions are more likely where $|g'|$ is large. 

\begin{wrapfigure}{R}{0.55\textwidth}
\vspace{-20pt}
\begin{minipage}{\linewidth}
\begin{algorithm}[H]
\caption{Sampling From DNDM}
\begin{algorithmic}[1]\label{alg:DNDMv1}
\REQUIRE Trained prediction function $p_{\btheta}$, $\qb_{\mathrm{noise}}$, $\cD_{\tau}$
\FOR{$n = 1 \ldots N$}
\STATE Initiate each token $\xb_{T,n} \sim \qb_{\mathrm{noise}}$
\STATE Initiate the transition time $\tau_{n} \sim \cD_{\tau}$
\ENDFOR
\STATE Collect transition time set $\cT = \{\tau_{n}\}_{n=1}^{N}$
\FOR{$t = T \ldots 1$}
\IF{$t \in \cT$}
\STATE Generate $\tilde{\xb}_{0,1:N}$ from $p_{\btheta}(\cdot |\xb_{t, 1:N})$
\FOR{$n = 1 \ldots N$}
\STATE Update $\xb_{t-1,n}$ based on condition of $\tau_{n}$
\ENDFOR
\ELSE
\STATE Update $\xb_{t-1,1:N} = \xb_{t,1:N}$
\ENDIF
\ENDFOR
\STATE \textbf{Return} $\xb_{0,1:N}$
\end{algorithmic}
\end{algorithm}
\end{minipage}
\vspace{-10pt}
  \end{wrapfigure}

In practice, we observed that the shape of the transition time does not need to exactly match the theoretically predicted schedule $\mathcal{D}{\tau}$ in Theorem~\ref{thm:transition probability}. Algorithm~\ref{alg:DNDMv1} works even if $\mathcal{D}{\tau}$ is unknown. In particular, we can approximate the schedule with a Beta distribution by first sampling a time $t \in [0,1]$ from a Beta distribution, then adjusting these samples to fit by multiplying by $T$ and rounding the result to obtain an integer.

\noindent\textbf{Accelerated Sampling.}
According to \eqref{eq:multi-token version1}, a token $\xb_{t-1, n}$ is updated only if step $t$ is the transition time for the $n$-th token. If step $t$ is not the transition time for any token, the sentence from the previous step can be directly copied: $\xb_{t-1, 1:N} = \xb_{t, 1:N}$. As a result, there is no need to do a function evaluation for the current step. Our attention, therefore, can be solely centered on the transition set $\cT$, necessitating function evaluations only for $t$ within $\cT$. For our method, when $N$ is fixed while $T \rightarrow \infty$, the total NFE $|\cT|$ will reach $N$.  On the other hand, when $T$ is fixed and $N \rightarrow \infty$, the NFE $\cT$ will reach $T$ (See Theorem~\ref{thm:1} for detail). It is worth noting that the auto-regressive diffusion model (ARDM) \citep{hoogeboom2021autoregressive} can also achieve at most $N$ NFE when $T = \infty$. However, ARDM only focuses on infinite time steps, while our method here is able to accelerate sampling for finite time steps. More detailed discussion and theoretical analysis can be found in Section \ref{sec:nfe}, where additional experiments also demonstrate that our DNDM achieves an NFE that is less than half of the original Markov sampling method for discrete diffusion.

By incorporating the forward process with different noises, we can develop DNDM-Multi and DNDM-Absorb, which accelerate the Multinomial and Absorbing sampling methods respectively. Recent works have demonstrated that the quality of samples can be enhanced by utilizing supplementary information derived from the neural network, \citep{ghazvininejad2019mask, savinov2021step, chang2022maskgit, he2022diffusionbert, zheng2023reparameterized}. Our \method{} can also be improved using this idea. We call it a discrete non-Markov Diffusion Model with Top-k Transition Time (\method-$k$). Due to the limit of the pages, we leave the detailed Algorithm and discussion to Appendix~\ref{Appendix: Top}.

\subsection{Continous-time (Infinite Step) Reverse Sampling}\label{sec:continuous}
In the context of continuous state spaces, continuous-time processes have been proposed to accommodate algorithms that offer faster sampling speeds and enhanced sample quality~\citep{jolicoeur2021gotta, zhang2022fast, salimans2022progressive, chung2022come, song2020score, dockhorn2021score}. However, the application of continuous-time schemes to discrete-state spaces remains largely unexplored. \citet{campbell2022continuous} first developed a continuous framework for discrete-time diffusion for the Markovian process and randomized sampling, but not in our non-Markovian setting. In this section, we investigate the transition from finite to infinite step sampling, providing new insights into bridging the gap between discrete and continuous-time processes for discrete diffusion models.

\noindent\textbf{Continuous-time Forward and Backward process.}
Recall that the forward process described in~\eqref{eq:2} can be sampled from $\mathbf{x}_{0, n}$ through the following process:
\begin{align}\label{eq:zeroton}
\mathbf{x}_{t, n} = \alpha_t \mathbf{x}_{0, n} + (1 - \alpha_t) \mathbf{q}_{\text{noise}},\quad \alpha_t = \prod_{i=1}^t \beta_i.
\end{align}

\begin{wrapfigure}{R}{0.55\textwidth}
\vspace{-25pt}
\begin{minipage}{\linewidth}
\begin{algorithm}[H]
    \caption{Sampling from DNDM-C}
    \begin{algorithmic}[1]\label{alg:continuous}
    \REQUIRE Trained prediction function $p_{\btheta}$, $\qb_{\mathrm{noise}}$, $\cD_{\tau}$

    \FOR{$n = 1 \ldots N$} 
    \STATE Initiate each token $\xb_{T,n} \sim \qb_{\mathrm{noise}}$
    \STATE Initiate the transition time $\tau_{n} \sim \cD_{\tau}$ and order them as $\tau_{n_1} < \ldots < \tau_{n_N}$
    \ENDFOR
    \FOR{$k = N \ldots 1$}
    \STATE Generate $\tilde{\xb}_{0,1:N}$ from $p_{\btheta}(\cdot|\xb_{\tau_{n_k}, 1:N},\tau_{n_k})$

    \FOR{$n = 1 \ldots N$} 
    \STATE Update $\xb_{\tau_{n_{k-1}},n}$ based on condition of $\tau_{n}$
    \ENDFOR

    \ENDFOR

    \STATE \textbf{Return} $\xb_{0,1:N}$
    \end{algorithmic}
\end{algorithm}
\end{minipage}
\vspace{-10pt}
  \end{wrapfigure}
  
In the previous section, we are constrained to discrete time steps, where we must define a maximum step, denoted by $T$. The values of $\mathbf{x}_t$ are computed only for $t = 1, \ldots, T$.
As a result, during the training process, it is only possible to predict $\xb_0$ at these predetermined time steps. This constraint confines the computation of our reverse process exclusively to these fixed time stamps.
To derive the continuous limit of~\eqref{eq:zeroton}, for each $T$ we rescale~\eqref{eq:zeroton} to a diffusion process on $[0, 1]$, e.g., $\xb_{T, n} = \hat{\xb}_{1, n}, \xb_{0, n} = \hat{\xb}_{0, n}$, and $\xb_{t, n} = \hat{\xb}_{t/T, n}$. Therefore, when $T \rightarrow \infty$, $\hat{\xb}_{t, n}$ represents the continuous process that has values at arbitrary $t \in [0, 1]$.
If the choice of $\alpha_t$ for each $T$ is scale-invariant, we can define a continuous function $\alpha(t)$ as the continuous $\alpha$ schedule of the discrete counterpart\footnote{If we represent $\alpha_t$ with maximum step $T$ as $\alpha_t(T)$, the scale-invariant property states that $\alpha_{ct}(cT) = \alpha_{t}(T)$. The simplest example of such an $\alpha_t$ schedule is $\alpha_t(T) = 1 - t/T$, under which $\alpha(t) = 1 - t$.}.
More specifically, we obtain
\begin{align}\label{eq:continuous_forward}
\hat{\xb}_{t, n} = \alpha(t) \hat{\xb}_{0, n} + (1 - \alpha(t)) \qb_{\text{noise}}, \quad t \in[0, 1].
\end{align}

For the reverse-time process, we define the transition time set $\mathcal{T}:= \{\tau_n\}_{n=1}^N$ consistent with Theorem~\ref{thm:transition probability} and sample it from $\mathbb{P}(\tau_n = t) = -\alpha'(t)$ (we always use decreasing $\alpha(t)$). With $\mathcal{T}$ defined, the updates to $\mathbf{x}_{t, n}$ only occur at $\{\tau_n\}$. Consequently, we arrange $\tau_n$ to obtain an ordered sequence $\tau_{n_k}$, where $\tau_{n_1} < \tau_{n_2} < \ldots <\tau_{n_N}$. When omitting the infinitely many time steps between $\tau_{n_k}$ and $\tau_{n_{k-1}}$, the resulting reverse process is then given by:
\begin{align}
\mathbf{x}_{\tau_{n_{k-1}},n} = \ind(\tau_n=\tau_{n_{k-1}})\mathbf{x}_{0,n} + \ind(\tau_{n}\neq \tau_{n_{k-1}})\mathbf{x}_{\tau_{n_k},n},\label{eq:continuous_tokens}.
\end{align}
for all $n\in [N]$. The detailed algorithm named DNDM-C is shown in Algorithm~\ref{alg:continuous}. 

 \begin{remark} 
Autoregressive Diffusion Model (ARDM) \citep{hoogeboom2021autoregressive} is a discrete diffusion model built upon the autoregressive nature of data. ARDM is shown to be equivalent to a continuous-time absorbing diffusion model and thus provides a unique perspective for discrete diffusion. For continuous-time ($T = \infty$) reverse sampling, both ARDM and our method achieve $N$ NFEs. Unlike ARDM which is limited to absorbing-state transitions, our method provides a unified framework including both absorbing and multinomial diffusions, applicable to both finite time and continuous time diffusions.  For infinite timesteps, \citet{hoogeboom2021autoregressive} also proposed an advanced parallelizing technique that can reduce NFE according to the log-likelihood, which we have not considered in DNDM-C.
\end{remark}

\section{Experiments}\label{sec:experiments}

In this section, we evaluate \method{} and demonstrate its superior performance on two types of tasks: conditional sequence-to-sequence text generation (i.e., machine translation) and unconditional text generation. 
For the fairness of comparison, all the experiments are conducted using a single NVIDIA RTX A6000 GPU with 48 GB memory. Additional experiment details are provided in Appendix~\ref{Appendix: details}.

\subsection{Conditional Text Generation}\label{sec:exp_text}
We evaluate \method{}'s effectiveness on conditional text generation through machine translation tasks. Following \citet{zheng2023reparameterized}, we use Byte Pair Encoding (BPE) \citep{sennrich2016neural} to create a shared vocabulary of words and subwords from both source and target languages. We implement our experiments using FairSeq \citep{ott2019fairseq}, which employs an encoder-decoder architecture. The model uses bi-directional self-attention blocks without causal masking, allowing tokens to attend to both past and future positions during training and inference. The encoder processes the source text, while the decoder generates the target translation.



\noindent\textbf{Datasets.} \CC{ We use the following three datasets to compare with the baselines for machine translation tasks:
(1) $\mathtt{IWSLT14\ DE}$-$\mathtt{EN}$ \citep{cettolo-etal-2014-report}, a dataset with German as the source language and English as the target language. It consists of $174272$ examples (sentence pairs), and each of the validation set and the testing set accounts for $7283$ and $6750$ of the dataset; (2) $\mathtt{WMT14\ EN}$-$\mathtt{DE}$ \citep{bojar-etal-2014-findings}, which is an English-to-German translation dataset consisting of $3967182$ examples. Each of the validation set and the testing set accounts for $3000$ and $3003$ of the dataset; and (3) $\mathtt{WMT16\ EN}$-$\mathtt{RO}$ \citep{bojar-etal-2016-findings}, which is an English-to-Russian translation dataset consisting of $612317$ examples. Each of the validation sets and the testing set accounts for $1999$ and $1999$ of the dataset. The train-validation-test split is fixed across all experiments for all machine translation datasets to ensure fair comparison.}

\noindent\textbf{Performance Metrics.} We use the BLEU score \citep{papineni2002bleu} to evaluate the machine translation quality, where the BLEU score is calculated based on the similarity between the actual target sequence and the predicted target sequence. The sampling speed is measured by wall-clock time (in second).

\noindent\textbf{Baselines.} 
The main baselines we are comparing with are RDM and RDM-$k$ from \citet{zheng2023reparameterized}. Here, we use RDM-$k$ and RDM to denote the sampling method proposed in their paper with and without the usage of top-$k$ selection for the token generation technique (see Appendix~\ref{Appendix: Top} for more details), respectively. RDM and RDM-$k$ are applied to two previously proposed state-of-the-art discrete diffusion models: Multinomial Diffusion \citep{hoogeboom2021argmax} and  Absorbing Diffusion \citep{austin2021structured}.




\noindent\textbf{Results and Discussion.} Tables~\ref{tab:real} and \ref{tab:top} present the performance evaluations of our algorithms in machine translation tasks. Table~\ref{tab:real} presents results for multinomial diffusion, while Table~\ref{tab:top} displays results for absorbing diffusion. Our reported time and BLEU scores are averaged over 5 repeated experiments, except for the baseline RDM experiment\footnote{Due to computational intensity, we did not repeat the 1000-step sampling for the RDM baseline. However, reproducing it was deemed unnecessary as the sampling time is largely stable across repeated experiments, and the precise averaged timing is not critical for demonstrating the speed improvement of \method.}.

From Tables~\ref{tab:real} and~\ref{tab:top}, we observe that methods based on \method{} significantly accelerate the sampling process compared to baseline diffusion models. This acceleration allows for greater flexibility in increasing the number of steps (up to infinity) without imposing a significant computational burden. 
In particular, more \BB{sampling steps} lead to better generation quality (BLEU) at the expense of longer sampling time, as indicated in each column of Tables~\ref{tab:real} and~\ref{tab:top}.
For RDM-based methods, generation time increases linearly with the number of \BB{sampling steps}.
On the contrary, for our \method{}-based method, generation time only increases marginally (See Figure~\ref{fig:step-time} in Section~\ref{sec:add_exp}).
As a result of the difference in the growing speed of sampling time with respect to sampling steps, the more sampling steps, the more speedup \method{} can obtain.

Continuous-time results, as the ultimate limit of increasing sampling steps, are presented in the last row of each dataset with the tag $\infty$. Given that the results with 1000 steps consistently outperform those with 50 steps, we compare $\infty$ with 1000 steps in Table~\ref{tab:real} and~\ref{tab:top}.
For $\mathtt{IWSLT14}$ and $\mathtt{WMT16}$, where the generation BLEU score is relatively high, we observe a consistent performance improvement of up to $0.3$ in BLEU score when utilizing the DNDM-C algorithm, with the exception of a single case in the absorbing diffusion setting for $\mathtt{WMT16}$ without the use of top-$k$ selection.
The performance gain of the continuous-time method on $\mathtt{WMT14}$ is less significant, with both drops and gains. However, $\mathtt{WMT14}$ itself has not reached a high level of performance, with a BLEU score significantly lower than other datasets. In general, training $\mathtt{WMT14}$ poses challenges across all diffusion models, including multinomial diffusion~\citep{hoogeboom2021argmax}, absorbing diffusion~\citep{austin2021structured}, and RDM diffusion~\citep{zheng2023reparameterized}, etc. \blue{We defer a more detailed discussion on WMT14 to Appendix~\ref{app:cond_text_gen}}. Finally, when compared with the results obtained with 50 steps, the performance of DNDM-C demonstrates improvement consistently. 
Furthermore, we note that regardless of the dataset or the method (i.e., RDM or \method{}) employed, top-$k$ token generation consistently outperforms vanilla methods. This approach enhances the BLEU score by approximately $1$-$2$ points without introducing significant increases in sampling time.

 
\newcolumntype{g}{>{\columncolor{LightCyan}\hspace{0pt}}c}

\begin{table*}[ht!]
  \caption{BLEU score comparison of multinomial diffusion on machine translation benchmarks $\mathtt{IWSLT14\ DE}$-$\mathtt{EN}$, $\mathtt{WMT14\ EN}$-$\mathtt{DE}$, and $\mathtt{WMT16\ EN}$-$\mathtt{RO}$. Below the dataset, we present the amount of data used to run the evaluation (sentences).  The blue background highlights our algorithms, and the bold number indicates the best performance within each row and each setting (i.e., with or without top-k).}
    \centering
    \small
\resizebox{\columnwidth}{!}{
    \begin{tabular}{c|c|c|c|g|g|c|c|g|g}
    \toprule 
        \multirow{2}{*}{\textbf{Dataset}} & \multirow{2}{*}{\textbf{Steps}} & \multicolumn{2}{c|}{\textbf{RDM-Multi}} & \multicolumn{2}{c|}{\textbf{DNDM-Multi}} & 
        \multicolumn{2}{c|}{\textbf{RDM-$k$-Multi}} & \multicolumn{2}{c}{\textbf{DNDM-$k$-Multi}} \\
        \cmidrule{3-10}
         &  & \textbf{BLEU} & \textbf{Time (s)} & \cellcolor{white}\textbf{BLEU} & \cellcolor{white}\textbf{Time (s)} & \textbf{BLEU} & \textbf{Time(s)} & \cellcolor{white}\textbf{BLEU} & \cellcolor{white}\textbf{Time (s)} \\
        \midrule
        \multirow{4}{*}{\centering$\mathtt{IWSLT14}$}  & 25 &\textbf{31.26} & 166.9 &30.95 &\textbf{52.9} & \textbf{32.82} & 161.9 & 32.30 &\textbf{52.6} \\
        
           \multirow{4.5}{*}{ (6.75k) }& 50 & \textbf{31.50} & 328.6 & 31.45& \textbf{83.9}& \textbf{32.82} & 321.2 & 32.80 & \textbf{93.2}\\
        
        & 1000 & 31.69 & 6308.9 &\textbf{31.82} & \textbf{191.3} & 32.64 & 6321.3 &  \textbf{33.15} & \textbf{191.5} \\

        & $\infty $ &-  & -& \textbf{31.89} & \textbf{225.2} &  -& -&\textbf{33.44}  & \textbf{228.1} \\
        \midrule
        \multirow{4}{*}{\centering$\mathtt{WMT14}$} & 25 & \textbf{25.25} & 237.3 &25.01 &  \textbf{90.7} & \textbf{26.03} & 230.9 & 25.98& \textbf{90.5}\\
        
          \multirow{4.5}{*}{ (3k) } & 50 & \textbf{25.75} & 466.1 & 25.33& \textbf{138.4} & 26.14 & 500.2 & \textbf{26.37} & \textbf{138.3} \\
        
        & 1000 & 25.66 & 8996.7 & \textbf{25.71} & \textbf{265.4} & 25.82 & 8991.7 &  \textbf{26.88} & \textbf{265.5}\\

        & $\infty $  & - &- & \textbf{24.79} & \textbf{307.5} & - &- & \textbf{26.39} & \textbf{307.3} \\
        \midrule
        \multirow{4}{*}{\centering$\mathtt{WMT16}$} & 25 &\textbf{32.29} &145.2 & 31.97& \textbf{36.4} & \textbf{33.12} & 143.5 & 32.94 & \textbf{36.4}\\
        
          \multirow{4.5}{*}{ (2k) } & 50 & \textbf{32.53} & 286.1&32.50 & \textbf{63.2} & \textbf{33.41} & 312.4 & 33.26 & \textbf{62.7}\\
        
        & 1000 & 32.63 & 5588.9& \textbf{32.86}& \textbf{171.4}& 33.67 & 5601.0 & \textbf{33.79} & \textbf{171.2}\\

        & $\infty $ &-  &- & \textbf{32.91} & \textbf{196.4} &-  &- & \textbf{33.86} & \textbf{196.3} \\
    \bottomrule
    \end{tabular}
    }
    \label{tab:real}
\end{table*}

\begin{table*}[ht!]
 \caption{BLEU score comparison of absorbing diffusion on machine translation benchmarks $\mathtt{IWSLT14\ DE}$-$\mathtt{EN}$, $\mathtt{WMT14\ EN}$-$\mathtt{DE}$, and $\mathtt{WMT16\ EN}$-$\mathtt{RO}$. Below the dataset, we present the amount of data used to run the evaluation (sentences). The blue background highlights our algorithms, and the bold number indicates the best performance within each row and each setting (i.e., with or without top-k).}
    \centering

\resizebox{\columnwidth}{!}{
    \begin{tabular}{c|c|c|c|g|g|c|c|g|g}
    \toprule 
        \multirow{2}{*}{\textbf{Dataset}} & \multirow{2}{*}{\textbf{Steps}} & \multicolumn{2}{c|}{\textbf{RDM-Absorb}} & \multicolumn{2}{c|}{\textbf{DNDM-Absorb}} & 
        \multicolumn{2}{c|}{\textbf{RDM-$k$-Absorb}} & \multicolumn{2}{c}{\textbf{DNDM-$k$-Absorb}} \\
        \cmidrule{3-10}
         &  & \textbf{BLEU} & \textbf{Time (s)} & \cellcolor{white}\textbf{BLEU} & \cellcolor{white}\textbf{Time (s)} & \textbf{BLEU} & \textbf{Time(s)} & \cellcolor{white}\textbf{BLEU} & \cellcolor{white}\textbf{Time (s)} \\
        \midrule
        \multirow{4}{*}{\centering$\mathtt{IWSLT14}$} & 25 & 31.58 & 116.3 & \textbf{32.43} & \textbf{67.2} & \textbf{34.50} & 108.9 & 34.14 & \textbf{67.3} \\
        
          \multirow{4.5}{*}{ (6.75k) } & 50 & 31.80 & 227.2 & \textbf{32.63} & \textbf{95.9}  & \textbf{34.58} & 213.9 & 34.34 & \textbf{96.2}  \\
        
        & 1000 & 31.91 & 4197.4 & \textbf{32.93} & \textbf{161.1}  & \textbf{34.60} & 4205.9 & 34.56 & \textbf{162.3}  \\

        & $\infty $ & - & - &  \textbf{33.03} & \textbf{174.6}  & - & - & \textbf{34.65} & \textbf{180.7}  \\
        \midrule
        \multirow{4}{*}{\centering$\mathtt{WMT14}$} & 25 & 24.97 & 116.4  & \textbf{25.79} & \textbf{68.1} & \textbf{27.50} & 107.5  & 27.18 & \textbf{68.0} \\
        
          \multirow{4.5}{*}{ (3k) } & 50 & 24.95 & 231.1 & \textbf{26.10} & \textbf{102.0}  & \textbf{27.73} & 255.2 & 27.66 & \textbf{102.5}  \\
        
        & 1000 & 25.22 & 4169.4 & \textbf{26.43} & \textbf{178.3}  & 27.75 & 4167.4 & \textbf{27.82} & \textbf{179.1}  \\

        & $\infty $ & -  & - &  \textbf{26.50} & \textbf{180.1}  & - & - & \textbf{27.50} & \textbf{181.2}  \\
        \midrule
        \multirow{4}{*}{\centering$\mathtt{WMT16}$} & 25 & 32.86 & 75.5 & \textbf{33.20} & \textbf{41.2} & 33.92 & 69.9  & \textbf{33.96} & \textbf{41.4} \\
        
           \multirow{4.5}{*}{ (2k) }& 50 & 32.93 & 148.4  & \textbf{33.30} &   \textbf{62.5}  & 34.10 & 166.1 & \textbf{34.20} & \textbf{62.7} \\
        
        & 1000 & 33.25 & 2951.7 & \textbf{33.60} & \textbf{121.3} & \textbf{34.44} & 2718.7 & 34.38 & \textbf{122.7}  \\

        & $\infty $ & - & - &  \textbf{33.42} & \textbf{121.8}  &  - & - & \textbf{34.41} & \textbf{121.9}  \\
    \bottomrule
    \end{tabular}
    }
    \label{tab:top}
\end{table*}
\noindent\textbf{Scaling Law in Sampling Speed.}
For illustrative purposes, we use the example of $\mathtt{IWSLT14}$ to visualize how the sample quality scales regarding sampling speed for different methods. In Figure~\ref{fig:bleuvstime}, we observe the trend of the BLEU score in relation to computational time. Each line in the legend represents a different sampling algorithm, and a steeper slope indicates a larger marginal gain when sampling for longer periods. Figure~\ref{fig:bleuvstime} demonstrates that our algorithm displays nearly linear growth in BLEU score over the log of time, which is remarkable in contrast with the flat curve of the baseline.
Particularly, for multinomial diffusion, the BLEU score increases by 1 in less than 60 seconds of additional sampling time. For absorbing diffusion, \method{} outperforms RDM before RDM samples 50 steps. 
\blue{In Tables~\ref{tab:real2} and~\ref{tab:top2} in Appendix~\ref{sec:nfe}, we further use the average number of function evaluations (NFE) to measure the improved speed within the specified number of sampling steps.}
Additionally, in Figure~\ref{fig:generation}, we visualize how the BLEU score and the generated text change throughout the sampling process.

\begin{figure}[t]
  \centering
  \begin{subfigure}{0.48\textwidth}
    \includegraphics[width=\linewidth]{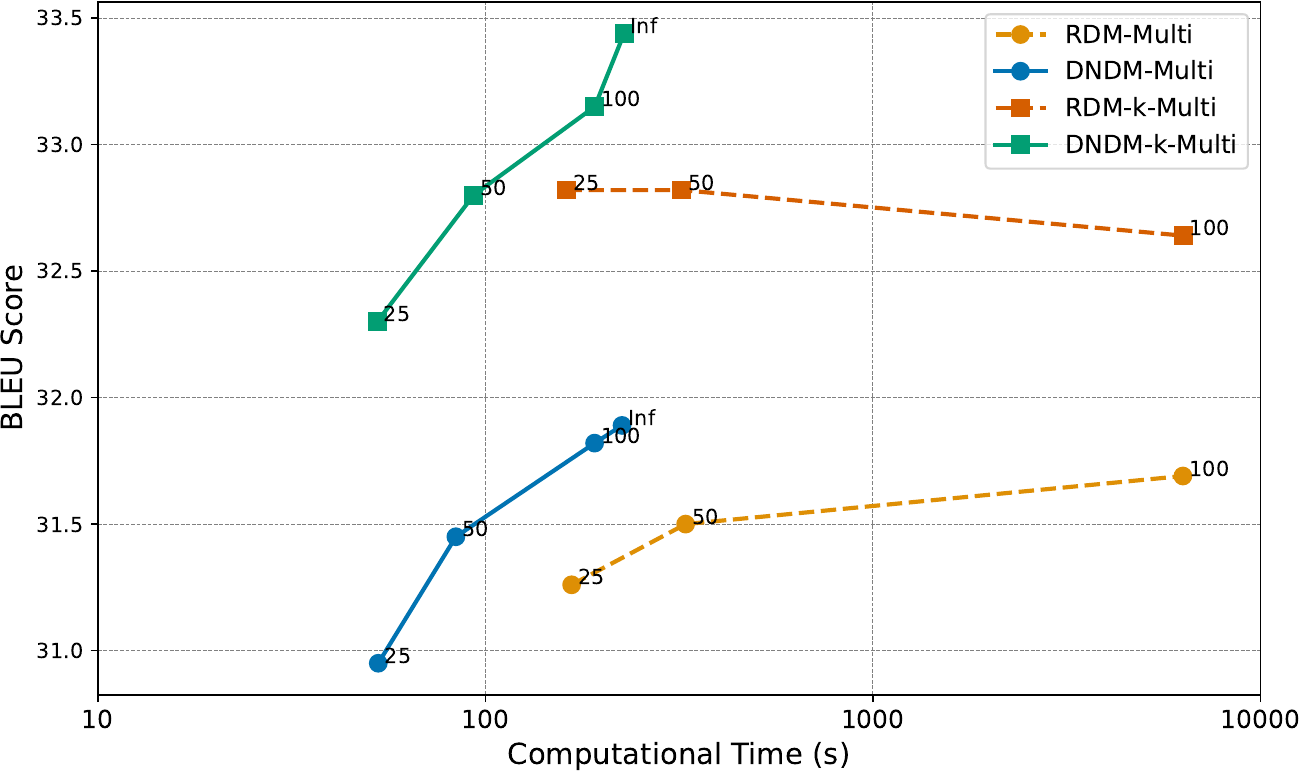}
    \caption{Multinomial Diffusion}
    \label{fig:multi-bleu}
  \end{subfigure}
    \begin{subfigure}{0.48\textwidth}
    \includegraphics[width=\linewidth]{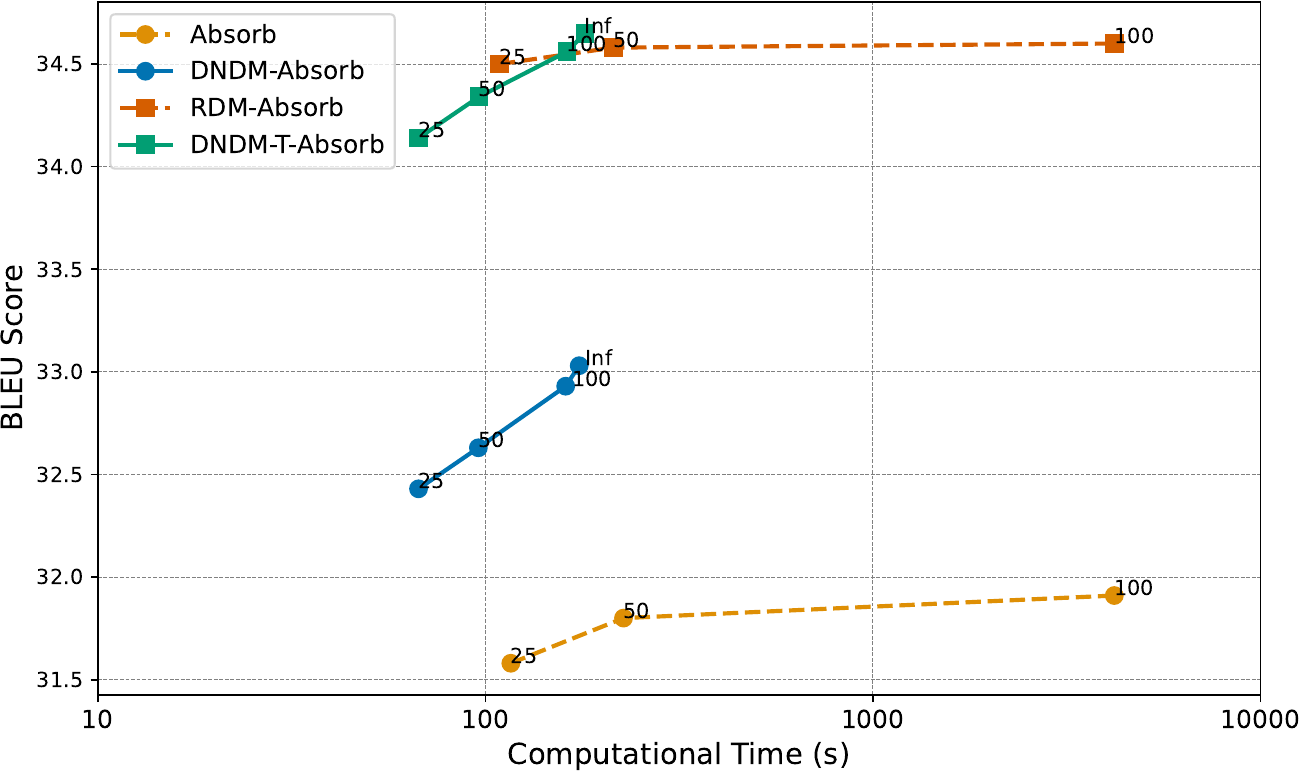}
    \caption{Absorbing Diffusion}
    \label{fig:absorb-bleu}
  \end{subfigure}
  \caption{Generation quality to generation time comparison on $\mathtt{IWSLT14}$. $x$-axis: computational time in seconds; $y$-axis: BLEU score. }
  \label{fig:bleuvstime}
\end{figure}

\begin{figure}[t]
  \centering
  \begin{subfigure}{0.46\textwidth}
    \includegraphics[width=\linewidth]{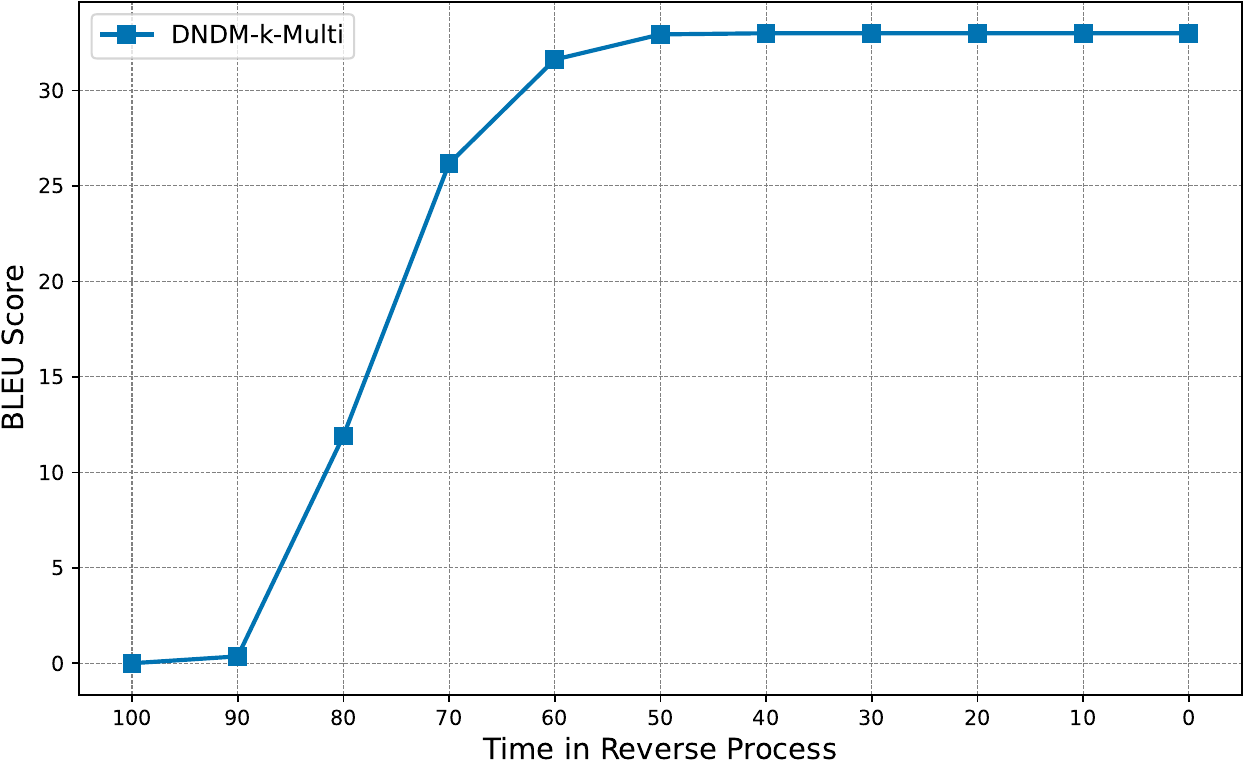}
    \caption{The BLEU Score in the Generation Process}
    \label{fig:bleu}
  \end{subfigure}
  \hspace{0.05\textwidth} 
  \blue{
  \begin{subfigure}{0.46\textwidth}
    \begin{minipage}{\linewidth}
      \small{\textbf{t  =  100}  ~ \textbf{[noise] [noise] [noise] [noise] $\cdots$}}  \\[3pt]
      \small{\textbf{t = 75} \textbf{[noise]} $\cdots$ \textbf{[noise]} \texttt{and we} \textbf{[noise]} $\cdots$ \textbf{[noise]} \texttt{govern}\textbf{[noise] [noise]} \texttt{year} \textbf{[noise]}  } \\[3pt]
      \small{\textbf{t = 67}  \texttt{we} \textbf{[noise] [noise]} \texttt{fello} \textbf{[noise] [noise] [noise]} and we let them \textbf{[noise] [noise]} \texttt{city govern}\textbf{[noise]} \texttt{every year.}}\\[3pt]
      \small{\textbf{t = 39} ~\texttt{we choose some fellows every year and we let them work with city governance every year.}}\\[3pt]
      \small{\textbf{t = 0} ~\texttt{we choose some fellows every year and we let them work with city governance every year.}}
    \end{minipage}
    \caption{Text in the Generation Process}
    \label{fig:text}
  \end{subfigure}
  }
  \caption{We demonstrate the 100-step generation process of  \method{}-$k$-Multi as an example, where the left is the change of the BLEU score along the generation process, and the right is the text at different time steps. As the time goes from 100 to 0, noise is gradually removed until the corresponding English text emerges. Since the transition time follows a Beta distribution as described in Section~\ref{sec:fast}, the majority of transitions occur near the starting time.}
  \label{fig:generation}
\end{figure}

\subsection{Unconditional Text Generation}
For unconditional text generation, we evaluate our approach on language modeling tasks, where the model learns to generate text that matches the statistical patterns of the training data. Unlike conditional generation, this task involves directly learning $q(\xb_0|\xb_t)$ without conditioning on any input text. We conduct experiments on the text8 and enwik8 datasets using a decoder-only architecture similar to GPT models. Since unconditional generation does not require encoding input sequences, we employ a 12-layer Transformer decoder without an encoder component.

\noindent\textbf{Datasets.} The natural language generation task is evaluated on two language datasets following \cite{hoogeboom2021argmax}: $\mathtt{text8}$ and $\mathtt{enwik8}$. Both datasets are from Wikipedia, but their contents are highly distinct. In $\mathtt{text8}$, the plain text consists of English words (all the letters are in lower case) and spaces, and it is tokenized into 26 characters and one blank space, resulting in 27 categories. In contrast to the cleanness of $\mathtt{text8}$, $\mathtt{enwik8}$ preserves the original XML dump contents, and there exist various special symbols in its raw text, so its text is tokenized into 1 Byte, resulting in 256 categories. We utilize $\mathtt{text8}$ dataset with sequence length 256 and $\mathtt{enwik8}$ dataset with sequence length 320. The train/val/test splits are $
9\text{e}7/5\text{e}6/5\text{e}5$ for both $\mathtt{text8}$ and $\mathtt{enwik8}$.

\noindent\textbf{Performance Metrics.} Our evaluation of text generation quality relies on the perplexity score. When generating $\mathtt{text8}$ data, we calculate perplexity scores using the GPT2 model, while for enwik8 data generation, we employ the \blue{GPT2-large} model. The sampling speed is measured in seconds.

\noindent\textbf{Baselines.} We compare our proposed \method{} on unconditional text generation task with the vanilla Multinomial Diffusion \citep{hoogeboom2021argmax}.
\begin{wraptable}{r}{0.55\textwidth}
 \vspace{-10pt} 
\centering
\caption{Comparison of different sampling methods for unconditional text generation (multinomial diffusion) on $\mathtt{text8}$ and $\mathtt{enwik8}$ benchmarks. Sampling time is computed by generating a single text sample of length 256 for $\mathtt{text8}$ and length 320 for $\mathtt{enwik8}$, averaged over 10 runs. The blue background represents our algorithms, and the bold number indicates the optimal value.}
\begin{tabular}{c|c|c|g}
\toprule
\multicolumn{2}{c|}{} & Vanilla & \cellcolor{white}DNDM \\ \midrule
\multirow{2}{*}{$\mathtt{text8}$} & Perplexity & 1,465.75 & \textbf{600.02} \\
& Time (s) & 135.9 & \textbf{31.1} \\ \midrule
\multirow{2}{*}{$\mathtt{enwik8}$} & Perplexity & 801.78 & \textbf{556.78} \\
& Time (s) & 602.8 & \textbf{47.4} \\
\bottomrule
\end{tabular}
\label{tab:multinomial}
 \vspace{0pt} 
\end{wraptable}
\noindent\textbf{Results and Discussion.}   Table~\ref{tab:multinomial} displays the performance of our algorithms in text generation tasks. We run the multinomial diffusion model on the $\mathtt{text8}$ dataset for 1000 diffusion steps and on the $\mathtt{enwik8}$ dataset for 4000 diffusion steps. Our \method{}-based algorithms outperform the vanilla sampling algorithm used in \cite{hoogeboom2021argmax}
in terms of both sampling time and perplexity score. Specifically, for the $\mathtt{text8}$ dataset, \method{}-based algorithms are $5$ times faster than the vanilla algorithm. For the $\mathtt{enwik8}$ dataset, \method{}-based algorithms are 14 times faster than the vanilla algorithm.


\section{Conclusion and Future Work}

This paper presents a novel discrete non-Markov diffusion model (\method{}) accompanied by an accelerated sampling algorithm designed to boost sampling speed in a discrete-state space. Our discrete diffusion model incorporates "transition time set" latent variables, establishing itself as an efficacious diffusion and data generation method. Thanks to our acceleration technique, we significantly decrease the number of neural network function evaluations without sacrificing sample quality. We also introduce an infinite-step sampling algorithm, \method{}-C, which provides new insights into bridging the gap between discrete and continuous-time processes for discrete diffusion models. While this study focuses on text generation using non-autoregressive models, a promising direction for future exploration is applying our method to other tasks, such as audio and image generation.

\section*{Acknowledgement}
We thank the anonymous reviewers and area chair for their helpful comments. ZC, HY, YL, YK, JZ, and QG are supported in part by the National Science Foundation CAREER Award 1906169, IIS-2008981, and the Sloan Research Fellowship. The views and conclusions contained in this paper are those of the authors and should not be interpreted as representing any funding agencies.

\bibliography{Diffusion}

\begin{thebibliography}{53}
\expandafter\ifx\csname natexlab\endcsname\relax\def\natexlab#1{#1}\fi
\expandafter\ifx\csname url\endcsname\relax
  \def\url#1{\texttt{#1}}\fi
\expandafter\ifx\csname urlprefix\endcsname\relax\def\urlprefix{URL }\fi

\bibitem[{Alain et~al.(2016)Alain, Bengio, Yao, Yosinski, Thibodeau-Laufer,
  Zhang and Vincent}]{alain2016gsns}
\textsc{Alain, G.}, \textsc{Bengio, Y.}, \textsc{Yao, L.}, \textsc{Yosinski,
  J.}, \textsc{Thibodeau-Laufer, E.}, \textsc{Zhang, S.} and \textsc{Vincent,
  P.} (2016).
\newblock Gsns: generative stochastic networks.
\newblock \textit{Information and Inference: A Journal of the IMA} \textbf{5}
  210--249.

\bibitem[{ALIAS PARTH~GOYAL et~al.(2017)ALIAS PARTH~GOYAL, Ke, Ganguli and
  Bengio}]{alias2017variational}
\textsc{ALIAS PARTH~GOYAL, A.~G.}, \textsc{Ke, N.~R.}, \textsc{Ganguli, S.} and
  \textsc{Bengio, Y.} (2017).
\newblock Variational walkback: Learning a transition operator as a stochastic
  recurrent net.
\newblock \textit{Advances in Neural Information Processing Systems}
  \textbf{30}.

\bibitem[{Austin et~al.(2021)Austin, Johnson, Ho, Tarlow and Van
  Den~Berg}]{austin2021structured}
\textsc{Austin, J.}, \textsc{Johnson, D.~D.}, \textsc{Ho, J.}, \textsc{Tarlow,
  D.} and \textsc{Van Den~Berg, R.} (2021).
\newblock Structured denoising diffusion models in discrete state-spaces.
\newblock \textit{Advances in Neural Information Processing Systems}
  \textbf{34} 17981--17993.

\bibitem[{Bao et~al.(2022)Bao, Li, Zhu and Zhang}]{bao2022analytic}
\textsc{Bao, F.}, \textsc{Li, C.}, \textsc{Zhu, J.} and \textsc{Zhang, B.}
  (2022).
\newblock Analytic-dpm: an analytic estimate of the optimal reverse variance in
  diffusion probabilistic models.
\newblock \textit{arXiv preprint arXiv:2201.06503} .

\bibitem[{Bengio et~al.(2014)Bengio, Laufer, Alain and
  Yosinski}]{bengio2014deep}
\textsc{Bengio, Y.}, \textsc{Laufer, E.}, \textsc{Alain, G.} and
  \textsc{Yosinski, J.} (2014).
\newblock Deep generative stochastic networks trainable by backprop.
\newblock In \textit{International Conference on Machine Learning}. PMLR.

\bibitem[{Bojar et~al.(2014)Bojar, Buck, Federmann, Haddow, Koehn, Leveling,
  Monz, Pecina, Post, Saint-Amand, Soricut, Specia and
  Tamchyna}]{bojar-etal-2014-findings}
\textsc{Bojar, O.}, \textsc{Buck, C.}, \textsc{Federmann, C.}, \textsc{Haddow,
  B.}, \textsc{Koehn, P.}, \textsc{Leveling, J.}, \textsc{Monz, C.},
  \textsc{Pecina, P.}, \textsc{Post, M.}, \textsc{Saint-Amand, H.},
  \textsc{Soricut, R.}, \textsc{Specia, L.} and \textsc{Tamchyna, A.} (2014).
\newblock Findings of the 2014 workshop on statistical machine translation.
\newblock In \textit{Proceedings of the Ninth Workshop on Statistical Machine
  Translation}. Association for Computational Linguistics, Baltimore, Maryland,
  USA.

\bibitem[{Bojar et~al.(2016)Bojar, Chatterjee, Federmann, Graham, Haddow, Huck,
  Jimeno~Yepes, Koehn, Logacheva, Monz, Negri, N{\'e}v{\'e}ol, Neves, Popel,
  Post, Rubino, Scarton, Specia, Turchi, Verspoor and
  Zampieri}]{bojar-etal-2016-findings}
\textsc{Bojar, O.}, \textsc{Chatterjee, R.}, \textsc{Federmann, C.},
  \textsc{Graham, Y.}, \textsc{Haddow, B.}, \textsc{Huck, M.},
  \textsc{Jimeno~Yepes, A.}, \textsc{Koehn, P.}, \textsc{Logacheva, V.},
  \textsc{Monz, C.}, \textsc{Negri, M.}, \textsc{N{\'e}v{\'e}ol, A.},
  \textsc{Neves, M.}, \textsc{Popel, M.}, \textsc{Post, M.}, \textsc{Rubino,
  R.}, \textsc{Scarton, C.}, \textsc{Specia, L.}, \textsc{Turchi, M.},
  \textsc{Verspoor, K.} and \textsc{Zampieri, M.} (2016).
\newblock Findings of the 2016 conference on machine translation.
\newblock In \textit{Proceedings of the First Conference on Machine
  Translation: Volume 2, Shared Task Papers}. Association for Computational
  Linguistics, Berlin, Germany.

\bibitem[{Bordes et~al.(2017)Bordes, Honari and Vincent}]{bordes2017learning}
\textsc{Bordes, F.}, \textsc{Honari, S.} and \textsc{Vincent, P.} (2017).
\newblock Learning to generate samples from noise through infusion training.
\newblock \textit{arXiv preprint arXiv:1703.06975} .

\bibitem[{Campbell et~al.(2022)Campbell, Benton, De~Bortoli, Rainforth,
  Deligiannidis and Doucet}]{campbell2022continuous}
\textsc{Campbell, A.}, \textsc{Benton, J.}, \textsc{De~Bortoli, V.},
  \textsc{Rainforth, T.}, \textsc{Deligiannidis, G.} and \textsc{Doucet, A.}
  (2022).
\newblock A continuous time framework for discrete denoising models.
\newblock \textit{Advances in Neural Information Processing Systems}
  \textbf{35} 28266--28279.

\bibitem[{Ceritli et~al.(2023)Ceritli, Ghosheh, Chauhan, Zhu, Creagh and
  Clifton}]{ceritli2023synthesizing}
\textsc{Ceritli, T.}, \textsc{Ghosheh, G.~O.}, \textsc{Chauhan, V.~K.},
  \textsc{Zhu, T.}, \textsc{Creagh, A.~P.} and \textsc{Clifton, D.~A.} (2023).
\newblock Synthesizing mixed-type electronic health records using diffusion
  models.
\newblock \textit{arXiv preprint arXiv:2302.14679} .

\bibitem[{Cettolo et~al.(2014)Cettolo, Niehues, St{\"u}ker, Bentivogli and
  Federico}]{cettolo-etal-2014-report}
\textsc{Cettolo, M.}, \textsc{Niehues, J.}, \textsc{St{\"u}ker, S.},
  \textsc{Bentivogli, L.} and \textsc{Federico, M.} (2014).
\newblock Report on the 11th {IWSLT} evaluation campaign.
\newblock In \textit{Proceedings of the 11th International Workshop on Spoken
  Language Translation: Evaluation Campaign}. Lake Tahoe, California.

\bibitem[{Chang et~al.(2022)Chang, Zhang, Jiang, Liu and
  Freeman}]{chang2022maskgit}
\textsc{Chang, H.}, \textsc{Zhang, H.}, \textsc{Jiang, L.}, \textsc{Liu, C.}
  and \textsc{Freeman, W.~T.} (2022).
\newblock Maskgit: Masked generative image transformer.
\newblock In \textit{Proceedings of the IEEE/CVF Conference on Computer Vision
  and Pattern Recognition}.

\bibitem[{Chen et~al.(2020)Chen, Zhang, Zen, Weiss, Norouzi and
  Chan}]{chen2020wavegrad}
\textsc{Chen, N.}, \textsc{Zhang, Y.}, \textsc{Zen, H.}, \textsc{Weiss, R.~J.},
  \textsc{Norouzi, M.} and \textsc{Chan, W.} (2020).
\newblock Wavegrad: Estimating gradients for waveform generation.
\newblock \textit{arXiv preprint arXiv:2009.00713} .

\bibitem[{Chung et~al.(2022)Chung, Sim and Ye}]{chung2022come}
\textsc{Chung, H.}, \textsc{Sim, B.} and \textsc{Ye, J.~C.} (2022).
\newblock Come-closer-diffuse-faster: Accelerating conditional diffusion models
  for inverse problems through stochastic contraction.
\newblock In \textit{Proceedings of the IEEE/CVF Conference on Computer Vision
  and Pattern Recognition}.

\bibitem[{Dockhorn et~al.(2021)Dockhorn, Vahdat and Kreis}]{dockhorn2021score}
\textsc{Dockhorn, T.}, \textsc{Vahdat, A.} and \textsc{Kreis, K.} (2021).
\newblock Score-based generative modeling with critically-damped langevin
  diffusion.
\newblock \textit{arXiv preprint arXiv:2112.07068} .

\bibitem[{Dockhorn et~al.(2022)Dockhorn, Vahdat and Kreis}]{dockhorn2022genie}
\textsc{Dockhorn, T.}, \textsc{Vahdat, A.} and \textsc{Kreis, K.} (2022).
\newblock Genie: Higher-order denoising diffusion solvers.
\newblock \textit{Advances in Neural Information Processing Systems}
  \textbf{35} 30150--30166.

\bibitem[{Ghazvininejad et~al.(2019)Ghazvininejad, Levy, Liu and
  Zettlemoyer}]{ghazvininejad2019mask}
\textsc{Ghazvininejad, M.}, \textsc{Levy, O.}, \textsc{Liu, Y.} and
  \textsc{Zettlemoyer, L.} (2019).
\newblock Mask-predict: Parallel decoding of conditional masked language
  models.
\newblock \textit{arXiv preprint arXiv:1904.09324} .

\bibitem[{Gruver et~al.(2024)Gruver, Stanton, Frey, Rudner, Hotzel,
  Lafrance-Vanasse, Rajpal, Cho and Wilson}]{gruver2024protein}
\textsc{Gruver, N.}, \textsc{Stanton, S.}, \textsc{Frey, N.}, \textsc{Rudner,
  T.~G.}, \textsc{Hotzel, I.}, \textsc{Lafrance-Vanasse, J.}, \textsc{Rajpal,
  A.}, \textsc{Cho, K.} and \textsc{Wilson, A.~G.} (2024).
\newblock Protein design with guided discrete diffusion.
\newblock \textit{Advances in Neural Information Processing Systems}
  \textbf{36}.

\bibitem[{He et~al.(2022)He, Sun, Wang, Huang and Qiu}]{he2022diffusionbert}
\textsc{He, Z.}, \textsc{Sun, T.}, \textsc{Wang, K.}, \textsc{Huang, X.} and
  \textsc{Qiu, X.} (2022).
\newblock Diffusionbert: Improving generative masked language models with
  diffusion models.
\newblock \textit{arXiv preprint arXiv:2211.15029} .

\bibitem[{Ho et~al.(2022)Ho, Chan, Saharia, Whang, Gao, Gritsenko, Kingma,
  Poole, Norouzi, Fleet et~al.}]{ho2022imagen}
\textsc{Ho, J.}, \textsc{Chan, W.}, \textsc{Saharia, C.}, \textsc{Whang, J.},
  \textsc{Gao, R.}, \textsc{Gritsenko, A.}, \textsc{Kingma, D.~P.},
  \textsc{Poole, B.}, \textsc{Norouzi, M.}, \textsc{Fleet, D.~J.}
  \textsc{et~al.} (2022).
\newblock Imagen video: High definition video generation with diffusion models.
\newblock \textit{arXiv preprint arXiv:2210.02303} .

\bibitem[{Ho et~al.(2020)Ho, Jain and Abbeel}]{ho2020denoising}
\textsc{Ho, J.}, \textsc{Jain, A.} and \textsc{Abbeel, P.} (2020).
\newblock Denoising diffusion probabilistic models.
\newblock \textit{Advances in neural information processing systems}
  \textbf{33} 6840--6851.

\bibitem[{Hoogeboom et~al.(2021{\natexlab{a}})Hoogeboom, Gritsenko, Bastings,
  Poole, Berg and Salimans}]{hoogeboom2021autoregressive}
\textsc{Hoogeboom, E.}, \textsc{Gritsenko, A.~A.}, \textsc{Bastings, J.},
  \textsc{Poole, B.}, \textsc{Berg, R. v.~d.} and \textsc{Salimans, T.}
  (2021{\natexlab{a}}).
\newblock Autoregressive diffusion models.
\newblock \textit{arXiv preprint arXiv:2110.02037} .

\bibitem[{Hoogeboom et~al.(2021{\natexlab{b}})Hoogeboom, Nielsen, Jaini,
  Forr{\'e} and Welling}]{hoogeboom2021argmax}
\textsc{Hoogeboom, E.}, \textsc{Nielsen, D.}, \textsc{Jaini, P.},
  \textsc{Forr{\'e}, P.} and \textsc{Welling, M.} (2021{\natexlab{b}}).
\newblock Argmax flows and multinomial diffusion: Learning categorical
  distributions.
\newblock \textit{Advances in Neural Information Processing Systems}
  \textbf{34} 12454--12465.

\bibitem[{Jolicoeur-Martineau et~al.(2021)Jolicoeur-Martineau, Li,
  Pich{\'e}-Taillefer, Kachman and Mitliagkas}]{jolicoeur2021gotta}
\textsc{Jolicoeur-Martineau, A.}, \textsc{Li, K.}, \textsc{Pich{\'e}-Taillefer,
  R.}, \textsc{Kachman, T.} and \textsc{Mitliagkas, I.} (2021).
\newblock Gotta go fast when generating data with score-based models.
\newblock \textit{arXiv preprint arXiv:2105.14080} .

\bibitem[{Karras et~al.(2022)Karras, Aittala, Aila and
  Laine}]{karras2022elucidating}
\textsc{Karras, T.}, \textsc{Aittala, M.}, \textsc{Aila, T.} and \textsc{Laine,
  S.} (2022).
\newblock Elucidating the design space of diffusion-based generative models.
\newblock \textit{Advances in Neural Information Processing Systems}
  \textbf{35} 26565--26577.

\bibitem[{Kong and Ping(2021)}]{kong2021fast}
\textsc{Kong, Z.} and \textsc{Ping, W.} (2021).
\newblock On fast sampling of diffusion probabilistic models.
\newblock \textit{arXiv preprint arXiv:2106.00132} .

\bibitem[{Kong et~al.(2020)Kong, Ping, Huang, Zhao and
  Catanzaro}]{kong2020diffwave}
\textsc{Kong, Z.}, \textsc{Ping, W.}, \textsc{Huang, J.}, \textsc{Zhao, K.} and
  \textsc{Catanzaro, B.} (2020).
\newblock Diffwave: A versatile diffusion model for audio synthesis.
\newblock \textit{arXiv preprint arXiv:2009.09761} .

\bibitem[{Liu et~al.(2022)Liu, Ren, Lin and Zhao}]{liu2022pseudo}
\textsc{Liu, L.}, \textsc{Ren, Y.}, \textsc{Lin, Z.} and \textsc{Zhao, Z.}
  (2022).
\newblock Pseudo numerical methods for diffusion models on manifolds.
\newblock \textit{arXiv preprint arXiv:2202.09778} .

\bibitem[{Lu et~al.(2022)Lu, Zhou, Bao, Chen, Li and Zhu}]{lu2022dpm}
\textsc{Lu, C.}, \textsc{Zhou, Y.}, \textsc{Bao, F.}, \textsc{Chen, J.},
  \textsc{Li, C.} and \textsc{Zhu, J.} (2022).
\newblock Dpm-solver: A fast ode solver for diffusion probabilistic model
  sampling in around 10 steps.
\newblock \textit{Advances in Neural Information Processing Systems}
  \textbf{35} 5775--5787.

\bibitem[{Lyu(2012)}]{lyu2012interpretation}
\textsc{Lyu, S.} (2012).
\newblock Interpretation and generalization of score matching.
\newblock \textit{arXiv preprint arXiv:1205.2629} .

\bibitem[{Movellan(2008)}]{movellan2008contrastive}
\textsc{Movellan, J.~R.} (2008).
\newblock Contrastive divergence in gaussian diffusions.
\newblock \textit{Neural Computation} \textbf{20} 2238--2252.

\bibitem[{Nachmani et~al.(2021)Nachmani, Roman and Wolf}]{nachmani2021non}
\textsc{Nachmani, E.}, \textsc{Roman, R.~S.} and \textsc{Wolf, L.} (2021).
\newblock Non gaussian denoising diffusion models.
\newblock \textit{arXiv preprint arXiv:2106.07582} .

\bibitem[{Nichol and Dhariwal(2021)}]{nichol2021improved}
\textsc{Nichol, A.~Q.} and \textsc{Dhariwal, P.} (2021).
\newblock Improved denoising diffusion probabilistic models.
\newblock In \textit{International Conference on Machine Learning}. PMLR.

\bibitem[{Ott et~al.(2019)Ott, Edunov, Baevski, Fan, Gross, Ng, Grangier and
  Auli}]{ott2019fairseq}
\textsc{Ott, M.}, \textsc{Edunov, S.}, \textsc{Baevski, A.}, \textsc{Fan, A.},
  \textsc{Gross, S.}, \textsc{Ng, N.}, \textsc{Grangier, D.} and \textsc{Auli,
  M.} (2019).
\newblock fairseq: A fast, extensible toolkit for sequence modeling.
\newblock \textit{arXiv preprint arXiv:1904.01038} .

\bibitem[{Papineni et~al.(2002)Papineni, Roukos, Ward and
  Zhu}]{papineni2002bleu}
\textsc{Papineni, K.}, \textsc{Roukos, S.}, \textsc{Ward, T.} and \textsc{Zhu,
  W.-J.} (2002).
\newblock Bleu: a method for automatic evaluation of machine translation.
\newblock In \textit{Proceedings of the 40th annual meeting of the Association
  for Computational Linguistics}.

\bibitem[{Reid et~al.(2022)Reid, Hellendoorn and Neubig}]{reid2022diffuser}
\textsc{Reid, M.}, \textsc{Hellendoorn, V.~J.} and \textsc{Neubig, G.} (2022).
\newblock Diffuser: Discrete diffusion via edit-based reconstruction.
\newblock \textit{arXiv preprint arXiv:2210.16886} .

\bibitem[{Salimans and Ho(2022)}]{salimans2022progressive}
\textsc{Salimans, T.} and \textsc{Ho, J.} (2022).
\newblock Progressive distillation for fast sampling of diffusion models.
\newblock \textit{arXiv preprint arXiv:2202.00512} .

\bibitem[{San-Roman et~al.(2021)San-Roman, Nachmani and Wolf}]{san2021noise}
\textsc{San-Roman, R.}, \textsc{Nachmani, E.} and \textsc{Wolf, L.} (2021).
\newblock Noise estimation for generative diffusion models.
\newblock \textit{arXiv preprint arXiv:2104.02600} .

\bibitem[{Savinov et~al.(2021)Savinov, Chung, Binkowski, Elsen and
  Oord}]{savinov2021step}
\textsc{Savinov, N.}, \textsc{Chung, J.}, \textsc{Binkowski, M.},
  \textsc{Elsen, E.} and \textsc{Oord, A. v.~d.} (2021).
\newblock Step-unrolled denoising autoencoders for text generation.
\newblock \textit{arXiv preprint arXiv:2112.06749} .

\bibitem[{Sennrich et~al.(2016)Sennrich, Haddow and Birch}]{sennrich2016neural}
\textsc{Sennrich, R.}, \textsc{Haddow, B.} and \textsc{Birch, A.} (2016).
\newblock Neural machine translation of rare words with subword units.

\bibitem[{Sohl-Dickstein et~al.(2009)Sohl-Dickstein, Battaglino and
  DeWeese}]{sohl2009minimum}
\textsc{Sohl-Dickstein, J.}, \textsc{Battaglino, P.} and \textsc{DeWeese,
  M.~R.} (2009).
\newblock Minimum probability flow learning.
\newblock \textit{arXiv preprint arXiv:0906.4779} .

\bibitem[{Sohl-Dickstein et~al.(2015)Sohl-Dickstein, Weiss, Maheswaranathan and
  Ganguli}]{sohl2015deep}
\textsc{Sohl-Dickstein, J.}, \textsc{Weiss, E.}, \textsc{Maheswaranathan, N.}
  and \textsc{Ganguli, S.} (2015).
\newblock Deep unsupervised learning using nonequilibrium thermodynamics.
\newblock In \textit{International conference on machine learning}. PMLR.

\bibitem[{Song et~al.(2020{\natexlab{a}})Song, Meng and
  Ermon}]{song2020denoising}
\textsc{Song, J.}, \textsc{Meng, C.} and \textsc{Ermon, S.}
  (2020{\natexlab{a}}).
\newblock Denoising diffusion implicit models.
\newblock \textit{arXiv preprint arXiv:2010.02502} .

\bibitem[{Song et~al.(2023)Song, Dhariwal, Chen and
  Sutskever}]{song2023consistency}
\textsc{Song, Y.}, \textsc{Dhariwal, P.}, \textsc{Chen, M.} and
  \textsc{Sutskever, I.} (2023).
\newblock Consistency models.
\newblock \textit{arXiv preprint arXiv:2303.01469} .

\bibitem[{Song and Ermon(2019)}]{song2019generative}
\textsc{Song, Y.} and \textsc{Ermon, S.} (2019).
\newblock Generative modeling by estimating gradients of the data distribution.
\newblock \textit{Advances in neural information processing systems}
  \textbf{32}.

\bibitem[{Song and Ermon(2020)}]{song2020improved}
\textsc{Song, Y.} and \textsc{Ermon, S.} (2020).
\newblock Improved techniques for training score-based generative models.
\newblock \textit{Advances in neural information processing systems}
  \textbf{33} 12438--12448.

\bibitem[{Song et~al.(2020{\natexlab{b}})Song, Sohl-Dickstein, Kingma, Kumar,
  Ermon and Poole}]{song2020score}
\textsc{Song, Y.}, \textsc{Sohl-Dickstein, J.}, \textsc{Kingma, D.~P.},
  \textsc{Kumar, A.}, \textsc{Ermon, S.} and \textsc{Poole, B.}
  (2020{\natexlab{b}}).
\newblock Score-based generative modeling through stochastic differential
  equations.
\newblock \textit{arXiv preprint arXiv:2011.13456} .

\bibitem[{Sun et~al.(2022)Sun, Yu, Dai, Schuurmans and Dai}]{sun2022score}
\textsc{Sun, H.}, \textsc{Yu, L.}, \textsc{Dai, B.}, \textsc{Schuurmans, D.}
  and \textsc{Dai, H.} (2022).
\newblock Score-based continuous-time discrete diffusion models.
\newblock \textit{arXiv preprint arXiv:2211.16750} .

\bibitem[{Vahdat et~al.(2021)Vahdat, Kreis and Kautz}]{vahdat2021score}
\textsc{Vahdat, A.}, \textsc{Kreis, K.} and \textsc{Kautz, J.} (2021).
\newblock Score-based generative modeling in latent space.
\newblock \textit{Advances in Neural Information Processing Systems}
  \textbf{34} 11287--11302.

\bibitem[{Watson et~al.(2021)Watson, Ho, Norouzi and Chan}]{watson2021learning}
\textsc{Watson, D.}, \textsc{Ho, J.}, \textsc{Norouzi, M.} and \textsc{Chan,
  W.} (2021).
\newblock Learning to efficiently sample from diffusion probabilistic models.
\newblock \textit{arXiv preprint arXiv:2106.03802} .

\bibitem[{Ye et~al.(2023)Ye, Zheng, Bao, Qian and Gu}]{ye2023diffusion}
\textsc{Ye, J.}, \textsc{Zheng, Z.}, \textsc{Bao, Y.}, \textsc{Qian, L.} and
  \textsc{Gu, Q.} (2023).
\newblock Diffusion language models can perform many tasks with scaling and
  instruction-finetuning.
\newblock \textit{arXiv preprint arXiv:2308.12219} .

\bibitem[{Zhang and Chen(2022)}]{zhang2022fast}
\textsc{Zhang, Q.} and \textsc{Chen, Y.} (2022).
\newblock Fast sampling of diffusion models with exponential integrator.
\newblock \textit{arXiv preprint arXiv:2204.13902} .

\bibitem[{Zheng et~al.(2023)Zheng, Yuan, Yu and
  Kong}]{zheng2023reparameterized}
\textsc{Zheng, L.}, \textsc{Yuan, J.}, \textsc{Yu, L.} and \textsc{Kong, L.}
  (2023).
\newblock A reparameterized discrete diffusion model for text generation.
\newblock \textit{arXiv preprint arXiv:2302.05737} .

\end{thebibliography}
\bibliographystyle{ims}

\newpage
\appendix

\section*{Broader Impact} 

This paper presents work that aims to advance the field of diffusion models. We believe this work may enable future applications of synthetic data generation, which may lead to positive impacts. Our experiments demonstrate that the proposed method achieves state-of-the-art performance in the acceleration of the generative model. However, proper controls may be needed whenever applying our method to tasks that involve sensitive data data. There may be other potential societal consequences of our work, none of which we feel must be specifically highlighted here.

\section*{Limitations}

\begin{itemize}[leftmargin=*]

\item The scope of the empirical claims is limited to the text domain with non-auto regressive setting. The applicability and performance of \method{} for other tasks like audio and image generation, as well as with other architectures like auto-regressive GPT models, are not explored and left as future work.

\item While \method{}-C, the infinite-step sampling algorithm, offers new insights into bridging the gap between discrete and continuous-time processes for discrete diffusion models, the sample quality is not guaranteed to be superior to the accelerated algorithm with $1000$ steps. Some intuitions here: the assumption that the neural network can be optimally trained is an ideal case and is often not realized in practice. There is an inherent estimation error associated with the training process. As the number of steps increases, these estimation errors can accumulate, potentially leading to a degradation in performance. This cumulative estimation error might explain why using an infinite number of steps does not necessarily yield better results than a finite number of steps like 1000 in the conditional generation experiments. How to further improve sample quality of infinite steps is interesting but beyond the scope of this paper.

\item This paper focuses on the comparison with discrete Markov diffusion models since it aims to propose an accelerated algorithm for discrete diffusion with DNDM. Other text generation models, such as continuous diffusion models or auto-regressive models, are not considered in this paper.

\item This paper focuses on acceleration while maintaining good sample quality. The hyper parameter regions with poor sample qualities are not explored in this paper.

\end{itemize}

By highlighting these limitations, this paper aims to clearly scope its contributions and spark future work on addressing these important challenges with discrete diffusion models for generative modeling.

\section{Related Work}
\noindent\textbf{Continous Diffusion Models.} Generative modeling via continuous-time stochastic process has been investigated thoroughly in a series of work \citep{movellan2008contrastive, lyu2012interpretation, sohl2009minimum, bengio2014deep, alain2016gsns, alias2017variational, bordes2017learning}. The two lines of probabilistic modeling, \textit{denoising diffusion probabilistic model} \citep{sohl2015deep, ho2020denoising} and \textit{score matching with Langevin dynamics} \citep{song2019generative} are unified by \citet{song2020score} through introducing the SDE framework for SGM. Based on it, subsequent works \citep{dockhorn2021score, nachmani2021non, vahdat2021score} introduced a more complex diffusion process to improve the generation speed and quality. On the other hand, the score-based sampling process is time-consuming and has attracted much attention for improvements in speed \citep{san2021noise, watson2021learning, kong2021fast, karras2022elucidating, song2023consistency}. ``Gotta go fast'' (GGF), an SDE solver with adaptive step size tailored to SGM, is proposed in \citet{jolicoeur2021gotta}. \citet{song2020denoising} introduced a non-Markov diffusion process that corresponds to a deterministic sampling process, enabling the generation of high-quality samples more rapidly. \citet{dockhorn2022genie, liu2022pseudo} proposed a high-order SDE/ODE solver to achieve lower discretization error. \citet{lu2022dpm, zhang2022fast} leveraged the semi-linear structure of reverse ODE to reduce the discretization error and achieve state-of-the-art sampling speed. 

\noindent\textbf{Discrete Diffusion Models.} Research on discrete diffusion models was initiated by \citet{sohl2015deep}, who investigated diffusion processes over binary random variables. The methodology was expanded upon by \citet{ho2020denoising}, integrating categorical random variables through transition matrices with uniform probabilities. Though \citet{song2020denoising} suggested a similar extension in their supplementary content, they abstained from experimenting with this model type. Later on, \citet{austin2021structured} unveiled a more intricate framework for diffusion concerning categorical random variables, enhancing the discrete diffusion models by merging them with Masked language models (MLMs). Contemporary research has furthered this domain by introducing features like editing-based operations \citep{jolicoeur2021gotta, reid2022diffuser}, auto-regressive diffusion models \citep{hoogeboom2021autoregressive, ye2023diffusion}, the evolution of a continuous-time structure \citep{campbell2022continuous}, and the exploration of neural network analogs for learning \citep{sun2022score}. Additionally, \citet{zheng2023reparameterized} introduced a re-parameterized loss and an associated sampling technique, attaining commendable outcomes in fewer iterations. Our contributions run parallel to these aforementioned studies.

\section{Additional details of Discrete Diffusion}\label{sec:morebackground}
In our paper, we treat all the $\xb, \qb_{\mathrm{noise}}$ as a row vector and treat $\ind$ as a column vector with all elements equal $1$.

\subsection{Comparison between D3PM and DNDM}\label{app:ddim}
In Section~\ref{sec: De}, we introduced two different diffusion processes, the Markov process in \eqref{eq:1} and the non-Markov process in \eqref{eq:2}. In this section, we explain why they are different but result in the same joint distribution of $(\xb_0, \xb_t)$ for every time step $t$. Since $\qb(\xb_0)$ keeps the same, we only need to prove that the conditional distribution $\qb(\xb_{t}|\xb_{0})$ is the same for the two processes.

\noindent\textbf{Markov Process.} \ref {eq:1} is a Markov process since $\wb_n$ is independent with $\xb_{t-1}, \ldots, \xb_0$, so $\xb_{t}$ is independent of all the past states given the present state. This can also be inferred from the following distribution, which does not depend on $\xb_0, \ldots, \xb_{t-2}$,
\begin{align}
q(\xb_{t}|\xb_{t-1}) = \mathrm{Cat}\big(\xb_{t}; \pb = \beta_{t}\xb_{t-1} + (1- \beta_{t})\qb_{\mathrm{noise}}\big) \label{simpleQ}. 
\end{align}
Denote $\Qb_{t}:= \beta_{t}\Ib + (1-\beta_t)\ind \qb_{\mathrm{noise}}$, then we have that 
\begin{align*}
\xb_{t-1}\Qb_t  = \beta_{t}\xb_{t-1} + (1-\beta_t)\xb_{t-1}\ind \qb_{\mathrm{noise}} = \beta_{t}\xb_{t-1} + (1- \beta_{t})\qb_{\mathrm{noise}},
\end{align*}
where the last equality holds due to the fact that $\xb_{t-1}$ is a one hot vector and thus $\xb_{t-1}\ind = 1$. Therefore, we can rewrite
\eqref{simpleQ} as $q(\xb_{t}|\xb_{t-1}) =  \mathrm{Cat}\big(\xb_{t}; \pb = \xb_{t-1}\Qb_t  \big)$. Then, it is a Markov process with transition kernel $\Qb_t$. So $q(\xb_{t}|\xb_{0}) = \mathrm{Cat}\big(\xb_{t}; \pb = \xb_{0}\Qb_{0}\ldots\Qb_{t} \big)$ \citep{austin2021structured}. We can then have that
\begin{align*}
\Qb_{0}\ldots\Qb_{t} &= [  \beta_{0}\Ib + (1-\beta_0)\ind \qb_{\mathrm{noise}}]\ldots [  \beta_{t}\Ib + (1-\beta_t)\ind \qb_{\mathrm{noise}}]\\
&= \Pi_{s=0}^{t}\beta_{s}\Ib  + (1- \Pi_{s=0}^{t}\beta_{s})\ind \qb_{\mathrm{noise}},
\end{align*}
where the last equality holds since identity matrix $\Ib$ multiplying any vector equals the vector itself and $\ind \qb_{\mathrm{noise}}\ind \qb_{\mathrm{noise}} = \ind (\qb_{\mathrm{noise}}\ind )\qb_{\mathrm{noise}} = \ind \qb_{\mathrm{noise}}$. Therefore, we have that 
\begin{align*}
q(\xb_{t}|\xb_{0}) = \mathrm{Cat}\big(\xb_t; \pb = \Pi_{s=0}^{t}\beta_{s}\xb_{0} + (1- \Pi_{s=0}^{t}\beta_{s})\qb_{\mathrm{noise}}\big)  = \mathrm{Cat}\big(\xb_t; \pb = \alpha_{t}\xb_{0} + (1- \alpha_{t})\qb_{\mathrm{noise}}\big),
\end{align*}
where the last equality holds due to the definition $\alpha_{t} = \Pi_{s=0}^{t}\beta_{s}$. This gives rise to why the Markov process \eqref{eq:1} results in conditional distribution $q(\xb_{t}|\xb_{0}) = \mathrm{Cat}\big(\xb_t; \pb = \alpha_{t}\xb_{0} + (1- \alpha_{t})\qb_{\mathrm{noise}}\big)$. 

\noindent\textbf{Non-Markov Process.} Recall that our DNDM is defined by 
\begin{align*}
\xb_{t} = b_{t}\xb_{t-1} + (1-b_{t})\wb,
\end{align*}
where $\wb$ is fixed for any time $t$. Therefore, $\wb$ is no longer independent with $\xb_{0}, \ldots, \xb_{t-1}$. Therefore, we can't define the transition kernel and compute $\qb(\xb_{t}|\xb_0)$ by using the property of Markov. Therefore, we need to advance the technique to calculate the conditional distribution.

\begin{proof}[Proof of Theorem~\ref{thm:same distribution}]
By \eqref{eq:2}, we can derive the following explicit expression for a recursive sequence,
\begin{align*}
\xb_{t} &= b_{1}\ldots b_{t}\xb_{0, n} + \sum_{s=1}^{t}(1-b_{s})b_{s+1}\ldots b_{t}\wb\\
&= b_{1}\ldots b_{t}\xb_{0} + (1- b_{1}\ldots b_{t} )\wb\\
&= a_{t}\xb_{0} +  (1- a_{t}) \wb,
\end{align*}
where second equality is by cancellation of terms, the last inequality holds by defining $a_{t} = b_{1}\ldots b_t$. Since $a_t$ either equals to $1$ or $0$. Besides, $a_{t}$ equals $1$ if and only if $b_{1}=b_{2}=\ldots=b_{t}=1$, so we have that $a_{t}$ follows Bernoulli distribution $\mathrm{Bernoulli }(\beta_{1}\ldots \beta_{t}) = \mathrm{Bernoulli }(\alpha_t)$ where $\alpha_t = \Pi_{i=1}^{t}\beta_s$. Therefore, we can conclude that $\qb(\xb_{t}|\xb_0) = \mathrm{Cat}\big(\xb_t; \pb = \alpha_{t}\xb_{0} + (1- \alpha_{t})\qb_{\mathrm{noise}}\big)$, which completes the proof.
\end{proof}

\noindent\textbf{Comparison between D3PM-Absorb and DNDM.} Recall the forward processes of D3PM and DNDM as follows:
\begin{align*}
\mathrm{D3PM:} \quad \xb_{t} &= b_{t}\xb_{t-1} + (1-b_t)\wb_t, \quad \forall t = 1 \dots T,   \\
\mathrm{DNDM:} \quad \xb_{t} &= b_{t}\xb_{t-1} + (1-b_t)\wb, \quad \forall t = 1 \dots T.
\end{align*}
For absorbing diffusion where $\wb = \mathrm{[Mask]}$, DNDM's forward process becomes equivalent to D3PM since $\wb_t = \wb = \mathrm{[Mask]}$ in this special case. However, for multinomial diffusion or other diffusion processes where $\wb_t \neq \wb$, these two processes exhibit different behaviors. In addition, even for absorbing diffusion, our proposed reverse sampling algorithm for DNDM  is still different from that for D3PM. 

To elucidate the key differences between the sampling algorithm in DNDM and that in D3PM for absorbing diffusion, let's directly compare the algorithms:
\begin{itemize}[leftmargin=*]
\item For the D3PM-Absorb algorithm: We begin with an all $\mathrm{[Mask]}$ sequence. At each time step $t$, we sample $\xb_0 \sim p_{\theta}(\xb_0|\xb_t)$. If $\xb_t=\mathrm{[Mask]}$, $\xb_{t-1}$ transitions to $\mathrm{[Mask]}$ with probability $(1-\alpha_{t-1})/(1-\alpha_t)$ and to $\xb_0$ with probability $(\alpha_{t-1} - \alpha_t)/(1-\alpha_t)$. If $x_{t}\not= \mathrm{[Mask]}$, it remains unchanged. 
\item For the DNDM-Absorb algorithm: We also start with an all $\mathrm{[Mask]}$ sequence, but crucially, we first determine the transition time set. During sampling, if $\xb_t=\mathrm{[Mask]}$, the transition probabilities for $\xb_{t-1}$ are identical to D3PM. However, we only sample $\xb_0 \sim p_{\theta}(\xb_0|\xb_t)$ when at least one token needs to change, as determined by our pre-computed transition set. This selective sampling is the key to our algorithm's efficiency.
\end{itemize}

Therefore, you can see that DNDM will skip many steps during the sampling process to avoid function evaluation and save computational cost. Even though the forward process of DNDM is the same as that of D3PM for absorbing diffusion, our DNDM approach introduces an algorithm design in the sampling process by pre-computing the transition time set and selectively applying function evaluations. This distinguishes DNDM from D3PM algorithm, offering a more computationally efficient approach to inference in discrete diffusion.

\noindent\textbf{Comparison between DDIM and DNDM for Multinomial Diffusion.}
While there are similarities between DNDM and DDIM (Appendix A), they are fundamentally different models, and DNDM is not a special case of DDIM. DNDM introduces a novel framework specifically designed for discrete spaces, while DDIM was originally developed for continuous diffusion models. The key differences for multinomial diffusion are as follows.

\begin{itemize}[leftmargin=*]
\item DDIM: Following \citet{song2020denoising} (eq. 19 in Appendix A), $q(\xb_{t-1}|\xb_t,\xb_0)=\text{Cat}(\sigma_t \xb_t + (\alpha_{t-1}-\sigma_t \alpha_t)\xb_0+((1-\alpha_{t-1})-(1-\alpha_t)\sigma_t)\mathbf{1}_K)$. Even with $\sigma_t = \frac{1-\alpha_{t-1}}{1-\alpha_t}$, the process remains stochastic: $q(\xb_{t-1}|\xb_t,\xb_0)=\text{Cat}(\sigma_t \xb_t + (1-\sigma_t)\xb_0)$. This means at every step, there's a probability of choosing $\xb_0$, regardless of whether it has transitioned to $\xb_0$ or not. Unlike Absorbing discrete diffusion, no $\mathrm{[Mask]}$ exists in multinomial diffusion. Therefore, DDIM cannot distinguish whether $\xb_t$ already equals $\xb_0$ or not.  In particular, although the sampling process becomes less stochastic in the DDIM setting, it will still be predicted $\xb_0$ with high probability $1-\sigma_t = \frac{\alpha_{t-1}- \alpha_t}{1-\alpha_t}$.  
\item  DNDM: Achieves full de-randomization using transition time $\tau$, where:
\begin{align}
\xb_{t-1} = \ind(\tau = t)\xb_0 + \ind(\tau \not= t)\xb_{t}, \quad \text{with } P(\tau = t) = \alpha_{t-1} - \alpha_t.
\end{align} 
This crucial difference allows DNDM to achieve full de-randomization once $\tau$ is sampled, leading to a deterministic evolution that DDIM cannot achieve.
\end{itemize}
While DNDM and DDIM are both non-Markov models for multinomial diffusion, their fundamental approaches to and achievements in de-randomization differ significantly in discrete spaces.

\subsection{Training Objective}
\citet{hoogeboom2021argmax} utilized $L_{t}$  derived from the negative variational bound. In detail, 
\begin{align}
L_{t} = \mathrm{KL}\big(\mathrm{Cat}(\xb; \pb = \btheta_{\mathrm{post}}(\xb_t, \xb_0)\big|\mathrm{Cat}(\xb; \pb =\btheta_{\mathrm{post}}(\xb_t, \hat{\xb}_0)\big),   \label{eq:complex}
\end{align}
where $\hat{\xb}_{0} \sim p_{\btheta}(\cdot| \xb_t)$, $\btheta_{\mathrm{post}} = (\beta_{t}\xb_t + (1-\beta_{t})/K \ind^{\top})\odot (\alpha_{t-1}\xb_0 + (1-\alpha_{t-1})/K \ind^{\top})$ and $\btheta_{\mathrm{post}} = (\beta_{t}\xb_t + (1-\beta_{t})/K \ind^{\top})\odot (\alpha_{t-1}\hat{\xb}_0 + (1-\alpha_{t-1})/K \ind^{\top})$. This loss evolves KL divergence between two categorical distributions. 

Building on this foundation, \citet{austin2021structured} introduced an auxiliary denoising objective to strengthen the data predictions $\xb_0$ at each time step. In detail, the auxiliary objective is as follows,
\begin{align*}
\EE_{q(\xb_{t}, \xb_0)}\Big[-\log p_{\btheta}(\xb_0|\xb_t)\Big],    
\end{align*}
where the auxiliary loss term is minimized exactly when $p_{\theta}( \cdot |\xb_{t})$ has all its mass
on the data point $\xb_0$.

Furthering the advancements, \citet{zheng2023reparameterized} put forth a reparametrized loss $L_{t}$  that incorporates a re-weighted parameter $\lambda_t$. The detailed loss is 
\begin{align*}
\overline{L}_{t} = \lambda_{t-1}\EE_{\xb_{t-1},\xb_{t}\sim q(\cdot|\xb_0)}  \text{KL}( q(\xb_{t-1} | \xb_{t}, \xb_{0})|p_\theta^{(t)}(\xb_{t-1} | \xb_{t})). 
\end{align*}
This loss can be related to the standard multi-class cross-entropy loss function, which is also simple and powerful. That's why we consider \citet{zheng2023reparameterized} as the baseline model.

In Section~\ref{sec:continuous}, we consider the continuous-time forward and backward process. Based on that, we were motivated to analyze the infinite limit of the average loss $\lim_{t \rightarrow \infty} \frac{1}{T}\sum_{t=1}^{T}L_{t}$. We find that the new loss can provide a better checkpoint than the loss averaged on the finite step on some tasks. 
{\color{black}
\subsection{Calculation of the Evidence Lower Bound}\label{subsubsec: finite time ELBO}
\subsubsection{Finite Time DNDM}
In this section, we derive the evidence lower bound (ELBO) for our model. The derivatives are inspired by the reasoning in DDIM \citep{song2020denoising}. 
Specifically, We denote the generative process as $p_\theta(\xb_{0:T}|\tau) = p_\theta^{(T)}(\xb_T|\tau) \prod_{t=1}^T p_\theta^{(t)}(\xb_{t-1}|\xb_{t},\tau)$. Here, $p_\theta^{(T)}$ is the pure noise and $p_\theta^{(t)}(\xb_{t-1}|\xb_t,\tau)=q(\xb_{t-1}|\xb_t,\hat{\xb}_0,\tau)$, where $\hat{\xb}_0$ is given by a neural network $p_{\btheta}$, i.e., $\hat{\xb}_0=p_{\btheta}(\xb_t,t)$. Notice that by Jensen's inequality, 
\begin{align}
\log p_\theta(\xb_0) = \log\EE_{\tau \sim \cD_{\tau}}[p_\theta(\xb_0 | \tau)] \geq \EE_{\tau \sim \cD_{\tau}}[\log p_\theta(\xb_0 | \tau)]. \label{eq:ELBO1} 
\end{align}

The evidence lower bound inequality gives
\begin{align}
\log p_\theta(\xb_0 | \tau) \ge \EE_{\xb_{1:T} \sim q(\xb_{1:T}|\xb_0,\tau)} \log \frac{p_{\theta}(\xb_{0:T}|\tau)}{q(\xb_{1:T}|\xb_0,\tau)}. \label{eq:ELBO2} 
\end{align}
Plugging \eqref{eq:ELBO2} into \eqref{eq:ELBO1} gives the following ELBO,
\begin{align*}
\log p_\theta(\xb_0) \ge \EE_{\tau \sim \cD_{\tau}}\EE_{\xb_{1:T} \sim q(\xb_{1:T}|\xb_0,\tau)} \log \frac{p_{\theta}(\xb_{0:T}|\tau)}{q(\xb_{1:T}|\xb_0,\tau)} := \text{ELBO}. 
\end{align*}

We factorize the $p_\theta$ and $q$ by
\begin{align*}
 p_\theta(\xb_{0:T}|\tau) & =  p_\theta^{(T)}(\xb_T | \tau) \prod_{t=1}^T p_\theta^{(t)}(\xb_{t-1} | \xb_{t}, \tau), \\
 q(\xb_{1:T}|\xb_0,\tau) & = q(\xb_T | \xb_0, \tau) \prod_{t=2}^T q(\xb_{t-1} | \xb_{t}, \xb_{0}, \tau).
\end{align*}
Here $q$ admits such a decomposition due to our definition of the diffusion process in \eqref{eq:2}, which introduce the following reverse process:
\begin{align*}
    \xb_{t-1} = \ind(\tau=t)\xb_{0} + \ind(\tau\not= t)\xb_{t}.
\end{align*}
Therefore, $\xb_{1:T}$ is Markovian when conditioned on $\xb_{0}$ and $\tau$. Based on the factorization, we have
\begin{align*}
     \text{ELBO} &= \EE_{\tau \sim \cD_{\tau}}\EE_{\xb_{1:T} \sim q(\xb_{1:T}|\xb_0,\tau)} \Big[ \log p_\theta^{(T)}(\xb_T | \tau)  + \sum_{t=1}^T \log p_\theta^{(t)}(\xb_{t-1} | \xb_{t}, \tau) 
     \\
     & \qquad  - \log q(\xb_T | \xb_0, \tau) - \sum_{t=2}^T \log q(\xb_{t-1} | \xb_{t}, \xb_{0}, \tau) \Big] \\
     & = \EE_{\tau \sim \cD_{\tau}}\EE_{\xb_{1:T} \sim q(\xb_{1:T}|\xb_0,\tau)} \Big[ \log p_\theta^{(1)} (\xb_{0} | \xb_{1}, \tau)  + \sum_{t=2}^{T} \log \frac{p_\theta^{(t)}(\xb_{t-1} | \xb_{t}, \tau)}{q(\xb_{t-1} | \xb_{t}, \xb_{0}, \tau)} \\
     &\qquad   + \log \frac{p_\theta^{(T)}(\xb_T | \tau)}{q(\xb_T | \xb_0, \tau)}\Big] \\
     & = \EE_{\tau \sim \cD_{\tau}}\EE_{\xb_1 \sim q(\cdot|\xb_0,\tau)} \log p_\theta^{(1)} (\xb_{0} | \xb_{1}, \tau)  \\
     & \qquad + \sum_{t=2}^{T} \EE_{\xb_{t-1},\xb_{t}\sim q(\cdot|\xb_0,\tau)} \log \frac{p_\theta^{(t)}(\xb_{t-1} | \xb_{t}, \tau)}{q(\xb_{t-1} | \xb_{t}, \xb_{0}, \tau)} + \text{const} \\
     & = \EE_{\tau \sim \cD_{\tau}} \underbrace{\EE_{\xb_1 \sim q(\cdot|\xb_0,\tau)} \log p_\theta^{(1)} (\xb_{0} | \xb_{1}, \tau)}_{\overline{\cL}_1}  \\
     & \qquad  - \sum_{t=2}^T\EE_{\tau \sim \cD_{\tau}} \underbrace{\EE_{\xb_{t-1},\xb_{t}\sim q(\cdot|\xb_0,\tau)} \text{KL}( q(\xb_{t-1} | \xb_{t}, \xb_{0}, \tau)|p_\theta^{(t)}(\xb_{t-1} | \xb_{t}, \tau))}_{\overline{\cL}_t} + \text{const}.
\end{align*}



By a slight abuse of notations we use $q(\xb_{t-1}|\xb_t,\xb_0),p_\theta^{(t)}(\xb_0|\xb_1)$ to indicate the distribution of the diffusion process defined in \citet{zheng2023reparameterized}, that is, the standard Markov discrete diffusion process.
In particular, we have
\begin{align*}
    & \overline{\cL}_1 = \left\{ \begin{array}{lc}
        \EE_{\xb_1 \sim q(\cdot|\xb_0)} \log p_\theta^{(1)} (\xb_{0} | \xb_{1}), & \tau = 1, \\
         \text{const}, & \tau \neq 1.
    \end{array} \right. \\
    & \overline{\cL}_t = \left\{ \begin{array}{lc}
       \EE_{\xb_{t-1},\xb_{t}\sim q(\cdot|\xb_0)} \text{KL}( q(\xb_{t-1} | \xb_{t}, \xb_{0})|p_\theta^{(t)}(\xb_{t-1} | \xb_{t})), &\tau = t, \\
         0, &\tau \neq t.
    \end{array} \right.
\end{align*}
Thus, we can obtain that
\begin{align*}
 \text{ELBO} = & \PP(\tau = 1)  \cdot \underbrace{\EE_{\xb_1 \sim q(\cdot|\xb_0)} \log p_\theta^{(1)} (\xb_{0} | \xb_{1})}_{\cL_1} \\
 & -  \sum_{t=2}^T \PP(\tau = t) \cdot \underbrace{\EE_{\xb_{t-1},\xb_{t}\sim q(\cdot|\xb_0)}  \text{KL}( q(\xb_{t-1} | \xb_{t}, \xb_{0})|p_\theta^{(t)}(\xb_{t-1} | \xb_{t}))}_{\cL_t}+ \text{const}.
\end{align*}

Here $\cL_t$ matches the loss terms in \citet{zheng2023reparameterized}. In the practical training process, \citet{zheng2023reparameterized} samples $t$ from $\text{Unif}\{1,\cdots,T \}$ in each iteration and optimizes $\lambda_t \cdot \cL_t$, where $\lambda_t$'s are weights. Thus, when we sample $\tau$ and optimize $\cL_\tau$, our ELBO indeed leads to the same training objective as \citet{zheng2023reparameterized} up to reweighting. Since \citet{zheng2023reparameterized} is a parametrization of existing works \citep{austin2021structured,hoogeboom2021argmax}, our training objective indeed aligns with previous discrete diffusion models.
\subsubsection{Continous Time DNDM}
In Section~\ref{subsubsec: finite time ELBO}, we derived an ELBO for DNDM and its accelerated algorithm defined in Section~\ref{sec: De} and \ref{sec:fast}. While for finite sampling steps, we can decompose the diffusion process via the sampling steps $1, \ldots, T$ in \eqref{eq:ELBO2}, it becomes intractable for continuous Time DNDM (Infinite steps $T\rightarrow \infty$). Therefore, we can formulate the ELBO of continuous time DNDM by decomposing the transition times. The idea of decomposition of transition times follows \citet{hoogeboom2021autoregressive}, but their proof is only applicable to absorbing discrete diffusion, while ours can deal with discrete diffusion with various noise $q_{\mathrm{noise}}$ including multinomial diffusion.

In Section~\ref{subsubsec: finite time ELBO}, we only consider the case of a single token $\xb \in \RR^{K}$ for simplicity as we decompose with the sampling steps $T$. In this section, we decompose over the transition time $\tau$. Therefore, we need to consider a sentence with multiple tokens $\xb_{t, 1:N} = [\xb_{t,1}, \ldots, \xb_{t,N}]$ where $\xb_{t,n}$ is the $n$-th token and $N$ is the sequence length. Recall that we defined the transition time set $\cT = \{\tau_{n}\}_{n=1}^{N}$ in Section~\ref{sec:fast}. We arrange $\tau_n$ to obtain an ordered sequence $\tau_{n_k}$, where $0 = \tau_{n_0} < \tau_{n_1} < \tau_{n_2} < \ldots <\tau_{n_N} = T$. Then conditioning on the transition time set $\cT = \{\tau_{1}, \ldots, \tau_{N}\}$, we have that 
\begin{align*}
p_\theta(\xb_{0:T, 1:N}|\cT) = p_\theta(\xb_{\tau_{n_{N}}, 1:N}|\cT) \prod_{s=N,\ldots, 1} p_\theta(\xb_{\tau_{n_{s-1}}, 1:N}|\xb_{\tau_{n_s},1:N}, \cT),
\end{align*}
where we omit the time superscript of $p$ for simplicity. Then, the evidence lower bound inequality gives
\begin{align}
\log p_\theta(\xb_{0,1:N} | \cT) \ge \EE_{\xb_{\tau_{n_1}:T,1:N} \sim q(\xb_{\tau_{n_1}:T,1:N}|\xb_{0,1:N},\cT)} \log \frac{p_{\theta}(\xb_{0:T,1:N}|\cT)}{q(\xb_{\tau_{n_1}:T,1:N}|\xb_{0,1:N},\cT)}. \label{eq:InfELBO2} 
\end{align}
By Jensen's inequality, we have
\begin{align}
\log p_\theta(\xb_{0, 1:N}) = \log\EE_{\tau_{1}, \ldots, \tau_{n} \sim \cD_{\tau}}[p_\theta(\xb_{0, 1:N} | \cT)] \geq \EE_{\tau_{1}, \ldots, \tau_{n}  \sim \cD_{\tau}}[\log p_\theta(\xb_0 | \cT)]. \label{eq:InfELBO1} 
\end{align}

Plugging \eqref{eq:InfELBO2} into \eqref{eq:InfELBO1} gives the following ELBO,
\begin{align*}
\log p_\theta(\xb_{0,1:N}) \ge \EE_{\tau_{1}, \ldots, \tau_{n}  \sim \cD_{\tau}}\EE_{\xb_{\tau_{n_1}:T} \sim q(\xb_{\tau_{n_1}:T}|\xb_0,\cT)} \log \frac{p_{\theta}(\xb_{0:T}|\cT)}{q(\xb_{\tau_{n_1}:T}|\xb_0,\cT)} := \text{ELBO}. 
\end{align*}

We factorize the $p_\theta$ and $q$ by
\begin{align*}
 p_\theta(\xb_{0:T, 1:N}|\cT) & =  p_\theta(\xb_{T, 1:N} | \cT) \prod_{s=N,\ldots,1} p_\theta(\xb_{\tau_{n_{s-1}}, 1:N} | \xb_{\tau_{n_s}, 1:N}, \cT), \\
 q(\xb_{\tau_{n_1}:T, 1:N}|\xb_{0,1:N},\cT) & = q(\xb_{T, 1:N} | \xb_0, \cT) \prod_{s=N,\ldots,2} q(\xb_{\tau_{n_{s-1}}, 1:N} | \xb_{\tau_{n_s}, 1:N}, \xb_{0, 1:N}, \cT).
\end{align*}
Therefore, we have
\begin{align}
    \text{ELBO} &= \EE_{\tau_{1}, \ldots, \tau_{n}  \sim \cD_{\tau}}\EE_{\xb_{\tau_{n_1}:T,1:N} \sim q(\xb_{\tau_{n_1}:T,1:N}|\xb_{0,1;N},\cT)} \Big[ \log p_\theta(\xb_{T, 1:N} | \cT) \notag\\
    &\qquad+ \sum_{s=1}^{N} \log p_\theta(\xb_{\tau_{n_{s-1}}, 1:N} | \xb_{\tau_{n_{s}}, 1:N}, \cT) - \log q(\xb_{T, 1:N} | \xb_{0, 1:N}, \cT) \notag\\
    &\qquad- \sum_{s=2}^{N} \log q(\xb_{\tau_{n_{s-1}}, 1:N} | \xb_{\tau_{n_s}, 1:N}, \xb_{0, 1:N}, \cT) \Big] \notag\\
     & = \EE_{\tau_{1}, \ldots, \tau_{n}  \sim \cD_{\tau}}\EE_{\xb_{\tau_{n_1}:T, 1:N} \sim q(\xb_{\tau_{n_1}:T, 1:N}|\xb_{0, 1:N},\cT)} \Big[ \log p_\theta (\xb_{0, 1:N} | \xb_{1, 1:N}, \cT)\notag\\
     &\qquad + \sum_{s=2}^{N} \log \frac{p_\theta(\xb_{\tau_{n_{s-1}}, 1:N} | \xb_{\tau_{n_{s}}, 1:N}, \cT)}{q(\xb_{\tau_{n_{s-1}}, 1:N} | \xb_{\tau_{n_{s}}, 1:N}, \xb_{0, 1:N}, \cT)}  + \log \frac{p_\theta(\xb_{T, 1:N} | \cT)}{q(\xb_{T, 1:N} | \xb_{0, 1:N}, \cT)}\Big] \notag \\
     & = \EE_{\tau_{1}, \ldots, \tau_{n}  \sim \cD_{\tau}}\EE_{\xb_{1, 1:N} \sim q(\cdot|\xb_{0, 1:N},\cT)} \log p_\theta (\xb_{0, 1:N} | \xb_{1, 1:N}, \cT) \notag\\
     &\qquad + \sum_{s=2}^{N} \EE_{\xb_{\tau_{n_{s-1}}, 1:N},\xb_{\tau_{n_{s}}, 1:N}\sim q(\cdot|\xb_{0, 1:N},\cT)} \log \frac{p_\theta(\xb_{\tau_{n_{s-1}}, 1:N} | \xb_{\tau_{n_{s}}, 1:N}, \cT)}{q(\xb_{\tau_{n_{s-1}}, 1:N} | \xb_{\tau_{n_{s}}, 1:N}, \xb_{0, 1:N}, \cT)} + \text{const} \notag\\
     & = \EE_{\tau_{1}, \ldots, \tau_{n}  \sim \cD_{\tau}}\EE_{\xb_{1, 1:N} \sim q(\cdot|\xb_{0, 1:N},\cT)} \log p_\theta (\xb_{0, 1:N} | \xb_{1, 1:N}, \cT) \notag\\
     &\qquad - \sum_{s=2}^{N}\EE_{\tau_{1}, \ldots, \tau_{n}\sim \cD_{\tau}}\EE_{\xb_{\tau_{n_{s-1}}, 1:N},\xb_{\tau_{n_{s}}, 1:N}\sim q(\cdot|\xb_{0, 1:N},\cT)} \notag \\
     &\qquad\text{KL}( q(\xb_{\tau_{n_{s-1}}, 1:N} | \xb_{\tau_{n_{s}}, 1:N}, \xb_{0, 1:N}, \cT)|p_\theta (\xb_{\tau_{n_{s-1}}, 1:N} | \xb_{\tau_{n_{s}}, 1:N}, \cT)) + \text{const}. \label{eq:Infinity ELBO}
\end{align}

\begin{remark}
\eqref{eq:Infinity ELBO} represents the ELBO utilized by the DNDM-C architecture. As our transition times $\tau_{n}$ are independently and identically drawn from the distribution $\cD_{\tau}$, we are unable to further decompose  \eqref{eq:Infinity ELBO} into a loss function related to the position information $1:N$, as was accomplished by \cite{hoogeboom2021autoregressive}.

\end{remark}

}

\section{Choice of the Transition Time}
Transition time $\tau$ in Definition~\ref{Def:transition time} plays an important role in \method{}. In this section, we provide a deeper discussion of the transition time. We first give a proof of the Theorem~\ref{thm:transition probability}.

\begin{proof}[Proof of Theorem~\ref{thm:transition probability}]
By the definition of $\tau$, we know that $\tau_{n} = t$ is equivalent to $b_{0,n} =1, \ldots, b_{t-1,n} = 1$ and $b_{t,n} = 0$. {\color{black} Since $\{ b_{t,n} \}_{t=0}^T$ is independent for different $n$ by definition, each $\tau_n$ is also independent.} Therefore, we drop the subscript $n$ for simplicity. On the other hand if $b_0 =1, \ldots, b_{t-1} = 1$ and $b_{t} = 0$ we can also conclude that $\tau = t$. Therefore, we have that
\begin{align*}
\mathbb{P}(\tau = t) &= \mathbb{P}(b_0 =1, \ldots, b_{t-1} = 1, b_{t}=0)\\
&=\big[\Pi_{s=1}^{t-1}\beta_{s}\big]\cdot (1-\beta_{t})\\
&= \Pi_{s=1}^{t-1}\beta_{s} - \Pi_{s=1}^{t}\beta_{s} \\
&= \alpha_{t-1} - \alpha_{t},    
\end{align*}
where the second equality is due to $b_s, s = 1,2,\ldots,t$ are independent random variable following $\mathrm{Bernoulli}(\beta_s)$ distribution and the last equality is by the definition of $\alpha_t = \Pi_{s=1}^{t}\beta_{s}$.
\end{proof}
Notice that $\alpha_{t}$ is a decreasing sequence in the $0$ to $1$ range. Therefore, $\mathbb{P}(\tau = t) \in [0,1]$ for any $t \in \{1, \ldots, T\}$. Besides $\sum \mathbb{P}(\tau = t) = \sum_{t=1}^{T}\big(\alpha_{t-1}-\alpha_{t}\big) = \alpha_0 - \alpha_{T} = 1$. Therefore, the derived distribution is valid as long as the $\alpha_{t}$ is decreasing from $1$ to $0$.

\begin{figure}[ht!]
  \centering
  
  \begin{subfigure}{0.45\textwidth}
    \centering
    \includegraphics[width=1.0\linewidth]{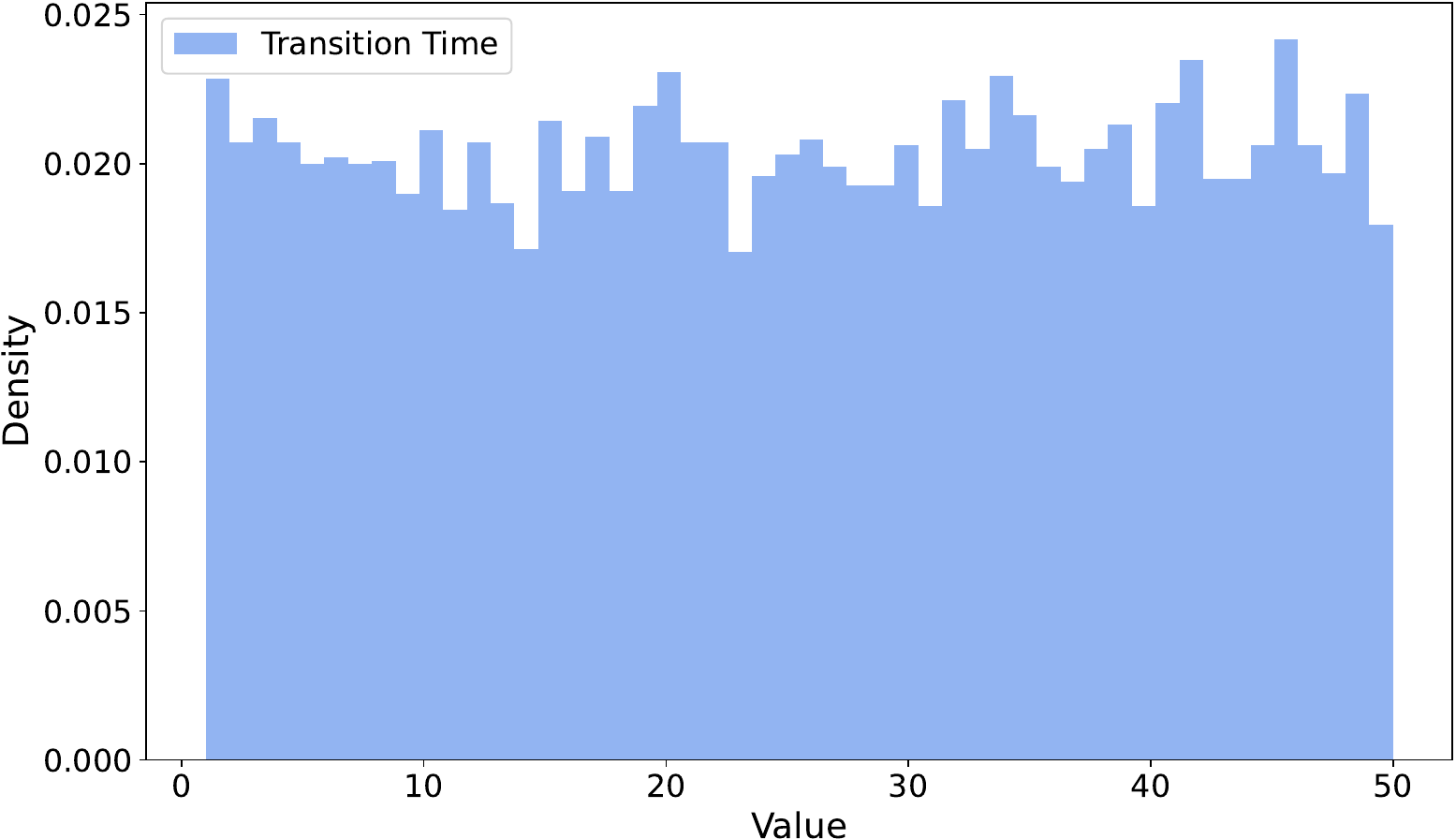}
    \caption{$\alpha_{t} = 1 - t/T$}
    \label{fig:transitiontime1}
  \end{subfigure}
  \hfill 
  \begin{subfigure}{0.45\textwidth}
    \centering
    \includegraphics[width=1.0\linewidth]{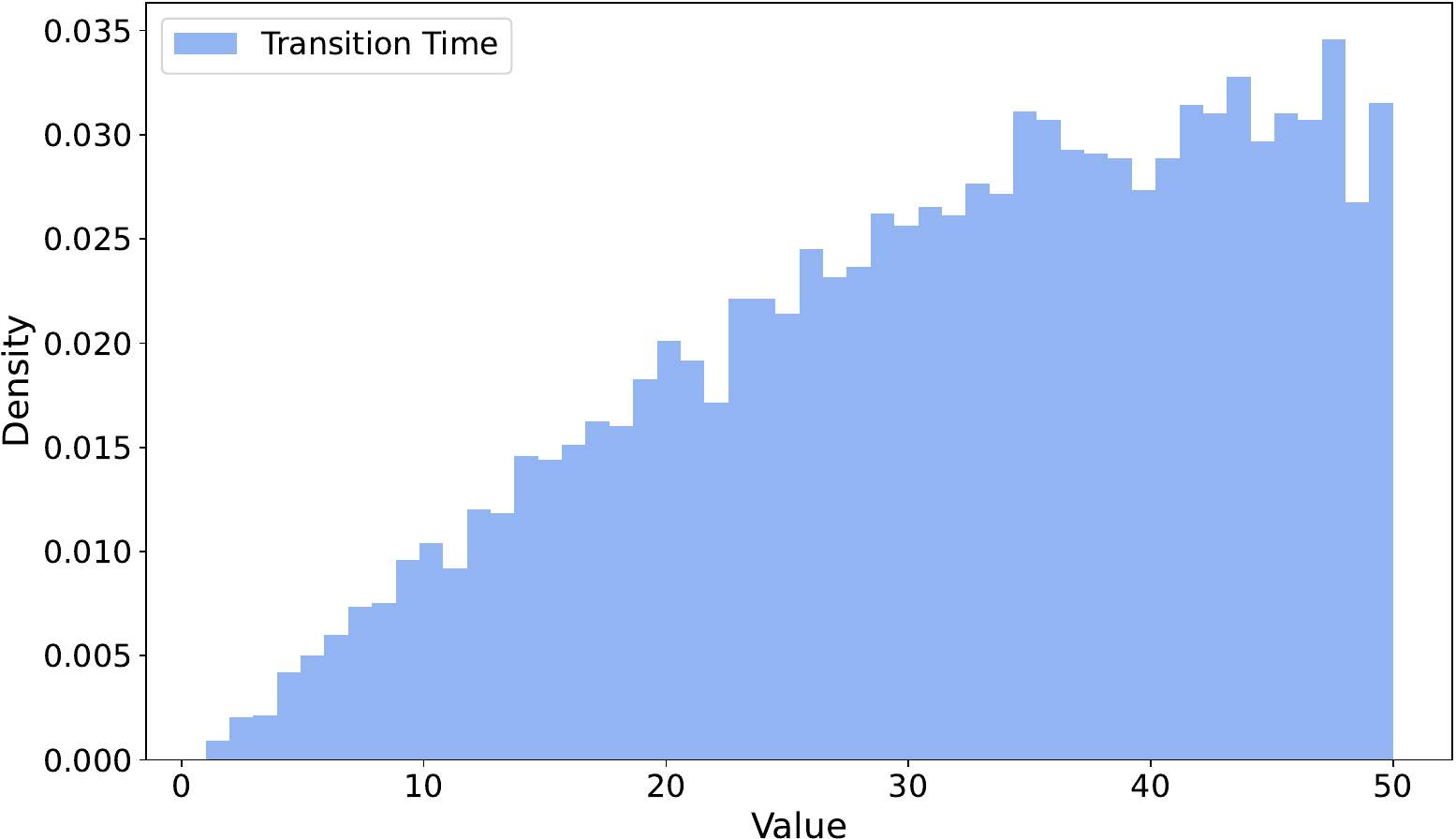}
    \caption{$\alpha_{t} = \cos(\pi*t/2T)$}
    \label{fig:transitiontime2}
  \end{subfigure}
  
  \vspace{0.5cm} 
  
  \begin{subfigure}{0.45\textwidth}
    \centering
    \includegraphics[width=1.0\linewidth]{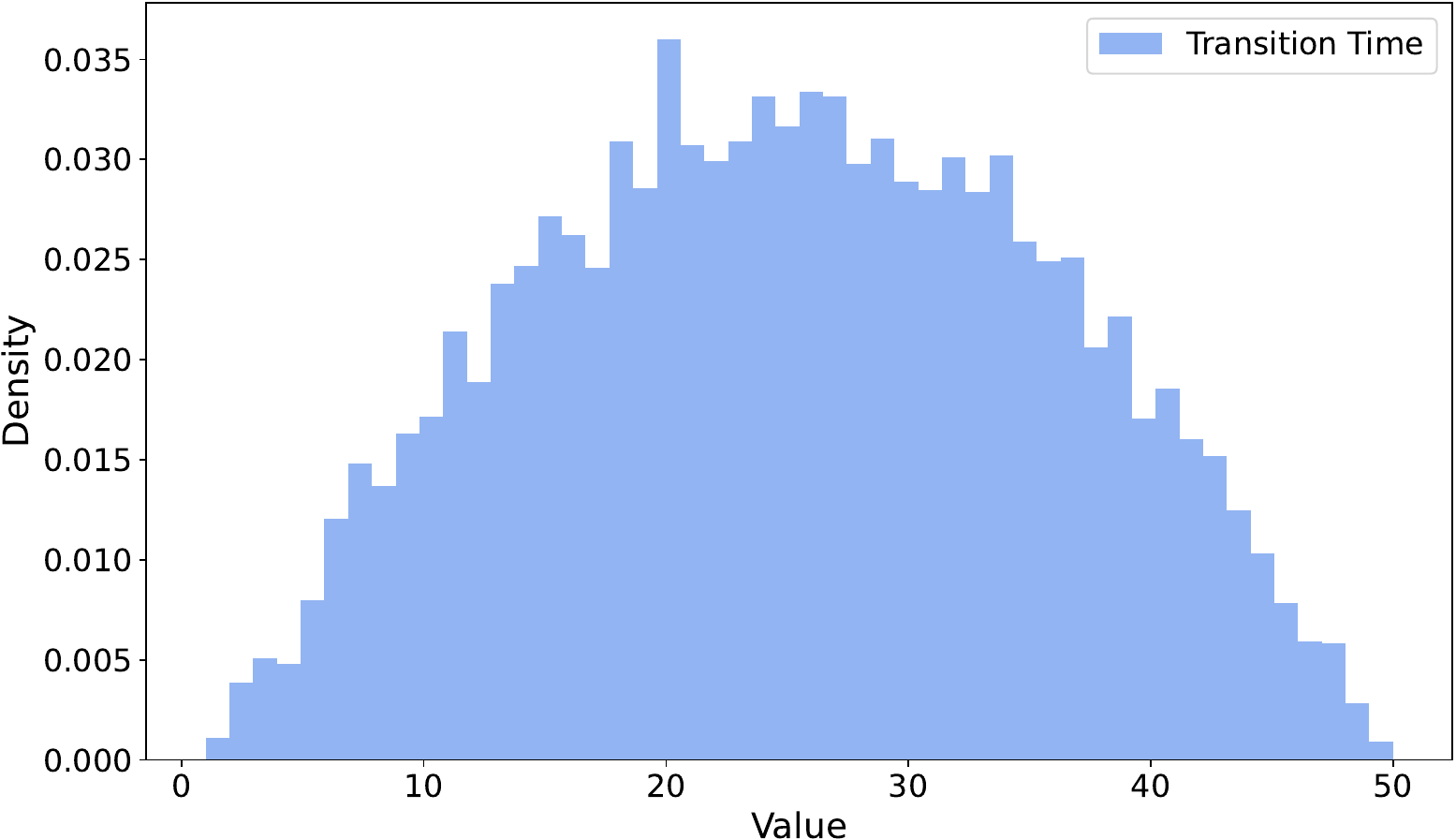}
    \caption{$\alpha_t = \cos^2(\pi*t/2T)$}
    \label{fig:transitiontime3}
  \end{subfigure}
  \hfill 
  \begin{subfigure}{0.45\textwidth}
    \centering
    \includegraphics[width=1.0\linewidth]{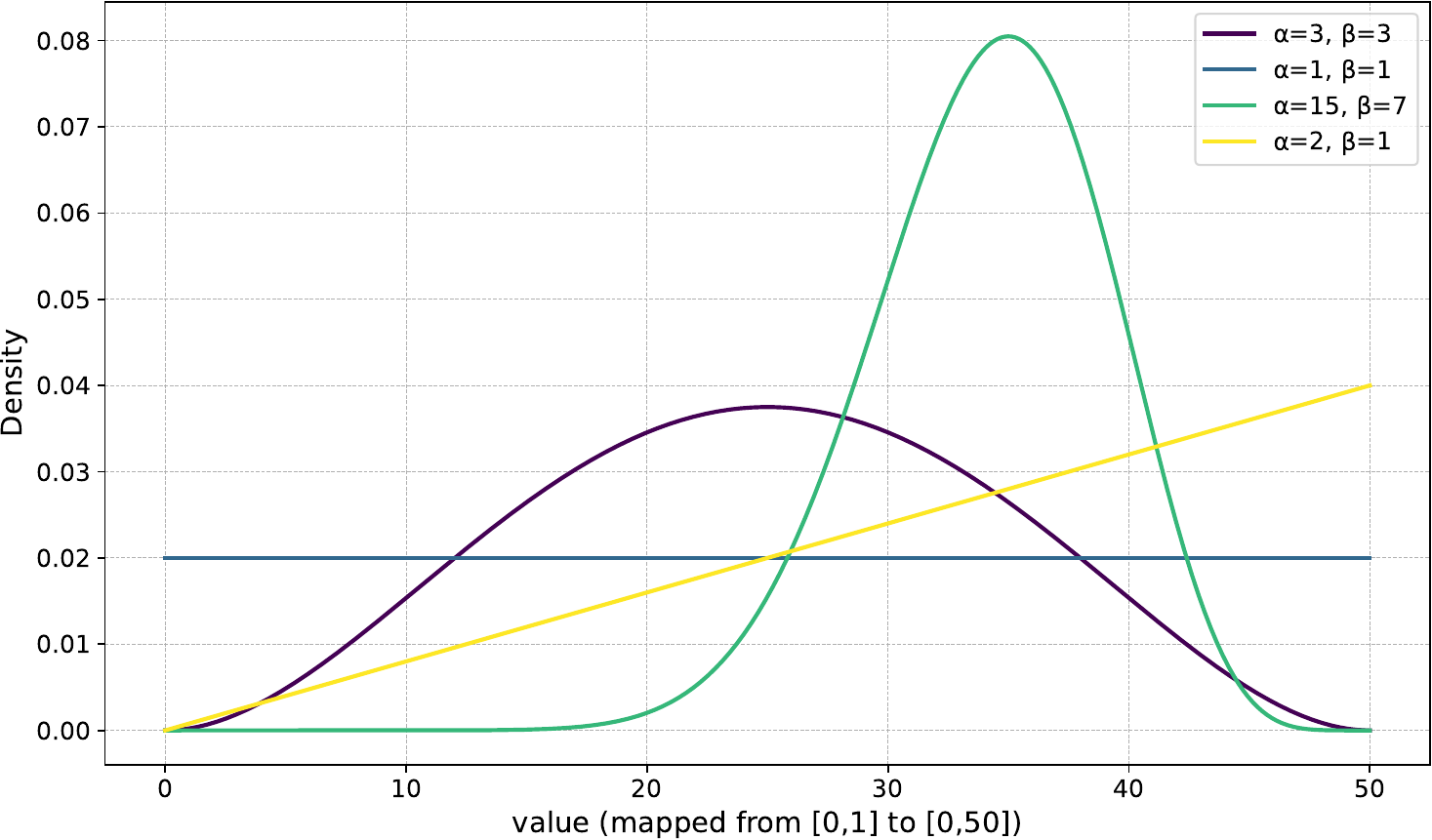}
    \caption{Beta Distribution with Different Parameter}
    \label{fig:transitiontime4}
  \end{subfigure}
  
  \caption{Different distribution of transition time for $T = 50$. $a), b), c)$ The transition time sampled 1K times under the different $\alpha_{t}$ schedule. d) The approximated transition time for $t = 1, \ldots, T$ using different hypter-parameters.}
  \label{fig:transitiontime}
\end{figure}

From Theorem~\ref{thm:transition probability}, we discern that the nature of the diffusion model scheduler, ${\alpha_t}$, clarifies the distribution of $\tau$. 

\noindent\textbf{Linear $\alpha$ schedule.} This is a schedule studied in \citet{austin2021structured}, where $\alpha_t = 1 - t/T$. This will result in $\mathbb{P}(\tau_n = t) = 1/T$ for every $t$ in the range $1$ to $T$. As a result, transition time distributes uniformly across each moment in the set $\{1, \ldots, T\}$. This can be verified in a) of Figure~\ref{fig:transitiontime}.

\noindent\textbf{Cosine $\alpha$ schedule.} This is a schedule studied in \citet{hoogeboom2021argmax}, where $\alpha_t = \cos(\pi*t/2T)$. For numerical consideration of the noise, a small offset $s$ is added, i.e., $\alpha_{t} = f(t)/f(0)$ where $f(t) = \cos((s+t/T)/(1+s)*\pi/2)$. As shown in b) of Figure~\ref{fig:transitiontime}, the transition time will concentrate more on the large $T$. 

\noindent\textbf{Cosine square $\alpha$ schedule.} This is a schedule studied in \citet{zheng2023reparameterized}, where $\alpha_t = \cos^2(\pi*t/2T)$, which motivated by \citet{nichol2021improved}. Again, for numerical consideration of the noise, a small offset $s$ is added, i.e., $\alpha_{t} = f(t)/f(0)$ where $f(t) = \cos^((s+t/T)/(1+s)*\pi/2)$. As shown in c) of Figure~\ref{fig:transitiontime}, the transition time will concentrate more on the middle of the range.

Generally, if we express $\alpha_t$ as $g(t/T)$, then we can simplify to $\mathbb{P}(\tau = t) = g((t-1)/T) - g(t/T)$, which further refines to $(1/T)|g'(t/T)| + o(1/T)$. This indicates that transitions are more likely where $|g'|$ is large. Such a mathematical finding can match our observation in Figure~\ref{fig:transitiontime}. 

In practice, we find that the shape of the transition time doesn't need to match the theoretical prediction schedule exactly. As we can see from d) in Figure~\ref{fig:transitiontime}. A reshaped Beta distribution can approximate all the transition time distributions in a fixed range. We first extract a time $t \in [0,1]$ from a Beta distribution, then adjust these samples to fit by multiplying $T$ and round them to acquire the integer. Our experiment finds that a properly chosen Beta distribution (tuned on the validation set) makes \method{} perform better on the translation tasks. 
\blue{Specifically, the chosen Beta distributions and the searching method are reported in Appendix \ref{Appendix: details}. The performance of the four transition time schedules mentioned above, including the reported Beta distributions for comparison, are listed in Table~\ref{tab: schedules}, where we find the other three schedules affect the performance, and most of their scores are lower than the scores of Beta distribution, but their scores are at least still close to the reported Beta distributions, especially for \method{}-k-absorb and \method{}-absorb. The efficiencies (measured by NFE) are also similar to one another.}

\blue{Additionally, the ablation study on a reasonable range of different Beta distributions with 50 and 1000 sampling steps are shown in Tables~\ref{tab:ablation50} and \ref{tab:ablation1000}, where the BLEU scores and NFE values on the test set of one of the three machine translation datasets, $\mathtt{WMT16}$, are shown for demonstration. The range of Beta distributions covers our chosen Beta schedules based on validation sets and a variety of basic Beta distribution shapes. These results show that the different Beta distributions influence the performance, but most of these choices of parameters still achieve results close to the optimal. Since the Beta distributions of the reported results in Tables~\ref{tab:real} and~\ref{tab:top} are selected using the validation set, they do not always have the highest scores on the test set, but their scores still at least belong to the top tiers according to these tables.}

\noindent\textbf{Another view of the transition time.} In Algorithm~\ref{alg:DNDMv1}, we only need to call the neural network when $t \in \cT$, which can significantly speed up the sampling since we reduce the function call. Notice that after we get the $\xb_{0}$ prediction, we only update the $\xb_{t}$ for those tokens at the transition time. However, \eqref{eq:deterministic} implies that $\xb_{t} = \xb_0$ as long as $\tau > t$. Therefore, instead of only updating the $\xb_{t}$ for those tokens at the transition time, i.e., $\tau = t$, we can also update those tokens with transition time $\tau >= t$. This motivates us to consider a variation presented as Algorithm~\ref{alg:DNDMv2},  which keeps almost the same sampling time but will update the tokens several times rather than just once. Since the tokens now get the chance to be corrected over time. The new Algorithm~\ref{alg:DNDMv2}  will be more robust than Algorithm~\ref{alg:DNDMv1}.

\begin{table}[ht!]
    \centering
 \caption{\blue{The BLEU scores and average number of function evaluations (NFE) values of different distributions of transition time for $1000$ \BB{sampling steps} with batch size 100. The parameters of the Beta distributions in this table are the same as in Tables \ref{tab:real} and \ref{tab:top} and are reported in Appendix \ref{Appendix: details}.}}
\resizebox{\columnwidth}{!}{
\begin{tabular}{l|l|ll|ll|ll|ll}
\toprule
\multicolumn{1}{c|}{\multirow{2}{*}{Datasets}} & \multicolumn{1}{c|}{\multirow{2}{*}{Schedules}} & \multicolumn{2}{c|}{DNDM-multi}                           & \multicolumn{2}{c|}{DNDM-absorb}                          & \multicolumn{2}{c|}{DNDM-k-multi}                         & \multicolumn{2}{c}{DNDM-k-absorb}                        \\ \cmidrule{3-10} 
\multicolumn{1}{c|}{}                          & \multicolumn{1}{c|}{}                           & \multicolumn{1}{c|}{BLEU}  & \multicolumn{1}{c|}{Avg NFE} & \multicolumn{1}{c|}{BLEU}  & \multicolumn{1}{c|}{Avg NFE} & \multicolumn{1}{c|}{BLEU}  & \multicolumn{1}{c|}{Avg NFE} & \multicolumn{1}{c|}{BLEU}  & \multicolumn{1}{c}{Avg NFE} \\ \midrule
\multirow{4}{*}{$\mathtt{IWSLT14}$}                        & Cosine                                          & \multicolumn{1}{l|}{31.72} & 31.71                        & \multicolumn{1}{l|}{32.71} & 31.21                        & \multicolumn{1}{l|}{32.91} & 31.71                        & \multicolumn{1}{l|}{34.50}  & 31.21                        \\
                                                & Cosine$^2$                       & \multicolumn{1}{l|}{31.78} & 31.74                        & \multicolumn{1}{l|}{$\mathbf{32.93}$} & 31.21                        & \multicolumn{1}{l|}{32.78} & 31.74                        & \multicolumn{1}{l|}{34.53} & 31.21                        \\
                                                & Linear $\alpha$                                   & \multicolumn{1}{l|}{31.77} & 31.82                        & \multicolumn{1}{l|}{32.65} & 31.33                        & \multicolumn{1}{l|}{32.83} & 31.82                        & \multicolumn{1}{l|}{34.53} & 31.33                        \\
                                                & Beta (reported)                                   & \multicolumn{1}{l|}{$\mathbf{31.82}$} & $\mathbf{30.33}$                        & \multicolumn{1}{l|}{$\mathbf{32.93}$} & $\mathbf{31.08}$                        & \multicolumn{1}{l|}{$\mathbf{33.15}$} & $\mathbf{30.33}$                        & \multicolumn{1}{l|}{$\mathbf{34.56}$} & $\mathbf{31.08}$                        \\ \midrule
\multirow{4}{*}{$\mathtt{WMT14}$}                          & Cosine                                          & \multicolumn{1}{l|}{$\mathbf{25.80}$} & 39.61                        & \multicolumn{1}{l|}{$\mathbf{26.54}$} & 39.18                        & \multicolumn{1}{l|}{26.63} & 39.61                        & \multicolumn{1}{l|}{27.81} & 39.18                        \\
                                                & Cosine$^2$                       & \multicolumn{1}{l|}{25.52} & 39.48                        & \multicolumn{1}{l|}{26.53} & 39.18                        & \multicolumn{1}{l|}{25.01} & 39.48                        & \multicolumn{1}{l|}{$\mathbf{27.95}$} & 39.18                        \\
                                                & Linear $\alpha$                                   & \multicolumn{1}{l|}{25.58} & 39.97                        & \multicolumn{1}{l|}{26.33} & 39.82                        & \multicolumn{1}{l|}{25.47} & 39.97                        & \multicolumn{1}{l|}{27.63} & 39.82                        \\
                                                & Beta (reported)                                   & \multicolumn{1}{l|}{25.71} & $\mathbf{38.94}$                        & \multicolumn{1}{l|}{26.43} & $\mathbf{38.76}$                        & \multicolumn{1}{l|}{$\mathbf{26.88}$} & $\mathbf{38.94}$                        & \multicolumn{1}{l|}{27.82} & $\mathbf{38.76}$                        \\ \midrule
\multirow{4}{*}{$\mathtt{WMT16}$}                          & Cosine                                          & \multicolumn{1}{l|}{32.71} & 40.50                        & \multicolumn{1}{l|}{33.56} & 40.45                        & \multicolumn{1}{l|}{33.46} & 40.50                        & \multicolumn{1}{l|}{34.37} & 40.45                        \\
                                                & Cosine$^2$                       & \multicolumn{1}{l|}{32.73} & 40.50                        & \multicolumn{1}{l|}{33.51} & 40.45                        & \multicolumn{1}{l|}{33.44} & 40.50                        & \multicolumn{1}{l|}{34.24} & 40.45                        \\
                                                & Linear $\alpha$                                   & \multicolumn{1}{l|}{32.85} & 40.36                        & \multicolumn{1}{l|}{33.46} & 40.36                        & \multicolumn{1}{l|}{33.47} & 40.36                        & \multicolumn{1}{l|}{33.88} & 40.36                        \\
                                                & Beta (reported)                                   & \multicolumn{1}{l|}{$\mathbf{32.86}$} & $\mathbf{38.46}$                        & \multicolumn{1}{l|}{$\mathbf{33.60}$} & $\mathbf{38.27}$                        & \multicolumn{1}{l|}{$\mathbf{33.79}$} & $\mathbf{38.45}$                        & \multicolumn{1}{l|}{$\mathbf{34.38}$} & $\mathbf{38.27}$                        \\ \bottomrule
\end{tabular}    }
    \label{tab: schedules}
\end{table}

\noindent\textbf{Impact of Transition Order.}
We further evaluate the impact of transition order. Building upon the results in Table~\ref{tab:top}, we investigate how the model performance will change if the transition time is influenced by the position of the tokens: from left to right and from right to left. In the left-to-right approach, tokens positioned on the left are transitioned to $\xb_0$ earlier, and vice versa for the right-to-left approach. Our experiments show that the left-to-right approach consistently outperforms the right-to-left approach across all datasets and step counts, as demonstrated in Table~\ref{tab:direction_comparison}.

\begin{table}[t]
\caption{Comparison of left-to-right and right-to-left transition approaches across different datasets and step counts.}
\label{tab:direction_comparison}
\centering
\begin{tabular}{c|c|c|c|c}
\toprule
\textbf{Steps} & \textbf{Direction} & \textbf{IWSLT14} & \textbf{WMT14} & \textbf{WMT16} \\
\midrule
\multirow{2}{*}{25} & Left-to-right & 31.08 & 24.41 & 31.67 \\
& Right-to-left & 30.54 & 23.33 & 31.33 \\
\midrule
\multirow{2}{*}{50} & Left-to-right & 32.87 & 26.46 & 33.37 \\
& Right-to-left & 32.47 & 25.18 & 32.78 \\
\midrule
\multirow{2}{*}{1000} & Left-to-right & 34.45 & 27.93 & 34.43 \\
& Right-to-left & 34.04 & 27.02 & 34.15 \\
\bottomrule
\end{tabular}
\end{table}

This result suggests that the order of token transitions significantly influences the model's performance, with earlier transitions of left-side tokens leading to better generation quality.

{\color{black}
\section{Discussion on the Number of Function Evaluations (NFE).}\label{sec:nfe}
In this section, we discuss the number of function evaluations (NFE) in DNDM. According to \eqref{eq:multi-token version1}, the update of a token $\xb_{t-1, n}$ occurs solely at its designated transition time. Meanwhile, if step $t$ does not coincide with a transition time for any token, we maintain the sentence from the preceding step unchanged: $\xb_{t, 1:N} = \xb_{t-1, 1:N}$. Therefore, our algorithm removes the need of function evaluation for steps outside the set of transition times. Given this structure, our analytical emphasis is on the transition set $\cT$ since function evaluations are required only at times $t$ that are members of $\cT$. Consequently, the NFE is precisely the cardinality of the transition set, denoted by $|\cT|$. In our main paper, we propose a naive upper bound for $|\cT|$ as $\min\{N, T\}$, which effectively demonstrates the speed of our method when $T > N$. Next, we demonstrate that DNDM also reduces the NFE when $T < N$, by providing a precise estimation of $|\cT|$.
\begin{theorem}\label{thm:1}
    Suppose transition time follows distribution $\mathcal{D}_{\tau}$, and consider a sequence of length $N$. Then, the cardinality of the transition set $\mathcal{T} := \{\tau_1, \ldots, \tau_{N}\}$ satisfies:
    \begin{itemize}[leftmargin=*,nosep]
        \item $1 \leq |\mathcal{T}| \leq \min\{N, T\}$,
        \item $\mathbb{E}[|\mathcal{T}|] = [1 - C_{T,N, \mathcal{D}_{\tau}}] \cdot T$, where $C_{T,N, \mathcal{D}_{\tau}}$ is a constant in the range $(0,1)$. Furthermore,
        \begin{align*}
            C_{T,N, \mathcal{D}_{\tau}} = \Big(\sum_{i=1}^{T}(1-p_i)^{N}\Big)/T \geq (1 - 1/T)^{N},
        \end{align*}
where $p_i = \PP(\tau = i)$ for $\tau \sim \cD_{\tau}$, and the equality holds if and only if $\mathcal{D}_{\tau}$ is a uniform distribution.
    \end{itemize}
\end{theorem}

\begin{proof}
    The first statement is straightforward. For completeness, the proof is provided. Since there are only $N$ transition times (possibly repeated): $\tau_{1}, \ldots, \tau_{N}$, the distinct transition times must satisfy $|\mathcal{T}| \leq N$. Additionally, since $\mathcal{T} \subseteq \{1, \ldots, T\}$, we also have $|\mathcal{T}| \leq T$.

    To prove the second statement, we decompose $\mathcal{T}$ and use the property of expectation. Note that $|\mathcal{T}| = \sum_{i=1}^{T} \ind\{i \in \mathcal{T}\}$. Thus,
    \begin{align}
        \mathbb{E}[|\mathcal{T}|] = \mathbb{E}\bigg[\sum_{i=1}^{T} \ind\{i \in \mathcal{T}\}\bigg] = \sum_{i=1}^{T} \mathbb{P}(i \in \mathcal{T}). \label{eq:decompose}
    \end{align}
    Assuming $\mathbb{P}_{\mathcal{D}_{\tau}}(\tau = i) = p_{i}$, and that $\tau_{n}$ are i.i.d. draws from $\mathcal{D}_{\tau}$, we have
    \begin{align}
        \mathbb{P}(i \in \mathcal{T}) = 1 - \mathbb{P}(i \notin \mathcal{T}) = 1 - (1-p_i)^{N}. \label{eq:each-term}
    \end{align}
    Substituting \eqref{eq:each-term} into \eqref{eq:decompose} yields
    \begin{align*}
        \mathbb{E}[|\mathcal{T}|] &= \sum_{i=1}^{T}\Big[1 - (1-p_i)^{N}\Big] = \Big[1 - \frac{\sum_{i=1}^{T}(1-p_i)^{N}}{T}\Big]\cdot T = [1 - C_{T,N, \mathcal{D}_{\tau}}]\cdot T,
    \end{align*}
    where $C_{T,N,\mathcal{D}_{\tau}} = \Big(\sum_{i=1}^{T}(1-p_i)^{N}\Big)/T$. An upper bound for $C_{T,N,\mathcal{D}_{\tau}}$ is given as
    \begin{align*}
        C_{T,N,\mathcal{D}_{\tau}} = \Big[1 - \frac{\sum_{i=1}^{T}(1-p_i)^{N}}{T}\Big]\cdot T \leq \Big[1 -\Big(1- \frac{1}{T}\Big)^{N}\Big]\cdot T,
    \end{align*}
    where the inequality holds if and only if $p_{i} = 1/T$ for all $i \in [T]$, i.e., $\mathcal{D}_{\tau}$ is a uniform distribution.
\end{proof}
\begin{remark}
    Theorem~\ref{thm:1} suggests that even when $T \leq N$, our method still provides a significant improvement. 
    Specifically, for $T = N \geq 4$, we have $C_{T,N, \mathcal{D}_{\tau}} = (1-1/N)^{N} \geq 0.3$. This implies that our model requires at most $0.7T$ even in the worst case.
    Moreover, if we consider a special scenario where the number of $p_i$ satisfying $p_i < \epsilon$ is more than $M$, then we have $C_{T,N,\mathcal{D}_{\tau}} > M (1 - \epsilon)^N/T$, indicating that with $M$ sufficiently large and $\epsilon$ sufficiently small, $C_{T,N,\mathcal{D}_{\tau}}$ can be pretty close to $1$.
\end{remark}

\begin{remark}
   In practical applications of our model, we employ a beta distribution for $\cD_{\tau}$, which typically exhibits a right-heavy tail. Therefore $C_{T,N, \mathcal{D}_{\tau}}$ tends to be larger than that in the worst-case scenario.
   In Tables~\ref{tab:real2} and \ref{tab:top2}, we list the average NFE for each experiment we run in~\S\ref{sec:experiments}. These results demonstrate a significant reduction in NFE compared to the original counts: for $T=25$, the NFE is only about half of the original count; for $T=50$, it is approximately one-third; and for $T=1000$, it reduces to less than one-twentieth of the original count. 
\end{remark}

\begin{remark}
By Bernoulli's inequality, $(1-p)^{N} > 1 - N \cdot p$ for $1> p > 0$. Therefore, $C_{T,N,\mathcal{D}_{\tau}} > 1 - N/T$, implying that $\mathbb{E}[|\mathcal{T}|] < N$. As $T \rightarrow \infty$, assuming the transition time does not concentrate at a single point, the probability that two transitions occur simultaneously is zero. Consequently, the generation process will sequentially go through each token. Thus, the expected number of function evaluations (NFE), $\mathbb{E}[|\mathcal{T}|]$, will be $N$. In contrast, when $T$ is finite, there is a non-zero probability that multiple transitions happen at the same time. Hence, in this case, the NFE, $|\mathcal{T}|$, is strictly less than $N$
\end{remark}

}


\begin{table}[ht!]
    \centering
     \caption{\blue{BLEU score and the average number of function evaluations (NFE) comparison of multinomial diffusion on machine translation benchmarks $\mathtt{IWSLT14\ DE}$-$\mathtt{EN}$, $\mathtt{WMT14\ EN}$-$\mathtt{DE}$, and $\mathtt{WMT16\ EN}$-$\mathtt{RO}$.  The blue background highlights our algorithms. The average NFE values are calculated by dividing the number of times calling the denoising function (neural network) during generation by the number of batches, where the batch sizes of all experiments are 100.}}
    \resizebox{\columnwidth}{!}{
    \begin{tabular}{c|c|c|c|g|g|c|c|g|g}
    \toprule 
        \multirow{2}{*}{\textbf{Dataset}} & \multirow{2}{*}{\textbf{Steps}} & \multicolumn{2}{c|}{\textbf{RDM-Multi}} & \multicolumn{2}{c|}{\textbf{DNDM-Multi}} & 
        \multicolumn{2}{c|}{\textbf{RDM-$k$-Multi}} & \multicolumn{2}{c}{\textbf{DNDM-$k$-Multi}} \\
        \cmidrule{3-10}
         &  & \textbf{BLEU} & \textbf{Avg NFE} & \cellcolor{white}\textbf{BLEU} & \cellcolor{white}\textbf{Avg NFE} & \textbf{BLEU} & \textbf{Avg NFE} & \cellcolor{white}\textbf{BLEU} & \cellcolor{white}\textbf{Avg NFE} \\
        \midrule
        \multirow{4}{*}{\centering$\mathtt{IWSLT14}$}  & 25 &\textbf{31.26} & 25 &30.95 &\textbf{9.03} & \textbf{32.82} & 25 & 32.30 &\textbf{9.03} \\
        
        & 50 & \textbf{31.50} & 50 & 31.45& \textbf{14.07}& \textbf{32.82} & 50 & 32.80 & \textbf{14.07}\\
        
        & 1000 & 31.69 & 1000 &\textbf{31.82} & \textbf{30.33} & 32.64 & 1000 &  \textbf{33.15} & \textbf{30.33} \\

        & $\infty $ &-  & -& \textbf{31.89} & \textbf{32.73} &  -& -&\textbf{33.44}  & \textbf{32.73} \\
        \midrule
        \multirow{4}{*}{\centering$\mathtt{WMT14}$} & 25 & \textbf{25.25} & 25 &25.01 &  \textbf{13.52} & \textbf{26.03} & 25 & 25.98& \textbf{13.52}\\
        
        & 50 & \textbf{25.75} & 50 & 25.33& \textbf{20.58} & 26.14 & 50 & \textbf{26.37} & \textbf{20.58} \\
        
        & 1000 & 25.66 & 1000 & \textbf{25.71} & \textbf{38.94} & 25.82 & 1000 &  \textbf{26.88} & \textbf{38.94}\\

        & $\infty $  & - &- & \textbf{24.79} & \textbf{40.67} & - &- & \textbf{26.39} & \textbf{40.67} \\
        \midrule
        \multirow{4}{*}{\centering$\mathtt{WMT16}$} & 25 &\textbf{32.29} &25 & 31.97& \textbf{8.5} & \textbf{33.12} & 25 & 32.94 & \textbf{8.5}\\
        
        & 50 & \textbf{32.53} & 50&32.50 & \textbf{14.73} & \textbf{33.41} & 50 & 33.26 & \textbf{14.73}\\
        
        & 1000 & 32.63 & 1000 & \textbf{32.86}& \textbf{38.45}& 33.67 & 1000 & \textbf{33.79} & \textbf{38.45}\\

        & $\infty $ &-  &- & \textbf{32.91} & \textbf{41.64} &-  &- & \textbf{33.86} & \textbf{41.64} \\
    \bottomrule
    \end{tabular}}
    \label{tab:real2}
\end{table}

\begin{table}[ht!]
 \caption{\blue{BLEU score and the average number of function evaluations (NFE) comparison of absorbing diffusion on machine translation benchmarks $\mathtt{IWSLT14\ DE}$-$\mathtt{EN}$, $\mathtt{WMT14\ EN}$-$\mathtt{DE}$, and $\mathtt{WMT16\ EN}$-$\mathtt{RO}$.  The blue background highlights our algorithms. The average NFE values are calculated by dividing the number of times calling the denoising function (neural network) during generation by the number of batches, where the batch sizes of all experiments are 100.}}
    \centering
   \resizebox{\columnwidth}{!}{
    \begin{tabular}{c|c|c|c|g|g|c|c|g|g}
    \toprule 
        \multirow{2}{*}{\textbf{Dataset}} & \multirow{2}{*}{\textbf{Steps}} & \multicolumn{2}{c|}{\textbf{RDM-Absorb}} & \multicolumn{2}{c|}{\textbf{DNDM-Absorb}} & 
        \multicolumn{2}{c|}{\textbf{RDM-$k$-Absorb}} & \multicolumn{2}{c}{\textbf{DNDM-$k$-Absorb}} \\
        \cmidrule{3-10}
         &  & \textbf{BLEU} & \textbf{Avg NFE} & \cellcolor{white}\textbf{BLEU} & \cellcolor{white}\textbf{Avg NFE} & \textbf{BLEU} & \textbf{Avg NFE} & \cellcolor{white}\textbf{BLEU} & \cellcolor{white}\textbf{Avg NFE} \\
        \midrule
        \multirow{4}{*}{\centering$\mathtt{IWSLT14}$} & 25 & 31.58 & 25 & \textbf{32.43} & \textbf{13.81} & \textbf{34.50} & 25 & 34.14 & \textbf{13.81} \\
        
        & 50 & 31.80 & 50 & \textbf{32.63} & \textbf{19.24}  & \textbf{34.58} & 50 & 34.34 & \textbf{19.24}  \\
        
        & 1000 & 31.91 & 1000 & \textbf{32.93} & \textbf{31.08}  & \textbf{34.60} & 1000 & 34.56 & \textbf{31.08}  \\

        & $\infty $ & - & - &  \textbf{33.03} & \textbf{32.07}  & - & - & \textbf{34.65} & \textbf{32.07}  \\
        \midrule
        \multirow{4}{*}{\centering$\mathtt{WMT14}$} & 25 & 24.97 & 25  & \textbf{25.79} & \textbf{15.09} & \textbf{27.50} & 25  & 27.18 & \textbf{15.09} \\
        
        & 50 & 24.95 & 50 & \textbf{26.10} & \textbf{22.45}  & \textbf{27.73} & 50 & 27.66 & \textbf{22.45}  \\
        
        & 1000 & 25.22 & 1000 & \textbf{26.43} & \textbf{38.76}  & 27.75 & 1000 & \textbf{27.82} & \textbf{38.76}  \\

        & $\infty $ & -  & - &  \textbf{26.50} & \textbf{40.39}  & - & - & \textbf{27.50} & \textbf{40.39}  \\
        \midrule
        \multirow{4}{*}{\centering$\mathtt{WMT16}$} & 25 & 32.86 & 25 & \textbf{33.20} & \textbf{13.91} & 33.92 & 25  & \textbf{33.96} & \textbf{13.91} \\
        
        & 50 & 32.93 & 50  & \textbf{33.30} &   \textbf{20.95}  & 34.10 & 50 & \textbf{34.20} & \textbf{20.95} \\
        
        & 1000 & 33.25 & 1000 & \textbf{33.60} & \textbf{38.27} & \textbf{34.44} & 1000 & 34.38 & \textbf{38.27}  \\

        & $\infty $ & - & - &  \textbf{33.42} & \textbf{41.59}  &  - & - & \textbf{34.41} & \textbf{41.59}  \\
    \bottomrule
    \end{tabular}
    }
    \label{tab:top2}
\end{table}

\section{Discrete Non-Markov Diffusion Model with Top-k Transition Time (\method{}-K).} \label{Appendix: Top}

\begin{figure}[ht]
    \centering
    \begin{minipage}{0.49\textwidth}
    \vspace{-55pt}
        \begin{algorithm}[H]
        \caption{Sampling From DNDM (Version 2)}
        \begin{algorithmic}[1]\label{alg:DNDMv2}
        \REQUIRE Trained prediction function $p_{\btheta}$, $\qb_{\mathrm{noise}}$, $\cD_{\tau}$
        
        \FOR{$n = 1 \ldots N$} 
        \STATE Initiate each token $\xb_{T,n} \sim \qb_{\mathrm{noise}}$
        \STATE Initiate the transition time $\tau_{n} \sim \cD_{\tau}$
        \ENDFOR
        
        \STATE Collect transition time set $\cT = \{\tau_{n}\}_{n=1}^{N}$
        
        \FOR{$t = T \ldots 1$}
        \IF{$t \in \cT$}
        \STATE Generate $\tilde{\xb}_{0,1:N}$ from $p_{\btheta}(\cdot|\xb_{t, 1:N})$
        
        \FOR{$n = 1 \ldots N$} 
        \STATE Update $\xb_{t-1,n}$ if $\tau_{n}\geq t$
        \ENDFOR
        
        \ELSE  
        \STATE Update $\xb_{t-1,1:N} = \xb_{t,1:N}$
        \ENDIF
        \ENDFOR
        
        \STATE \textbf{Return} $\xb_{0,1:N}$
        \end{algorithmic}
        \end{algorithm}
    \end{minipage}
    \hfill
    \begin{minipage}{0.49\textwidth}
   
\begin{algorithm}[H]
\caption{Sampling From \method{}-K}
\begin{algorithmic}\label{alg:DNDMK}
\STATE \textbf{Input:} Trained prediction function $p_{\btheta}$, $\qb_{\mathrm{noise}}$ and $\cD_{\tau}$
\FOR{$n = 1 \ldots N$} 
\STATE Initiate each token $\xb_{T,n} \sim \qb_{\mathrm{noise}}$
\STATE Initiate the top K number $\{K_{t}\}$
\STATE Initiate an empty set $U = \{\}$, which includes the index of the tokens that have been updated.
\ENDFOR
\FOR{$t = T \ldots 1$}
\IF{$K_{t-1} > K_{t}$}
\STATE Calculate the $\cP = \mathrm{argtop}_{K_t}\{s_{t,n}\}_{n=1}^{N}$;
\STATE Generate $\tilde{\xb}_{0,1:N}$ from  $p_{\btheta}(\cdot|\xb_{t, 1:N})$
\STATE Update $\xb_{t-1,n} = \tilde{\xb}_{0,n}$ for all $n$ in the set $\cP$ but not in the set $U$ (top score but not updated yet)
\STATE Update the set $U$ by appending the index of the updated tokens
\ELSE  
\STATE Update $\xb_{t-1,1:N} = \xb_{t,1:N}$;
\ENDIF
\ENDFOR
\STATE \textbf{Return} $\xb_{0,1:N}$.
\end{algorithmic}
\end{algorithm}
    \end{minipage}
\end{figure}

Recent works have demonstrated that the quality of samples can be enhanced by utilizing supplementary information derived from the neural network \citep{ghazvininejad2019mask, savinov2021step, chang2022maskgit, he2022diffusionbert}. Very recently, \citet{zheng2023reparameterized} applied this idea in their RDM framework and can achieve significant performance improvement. Specifically, after decoding $\hat{\xb}_{0, 1:N}$ from transformer $p_{\theta}(\cdot |\xb_{t, 1:N})$, the score corresponding to this decoded token from the transformer's last layer, is also recorded and denote as $s_{t,n}$. Tokens with high scores are more likely to be selected for updates. 

Inspired by \citet{zheng2023reparameterized}, we introduce the discrete non-Markov discrete diffusion Model with top-K transition time (DNDM-K). Instead of directly determining which token gets updated at step $t$ by first drawing transition time $\tau \sim \cD_{\tau}$, we employ a two-step process. 
\begin{enumerate}[leftmargin=*,nosep]
\item We first compute $K_{t} = \sum_{n=1}^{N}\ind(\tau_{n} \geq t)$. $k_t$ represents how many tokens should be decoded at the current step.
\item Compare $K_{t-1}$ and $K_{t}$, if $K_{t-1} = K_{t}$. There is no transition time at time $t$, we just update $\xb_{t-1,1:N} = \xb_{t,1:N}$. If $K_{t-1} > K_{t}$, Then there exist transition time at time $t$, we calculate and select the indexes with top-$K_{t-1}$ scores. Then we update those tokens if it hasn't been updated yet.
\end{enumerate}
 Subsequently, we will only update those tokens with the highest $K_t$ score that hasn't been changed yet. Since the function evaluation occurs only when $K_{t}$ changes, \method{}-K can give an accelerated sampling algorithm. The details are presented in Algorithm~\ref{alg:DNDMK}.

\section{Experiment details} \label{Appendix: details}

\subsection{Conditional Text Generation}\label{app:cond_text_gen}


\noindent\textbf{Parameter choices.} In all experiments, the batch size is chosen to be $100$. For RDM and RDM-$k$, our hyperparameter settings follow the original paper \citep{zheng2023reparameterized} except for the batch size. Before the sampling, we used the saved checkpoint of trained models provided by the authors for discrete sampling experiments, and we trained the corresponding models for continuous sampling experiments. 

For finite-step \method{}, the transition times are determined by the schedule, and we approximate the schedule with a Beta distribution $\text{Beta}(\alpha, \beta)$ (please refer to Section~\ref{sec:fast} for detailed explanation). The $\alpha$ and $\beta$ values are selected by applying grid search on the validation sets. Based on the BLEU scores on the validation sets, we have selected $\text{Beta}(15,7)$ for Multinormial Diffusion on $\mathtt{IWSLT14}$, $\text{Beta}(3,3)$ for Absorbing Diffusion on both $\mathtt{IWSLT14}$ and $\mathtt{WMT14}$, $\text{Beta}(\blue{5,3})$ for Multinormial Diffusion on $\mathtt{WMT14}$ and Absorbing Diffusion on $\mathtt{WMT16}$, and $\text{Beta}(20,7)$ for Multinormial Diffusion on $\mathtt{WMT16}$. 

For infinite-steps (continuous-step) diffusion (\method{}-C), the transition timestamps are sampled from $\text{Beta}(\alpha, \beta)$, where the choice of $(\alpha, \beta)$ are chosen from $(100.0, 4.0)$ or $(17.0, 4.0)$, based on the performance comparison on the validation set. In the end we choose $\text{Beta}(17,4)$ for $\mathtt{IWSLT14}$ and $\text{Beta}(100, 4)$ for $\mathtt{WMT14}$ and $\mathtt{WMT16}$.


\blue{We conduct a performance comparison based on varying configurations of the Beta and Alpha distributions. The results of these comparisons are presented in Tables~\ref{tab:ablation50} and \ref{tab:ablation1000}. Furthermore, to evaluate the efficacy of discrete versus continuous step schemes, we also conduct an ablation study under the same set of parameters $(100, 4)$ in Table~\ref{tab:100_4}.}

\begin{table}[ht!]
    \centering
      \caption{\blue{BLEU scores on dataset $\mathtt{WMT16}$  from the ablation study of other different Beta$(\alpha, \beta)$ distributions of the transition time with 1000 \BB{sampling steps}.}}
\resizebox{\columnwidth}{!}{
\begin{tabular}{c|c|cccccccccc}
\toprule
\multirow{2}{*}{Model}         & \multirow{2}{*}{Alpha} & \multicolumn{10}{c}{Beta}                                                                                                                                                                                                                                                 \\ \cmidrule{3-12} 
                               &                        & \multicolumn{1}{c|}{3}     & \multicolumn{1}{c|}{5}     & \multicolumn{1}{c|}{7}     & \multicolumn{1}{c|}{9}     & \multicolumn{1}{c|}{11}    & \multicolumn{1}{c|}{13}    & \multicolumn{1}{c|}{15}    & \multicolumn{1}{c|}{17}    & \multicolumn{1}{c|}{19}    & 21    \\ \midrule
\multirow{3}{*}{DNDM-k-Multi}  & 3                      & \multicolumn{1}{c|}{33.47} & \multicolumn{1}{c|}{33.67} & \multicolumn{1}{c|}{33.62} & \multicolumn{1}{c|}{33.77} & \multicolumn{1}{c|}{$\mathbf{33.87}$} & \multicolumn{1}{c|}{33.64} & \multicolumn{1}{c|}{33.73} & \multicolumn{1}{c|}{33.60}  & \multicolumn{1}{c|}{33.68} & 33.56 \\
                               & 5                      & \multicolumn{1}{c|}{33.18} & \multicolumn{1}{c|}{33.47} & \multicolumn{1}{c|}{33.68} & \multicolumn{1}{c|}{33.53} & \multicolumn{1}{c|}{33.71} & \multicolumn{1}{c|}{33.69} & \multicolumn{1}{c|}{33.73} & \multicolumn{1}{c|}{33.72} & \multicolumn{1}{c|}{33.74} & 33.82 \\
                               & 7                      & \multicolumn{1}{c|}{32.99} & \multicolumn{1}{c|}{33.20}  & \multicolumn{1}{c|}{33.49} & \multicolumn{1}{c|}{33.56} & \multicolumn{1}{c|}{33.58} & \multicolumn{1}{c|}{33.61} & \multicolumn{1}{c|}{33.67} & \multicolumn{1}{c|}{33.72} & \multicolumn{1}{c|}{33.78} & 33.83 \\ \midrule
\multirow{3}{*}{DNDM-Multi}    & 3                      & \multicolumn{1}{c|}{32.73} & \multicolumn{1}{c|}{32.66} & \multicolumn{1}{c|}{32.74} & \multicolumn{1}{c|}{32.82} & \multicolumn{1}{c|}{32.77} & \multicolumn{1}{c|}{$\mathbf{32.92}$} & \multicolumn{1}{c|}{32.80}  & \multicolumn{1}{c|}{32.81} & \multicolumn{1}{c|}{32.76} & 32.86 \\
                               & 5                      & \multicolumn{1}{c|}{32.32} & \multicolumn{1}{c|}{32.62} & \multicolumn{1}{c|}{32.70}  & \multicolumn{1}{c|}{32.80}  & \multicolumn{1}{c|}{32.83} & \multicolumn{1}{c|}{32.83} & \multicolumn{1}{c|}{32.90}  & \multicolumn{1}{c|}{32.95} & \multicolumn{1}{c|}{32.91} & 32.87 \\
                               & 7                      & \multicolumn{1}{c|}{32.35} & \multicolumn{1}{c|}{32.35} & \multicolumn{1}{c|}{32.53} & \multicolumn{1}{c|}{32.67} & \multicolumn{1}{c|}{32.75} & \multicolumn{1}{c|}{32.78} & \multicolumn{1}{c|}{32.86} & \multicolumn{1}{c|}{32.80}  & \multicolumn{1}{c|}{32.86} & 32.88 \\ \midrule
\multirow{3}{*}{DNDM-k-Absorb} & 3                      & \multicolumn{1}{c|}{34.19} & \multicolumn{1}{c|}{34.38} & \multicolumn{1}{c|}{34.34} & \multicolumn{1}{c|}{34.22} & \multicolumn{1}{c|}{34.21} & \multicolumn{1}{c|}{34.24} & \multicolumn{1}{c|}{34.07} & \multicolumn{1}{c|}{34.31} & \multicolumn{1}{c|}{$\mathbf{34.42}$} & 34.36 \\
                               & 5                      & \multicolumn{1}{c|}{32.15} & \multicolumn{1}{c|}{33.99} & \multicolumn{1}{c|}{34.29} & \multicolumn{1}{c|}{34.30}  & \multicolumn{1}{c|}{34.29} & \multicolumn{1}{c|}{34.40}  & \multicolumn{1}{c|}{34.40}  & \multicolumn{1}{c|}{34.24} & \multicolumn{1}{c|}{34.30}  & 34.22 \\
                               & 7                      & \multicolumn{1}{c|}{27.67} & \multicolumn{1}{c|}{32.87} & \multicolumn{1}{c|}{33.94} & \multicolumn{1}{c|}{34.28} & \multicolumn{1}{c|}{34.27} & \multicolumn{1}{c|}{34.38} & \multicolumn{1}{c|}{34.31} & \multicolumn{1}{c|}{34.29} & \multicolumn{1}{c|}{34.38} & 34.40  \\ \midrule
\multirow{3}{*}{DNDM-Absorb}   & 3                      & \multicolumn{1}{c|}{33.53} & \multicolumn{1}{c|}{33.60}  & \multicolumn{1}{c|}{33.67} & \multicolumn{1}{c|}{33.71} & \multicolumn{1}{c|}{33.71} & \multicolumn{1}{c|}{33.70}  & \multicolumn{1}{c|}{33.58} & \multicolumn{1}{c|}{33.63} & \multicolumn{1}{c|}{33.53} & 33.54 \\
                               & 5                      & \multicolumn{1}{c|}{32.70}  & \multicolumn{1}{c|}{33.33} & \multicolumn{1}{c|}{33.52} & \multicolumn{1}{c|}{33.60}  & \multicolumn{1}{c|}{33.66} & \multicolumn{1}{c|}{33.73} & \multicolumn{1}{c|}{33.70}  & \multicolumn{1}{c|}{$\mathbf{33.74}$} & \multicolumn{1}{c|}{33.72} & $\mathbf{33.74}$ \\
                               & 7                      & \multicolumn{1}{c|}{30.56} & \multicolumn{1}{c|}{32.65} & \multicolumn{1}{c|}{33.28} & \multicolumn{1}{c|}{33.37} & \multicolumn{1}{c|}{33.51} & \multicolumn{1}{c|}{33.52} & \multicolumn{1}{c|}{33.61} & \multicolumn{1}{c|}{33.67} & \multicolumn{1}{c|}{33.63} & 33.67 \\ \bottomrule
\end{tabular}    }
    \label{tab:ablation1000}
\end{table}
\begin{table}[ht!]
 \caption{\blue{BLEU scores on dataset $\mathtt{WMT16}$ from the ablation study of other different Beta$(\alpha, \beta)$ distributions of the transition time with 50 \BB{sampling steps}.}}
    \centering
    \resizebox{\columnwidth}{!}{
\begin{tabular}{c|c|cccccccccc}
\toprule
\multirow{2}{*}{Model}         & \multirow{2}{*}{Alpha} & \multicolumn{10}{c}{Beta}                                                                                                                                                                                                                                          \\ \cmidrule{3-12}  
                               &                               & \multicolumn{1}{c|}{3}     & \multicolumn{1}{c|}{5}     & \multicolumn{1}{c|}{7}     & \multicolumn{1}{c|}{9}     & \multicolumn{1}{c|}{11}    & \multicolumn{1}{c|}{13}    & \multicolumn{1}{c|}{15}    & \multicolumn{1}{c|}{17}    & \multicolumn{1}{c|}{19}    & 21    \\ \midrule
\multirow{3}{*}{DNDM-k-Multi}  & 3                             & \multicolumn{1}{c|}{33.31} & \multicolumn{1}{c|}{33.47} & \multicolumn{1}{c|}{33.39} & \multicolumn{1}{c|}{33.48} & \multicolumn{1}{c|}{33.29} & \multicolumn{1}{c|}{33.23} & \multicolumn{1}{c|}{33.25} & \multicolumn{1}{c|}{33.27} & \multicolumn{1}{c|}{33.11} & 33.17 \\
                               & 5                             & \multicolumn{1}{c|}{32.93} & \multicolumn{1}{c|}{33.28} & \multicolumn{1}{c|}{33.29} & \multicolumn{1}{c|}{$\mathbf{33.58}$} & \multicolumn{1}{c|}{33.45} & \multicolumn{1}{c|}{33.21} & \multicolumn{1}{c|}{33.40}  & \multicolumn{1}{c|}{33.49} & \multicolumn{1}{c|}{33.16} & 33.19 \\
                               & 7                             & \multicolumn{1}{c|}{32.61} & \multicolumn{1}{c|}{32.98} & \multicolumn{1}{c|}{33.31} & \multicolumn{1}{c|}{33.20}  & \multicolumn{1}{c|}{33.27} & \multicolumn{1}{c|}{33.41} & \multicolumn{1}{c|}{33.39} & \multicolumn{1}{c|}{33.53} & \multicolumn{1}{c|}{33.35} & 33.08 \\ \midrule
\multirow{3}{*}{DNDM-Multi}    & 3                             & \multicolumn{1}{c|}{32.63} & \multicolumn{1}{c|}{32.46} & \multicolumn{1}{c|}{32.44} & \multicolumn{1}{c|}{32.56} & \multicolumn{1}{c|}{32.59} & \multicolumn{1}{c|}{32.55} & \multicolumn{1}{c|}{32.37} & \multicolumn{1}{c|}{32.33} & \multicolumn{1}{c|}{32.22} & 32.23 \\
                               & 5                             & \multicolumn{1}{c|}{32.31} & \multicolumn{1}{c|}{32.43} & \multicolumn{1}{c|}{32.66} & \multicolumn{1}{c|}{32.64} & \multicolumn{1}{c|}{$\mathbf{32.68}$} & \multicolumn{1}{c|}{32.55} & \multicolumn{1}{c|}{32.55} & \multicolumn{1}{c|}{32.44} & \multicolumn{1}{c|}{32.35} & 32.30  \\
                               & 7                             & \multicolumn{1}{c|}{31.95} & \multicolumn{1}{c|}{32.11} & \multicolumn{1}{c|}{32.22} & \multicolumn{1}{c|}{32.26} & \multicolumn{1}{c|}{32.54} & \multicolumn{1}{c|}{32.52} & \multicolumn{1}{c|}{32.50}  & \multicolumn{1}{c|}{32.58} & \multicolumn{1}{c|}{32.48} & 32.41 \\ \midrule
\multirow{3}{*}{DNDM-k-Absorb} & 3                             & \multicolumn{1}{c|}{34.05} & \multicolumn{1}{c|}{34.2}  & \multicolumn{1}{c|}{34.31} & \multicolumn{1}{c|}{34.37} & \multicolumn{1}{c|}{34.15} & \multicolumn{1}{c|}{34.05} & \multicolumn{1}{c|}{34.06} & \multicolumn{1}{c|}{33.77} & \multicolumn{1}{c|}{33.81} & 33.84 \\
                               & 5                             & \multicolumn{1}{c|}{32.30}  & \multicolumn{1}{c|}{34.08} & \multicolumn{1}{c|}{34.30}  & \multicolumn{1}{c|}{$\mathbf{34.38}$} & \multicolumn{1}{c|}{34.26} & \multicolumn{1}{c|}{34.23} & \multicolumn{1}{c|}{34.09} & \multicolumn{1}{c|}{34.06} & \multicolumn{1}{c|}{34.02} & 34.13 \\
                               & 7                             & \multicolumn{1}{c|}{27.39} & \multicolumn{1}{c|}{32.64} & \multicolumn{1}{c|}{33.71} & \multicolumn{1}{c|}{34.18} & \multicolumn{1}{c|}{34.02} & \multicolumn{1}{c|}{34.33} & \multicolumn{1}{c|}{34.31} & \multicolumn{1}{c|}{34.17} & \multicolumn{1}{c|}{34.12} & 34.19 \\ \midrule
\multirow{3}{*}{DNDM-Absorb}   & 3                             & \multicolumn{1}{c|}{33.26} & \multicolumn{1}{c|}{33.30}  & \multicolumn{1}{c|}{33.29} & \multicolumn{1}{c|}{33.24} & \multicolumn{1}{c|}{33.23} & \multicolumn{1}{c|}{32.97} & \multicolumn{1}{c|}{33.06} & \multicolumn{1}{c|}{32.85} & \multicolumn{1}{c|}{32.89} & 32.63 \\
                               & 5                             & \multicolumn{1}{c|}{32.47} & \multicolumn{1}{c|}{33.08} & \multicolumn{1}{c|}{33.31} & \multicolumn{1}{c|}{33.22} & \multicolumn{1}{c|}{$\mathbf{33.41}$} & \multicolumn{1}{c|}{33.25} & \multicolumn{1}{c|}{33.15} & \multicolumn{1}{c|}{33.27} & \multicolumn{1}{c|}{33.04} & 32.98 \\
                               & 7                             & \multicolumn{1}{c|}{30.34} & \multicolumn{1}{c|}{32.27} & \multicolumn{1}{c|}{33.27} & \multicolumn{1}{c|}{33.03} & \multicolumn{1}{c|}{33.16} & \multicolumn{1}{c|}{33.14} & \multicolumn{1}{c|}{33.27} & \multicolumn{1}{c|}{33.11} & \multicolumn{1}{c|}{33.11} & 33.07 \\ \bottomrule
\end{tabular}    }
    \label{tab:ablation50}
\end{table}

\begin{table}[ht!]
  \caption{\blue{The BLEU scores on dataset $\mathtt{WMT16}$ 
 with Beta(100,4) as the transition time schedule for discrete sampling or the distribution to sample transition timestamps for continuous sampling.}}
\centering
\begin{tabular}{|c|l|l|l|l|}
\hline
Steps    & DNDM-k-multi & DNDM-k-absorb & DNDM-multi & DNDM-absorb \\ \hline
50       & 31.60         & 31.74         & 30.39      & 29.69      \\
1000     & 33.59        & 34.37         & 32.87      & 33.52      \\
$\infty $ & 33.86        & 34.41         & 32.91      & 33.42      \\ \hline
\end{tabular}    

    \label{tab:100_4}
\end{table}

\noindent\textbf{Continuous time vs discrete time diffusions. }
To test our hypothesis that the continuous-time sampler will produce more accurate results in reverse sampling if our $\xb_0$ estimator consistently approximates the true $\xb_0$ over time, we conduct various sampling experiments using a shared pre-trained neural network.
For discrete-time sampling, we consider three cases: $T=25, 50, 1000$. In each case, we rescale the interval $[0, T]$ to $[0, 50]$ and divide it into $T$ fractions. In contrast, for continuous-time sampling, we directly sample from a continuous distribution over the interval $[0, 50]$ without any partitioning.

\noindent\textbf{Training approach.} 
In machine translation tasks, the neural network is designed to learn $q(\xb_0|\xb_t, \zb)$, where $\zb$ represents the embedding of the source text obtained using transformer encoder layers. For a fair comparison, we employ the same neural network structure as our baseline, with detailed architecture specifications available in Section E.2 of~\citet{zheng2023reparameterized}.
Furthermore, given that the primary focus of this paper is the speed and effectiveness of our sampling algorithm, we omit the training procedure and instead use a state-of-the-art diffusion-based pretrained checkpoint from~\citet{zheng2023reparameterized}.
In the Appendix, we present additional results of continuous sampling based on a continuously trained checkpoint. In this setting, we rescale our network input to the interval $[0, 1]$ and uniformly sample from this interval. The rest of the architecture follows that of~\citet{zheng2023reparameterized}.

\blue{
\noindent\textbf{Performance on WMT14.} Our work primarily focuses on the sampling process, and for the training, we utilized a pretrained checkpoint trained on 50 steps. In our sampling experiments we noticed that our method does not work ideally on $\mathtt{WMT14}$, this could be possibly attributed to the fact that the training performance on $\mathtt{WMT14}$ was not ideal.
Specifically, when we performed sampling using 1000 steps, the network was trained with exposure to only 50 time steps, specifically at intervals of 20 (0, 20, 40, ..., 980, 1000). As a result, when we apply our model to generation using 1000 steps, the checkpoint NN has only been explicitly trained on these intervals. While we generally assume that the network can still provide a good estimate for the untrained steps, this might not hold under some hard scenarios.
Considering the longer training time and poorer performance of $\mathtt{WMT14}$, it is likely that the training performance is insufficient for us to rely on those unseen steps. In a word, the model's trained checkpoint may not be robust enough to effectively handle unseen steps, especially for timesteps 1000 or infinite timesteps.
}

\subsection{Unconditional Text Generation}

\textbf{Parameter choices.} We recover the checkpoints of the multinomial diffusion model employing the provided code by \citet{hoogeboom2021argmax}. We train 12-layer Transformers for both $\mathtt{text8}$ and enwik8 datasets for 500 epochs with the cosine schedule. For the $\mathtt{text8}$ dataset, we utilize a training batch size of 256, while for the $\mathtt{enwik8}$ dataset, we use a batch size of 128.  During training, we employ a learning rate of 0.0001, a weight decay parameter of 0.99, and the Adam optimizer.

\section{Additional Experiments}\label{sec:add_exp}
In this section, we present additional experimental results. We begin by plotting the relationship between computational time and the number of sampling steps, using the absorbing diffusion in $\mathtt{IWSLT14}$ as an example. Figure~\ref{fig:step-time} displays the growth of computational time for absorbing diffusion (yellow and orange lines), RDM-absorbing diffusion, and our model DNDM-Absorb and DNDM-T-Absorb (green and blue lines).
\begin{figure}[ht]
    \centering
    \includegraphics[width=1.0\textwidth]{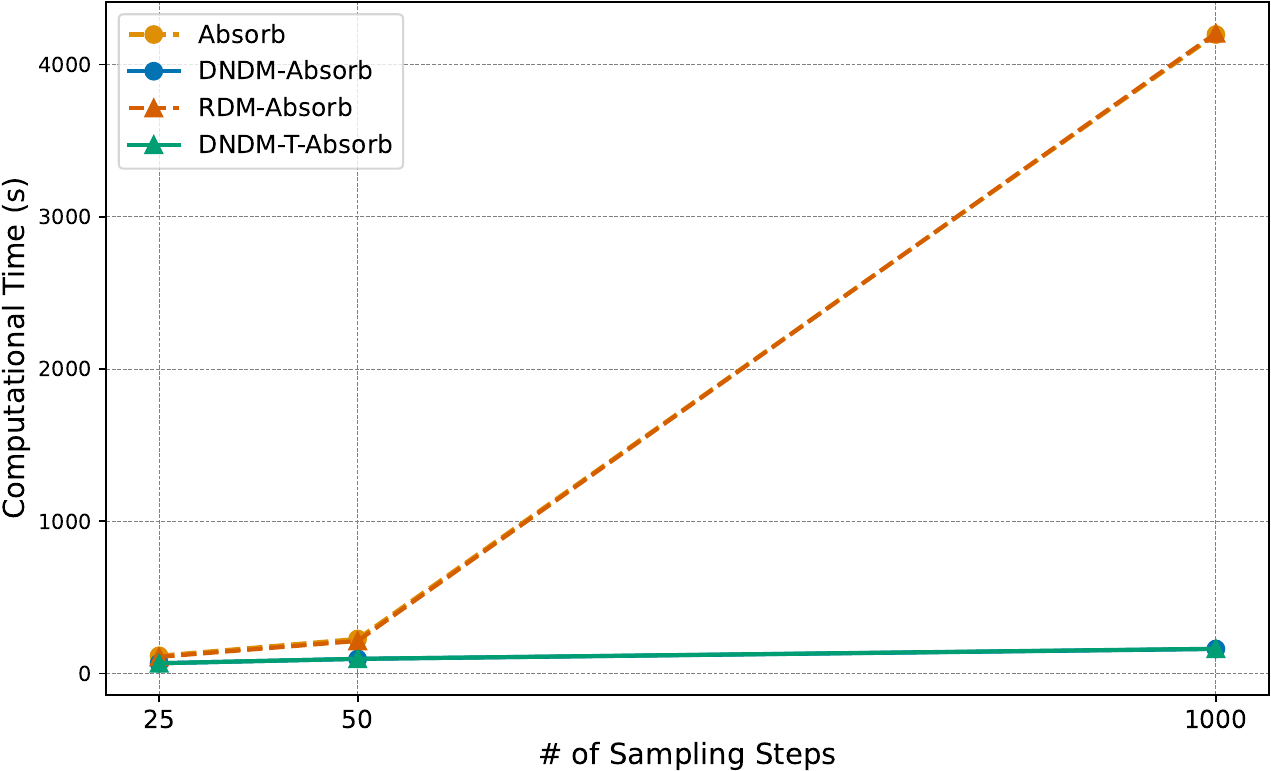} 
    \caption{The growth of computational time with the increase of the sampling steps}
    \label{fig:step-time}
\end{figure}
We see from Figure~\ref{fig:step-time} that previous algorithms, including absorbing diffusion and RDM-absorbing diffusion all suffer from linear growth of computational time.

\subsection{Continuous Training}\label{sec:cont}

In Section~\ref{sec:exp_text}, we introduce the DNDM-C algorithm, designed for continuous-time, over discrete-time algorithms. However, this algorithm assumes that we have learned a sufficiently accurate neural network at any timestamp $t \in [0, 1]$. Using the checkpoint trained with 50 discrete time partitions might not suffice for the purpose of continuous sampling.
In this section, we investigate the performance of continuous sampling when training is also done continuously.

\begin{table}[ht!]
    \centering
        \caption{Continuous Training + Continuous Sampling}
    \begin{tabular}{c|c|cccc}
    \toprule
    \multirow{2}{*}{Dataset} & \multirow{2}{*}{Step scheme} & \multicolumn{2}{c|}{C-DNDM-Multi} & \multicolumn{2}{c}{C-DNDM-Absorb} \\ \cmidrule{3-6}
                             &                              & \multicolumn{1}{c|}{Default} & Top-k & \multicolumn{1}{c|}{Default} & Top-k \\ \midrule
    \texttt{IWSLT14}         & Continuous                   & \multicolumn{1}{c|}{\textbf{32.07}} & \textbf{33.57} & \multicolumn{1}{c|}{32.80} & 34.52 \\ \midrule
    \texttt{WMT16}           & Continuous                   & \multicolumn{1}{c|}{\textbf{33.48}} & 33.71 & \multicolumn{1}{c|}{\textbf{33.50}} & 34.36 \\ \bottomrule
    \end{tabular}

    \label{tab:cont_and_cont}
\end{table}

In Table~\ref{tab:cont_and_cont}, we summarize the performance of DNDM-C based on a neural network estimated continuously during training time. This involves sampling time uniformly from $[0, 1]$ during training, and the forward process follows~\eqref{eq:continuous_forward} in Section~\ref{sec:continuous}. The training objective remains the same as in discrete-time training.
In Table~\ref{tab:cont_and_cont} we list the result of IWSLT14 and WMT16 with continuous training followed by continuous sampling.
In addition, we compare the value with the corresponding value during discrete training and continuous sampling in Section~\ref{sec:exp_text} and mark every item that improves in bold. As demonstrated in Table~\ref{tab:cont_and_cont}, there is room for enhancement in the overall sampling scores by training the neural network in a complete space of timestamps.

\blue{
\subsection{Comparison with more generative models}
}
\blue{
In our study, a key aspect of evaluating our fast discrete generative model involves comparisons with prior work known for speed in sampling with minimal steps. Specifically, we draw a direct comparison with the Mask-Predict~\citep{ghazvininejad2019mask}, which is notable for its ability to generate high-quality results within just 10 iterations. The results are shown in Table~\ref{tab:maskpredict_wmt16}. All experiments were conducted on the same GPU and within the same machine setup.
}

\begin{table}[ht!]
    \caption{\blue{The performance comparison on $\mathtt{WMT16}$ of DNDM with Mask-Predict~\citep{ghazvininejad2019mask}. We align the number of sampling steps used in Mask-Predict with a similar number of function evaluations (NFE) in our DNDM algorithm. We see that our Algorithm runs faster, with better BLEU score.}}
    \label{tab:maskpredict_wmt16}
    \centering
    \begin{tabular}{|c|c|c|c|c|c|c|c|c|c|c|c|}
    \hline  \multicolumn{3}{|c|}{Mask-Predict} &\multicolumn{4}{|c|}{DNDM-Absorb} & \multicolumn{4}{|c|}{DNDM-k-Absorb} \\
        \hhline{---------}
           \textbf{Steps} & \textbf{BLEU} & \textbf{Time} & \textbf{Steps} & \textbf{BLEU} & \textbf{Time} &\textbf{NFE} & \textbf{Steps} & \textbf{BLEU} & \textbf{Time}  & \textbf{NFE}
           \\\hline 10& 33.08& 49.25 & 25 & 33.20 & 41.2& 13.91 & 25& 33.96 & 41.4 & 13.91\\
        \hline  15 & 33.06 & 67.94 & 50 & 33.30 &  62.5 &20.95 &50 & 34.20 & 62.7 & 20.95\\
        \hline  25 & 33.16 & 111.89 & 1000 & 33.60 & 121.3 & 38.27 & 1000 & 34.38 & 122.7 & 38.27\\
        \hline 40 & 33.10 & 169.95 & $\infty$ & 33.42 & 121.8 &41.59 & $\infty$ & 34.41 & 121.9 & 41.59\\
        \hline
    \end{tabular}
\end{table}

\subsection{Samples from the multinomial text models}

\noindent \textbf{Conditional Generation.} For DNDM-Multi trained on $\mathtt{IWSLT14}$, we provide a full generation process with 100 steps in Figure~\ref{fig:iwslt14_generation}. A token ending with $\texttt{@@}$ indicates it is an incomplete word; it will be concatenated with the following token to form a complete word. For example, $``\texttt{fel@@ lo@@ ws}''$ means $``\texttt{fellows}''$.
We can see that after $t=39$, the generate sentence converges.
\begin{figure}[ht]
\centering
\blue{
\begin{minipage}{\linewidth}
\small{
\textbf{t = 100} \\
\textbf{[noise]} \textbf{[noise]} \textbf{[noise]} \textbf{[noise]} \textbf{[noise]} \textbf{[noise]} \textbf{[noise]} \textbf{[noise]} \textbf{[noise]} \textbf{[noise]} \textbf{[noise]} \textbf{[noise]} \textbf{[noise]} \textbf{[noise]} \textbf{[noise]} \textbf{[noise]} \textbf{[noise]} \textbf{[noise]} \textbf{[noise]} \textbf{[noise]} \\[3pt]
\textbf{t = 79} \\
\textbf{[noise]} \textbf{[noise]} \textbf{[noise]} \textbf{[noise]} \textbf{[noise]} \textbf{[noise]} \textbf{[noise]} \textbf{[noise]} \textbf{[noise]} \textbf{[noise]} \textbf{[noise]} \textbf{[noise]} \textbf{[noise]} \textbf{[noise]} \textbf{[noise]} \textbf{[noise]} \textbf{[noise]} \texttt{year} \textbf{[noise]} \\[3pt]
\textbf{t = 78} \\
\textbf{[noise]} \textbf{[noise]} \textbf{[noise]} \textbf{[noise]} \textbf{[noise]} \textbf{[noise]} \textbf{[noise]} \textbf{[noise]} \textbf{[noise]} \texttt{we} \textbf{[noise]} \textbf{[noise]} \textbf{[noise]} \textbf{[noise]} \textbf{[noise]} \textbf{[noise]} \textbf{[noise]} \textbf{[noise]} \texttt{year} \textbf{[noise]} \\[3pt]
\textbf{t = 77} \\
\textbf{[noise]} \textbf{[noise]} \textbf{[noise]} \textbf{[noise]} \textbf{[noise]} \textbf{[noise]} \textbf{[noise]} \textbf{[noise]} \texttt{and we} \textbf{[noise]} \textbf{[noise]} \textbf{[noise]} \textbf{[noise]} \textbf{[noise]} \textbf{[noise]} \textbf{[noise]} \textbf{[noise]} \texttt{year} \textbf{[noise]} \\[3pt]
\textbf{t = 75} \\
\textbf{[noise]} \textbf{[noise]} \textbf{[noise]} \textbf{[noise]} \textbf{[noise]} \textbf{[noise]} \textbf{[noise]} \textbf{[noise]} \texttt{and we} \textbf{[noise]} \textbf{[noise]} \textbf{[noise]} \textbf{[noise]} \texttt{govern@@} \textbf{[noise]} \textbf{[noise]} \texttt{year} \textbf{[noise]} \\[3pt]
\textbf{t = 74} \\
\texttt{we} \textbf{[noise]} \textbf{[noise]} \textbf{[noise]} \texttt{lo@@} \textbf{[noise]} \textbf{[noise]} \textbf{[noise]} \texttt{and we} \textbf{[noise]} \textbf{[noise]} \textbf{[noise]} \textbf{[noise]} \texttt{govern@@} \textbf{[noise]} \textbf{[noise]} \texttt{year} \textbf{[noise]} \\[3pt]
\textbf{t = 73} \\
\texttt{we} \textbf{[noise]} \textbf{[noise]} \texttt{fel@@ lo@@} \textbf{[noise]} \textbf{[noise]} \textbf{[noise]} \texttt{and we let} \textbf{[noise]} \textbf{[noise]} \textbf{[noise]} \textbf{[noise]} \texttt{govern@@} \textbf{[noise]} \textbf{[noise]} \texttt{year} \textbf{[noise]} \\[3pt]
\textbf{t = 71} \\
\texttt{we} \textbf{[noise]} \textbf{[noise]} \texttt{fel@@ lo@@} \textbf{[noise]} \textbf{[noise]} \textbf{[noise]} \texttt{and we let} \textbf{[noise]} \textbf{[noise]} \textbf{[noise]} \textbf{[noise]} \texttt{govern@@} \textbf{[noise]} \texttt{every year} \textbf{[noise]} \\[3pt]
\textbf{t = 67} \\
\texttt{we} \textbf{[noise]} \textbf{[noise]} \texttt{fel@@ lo@@} \textbf{[noise]} \textbf{[noise]} \textbf{[noise]} \texttt{and we let them} \textbf{[noise]} \textbf{[noise]} \texttt{city govern@@} \textbf{[noise]} \texttt{every year} . \\[3pt]
\textbf{t = 66} \\
\texttt{we} \textbf{[noise]} \textbf{[noise]} \texttt{fel@@ lo@@ ws} \textbf{[noise]} \textbf{[noise]} \texttt{and we let them work} \textbf{[noise]} \texttt{city govern@@} \textbf{[noise]} \texttt{every year} . \\[3pt]
\textbf{t = 64} \\
\texttt{we} \textbf{[noise]} \textbf{[noise]} \texttt{fel@@ lo@@ ws} \textbf{[noise]} \textbf{[noise]} \texttt{and we let them work} \textbf{[noise]} \texttt{city govern@@ ance every year} . \\[3pt]
\textbf{t = 61} \\
\texttt{we} \textbf{[noise]} \textbf{[noise]} \texttt{fel@@ lo@@ ws} \textbf{[noise]} \textbf{[noise]} \texttt{and we let them work with city govern@@ ance every year} . \\[3pt]
\textbf{t = 60} \\
\texttt{we} \textbf{[noise]} \textbf{[noise]} \texttt{fel@@ lo@@ ws} \textbf{[noise]} \texttt{year and we let them work with city govern@@ ance every year} . \\[3pt]
\textbf{t = 58} \\
\texttt{we} \textbf{[noise]} \textbf{[noise]} \texttt{fel@@ lo@@ ws every year and we let them work with city govern@@ ance every year} . \\[3pt]
\textbf{t = 52} \\
\texttt{we} \textbf{[noise]} \texttt{some fel@@ lo@@ ws every year and we let them work with city govern@@ ance every year} . \\[3pt]
\textbf{t = 39} \\
\texttt{we choose some fel@@ lo@@ ws every year and we let them work with city governance every year.} \\[3pt]
\textbf{t = 0} \\
\texttt{we choose some fel@@ lo@@ ws every year and we let them work with city governance every year.} \\
}
\end{minipage}}
\caption{Text in the Generation Process}
\label{fig:iwslt14_generation}
\end{figure}


\newpage

\section*{NeurIPS Paper Checklist}

\begin{enumerate}

\item {\bf Claims}
    \item[] Question: Do the main claims made in the abstract and introduction accurately reflect the paper's contributions and scope?
    \item[] Answer: \answerYes{} 
    \item[] Justification: The contributions are summarized as three points at the end of the introduction. The scope is fast sampling via discrete non-Markov diffusion
models, provided in the abstract.

\item {\bf Limitations}
    \item[] Question: Does the paper discuss the limitations of the work performed by the authors?
    \item[] Answer: \answerYes{} 
    \item[] Justification: We add a limitation section in front of the Appendix.

\item {\bf Theory Assumptions and Proofs}
    \item[] Question: For each theoretical result, does the paper provide the full set of assumptions and a complete (and correct) proof?
    \item[] Answer: \answerYes{} 
    \item[] Justification: Theorems 3.1, 3.5, and D.1 are clearly stated, well-organized with consistent numbering, and supported by rigorous proofs that establish their validity.

    \item {\bf Experimental Result Reproducibility}
    \item[] Question: Does the paper fully disclose all the information needed to reproduce the main experimental results of the paper to the extent that it affects the main claims and/or conclusions of the paper (regardless of whether the code and data are provided or not)?
    \item[] Answer: \answerYes{} 
    \item[] Justification: The paper provides detailed information on the experimental setup, model architecture, and training procedures. The authors have submitted their training code along with the main paper, which enables reproducibility of the main results. The code and detailed instructions allow other researchers to replicate the key findings of the paper. 
    \item[] Guidelines: In addition to experiment and implementation details on appendix, we submit our training and evaluation codes when submtting our main paper.
    \begin{itemize}
        \item The answer NA means that the paper does not include experiments.
        \item If the paper includes experiments, a No answer to this question will not be perceived well by the reviewers: Making the paper reproducible is important, regardless of whether the code and data are provided or not.
        \item If the contribution is a dataset and/or model, the authors should describe the steps taken to make their results reproducible or verifiable. 
        \item Depending on the contribution, reproducibility can be accomplished in various ways. For example, if the contribution is a novel architecture, describing the architecture fully might suffice, or if the contribution is a specific model and empirical evaluation, it may be necessary to either make it possible for others to replicate the model with the same dataset, or provide access to the model. In general. releasing code and data is often one good way to accomplish this, but reproducibility can also be provided via detailed instructions for how to replicate the results, access to a hosted model (e.g., in the case of a large language model), releasing of a model checkpoint, or other means that are appropriate to the research performed.
        \item While NeurIPS does not require releasing code, the conference does require all submissions to provide some reasonable avenue for reproducibility, which may depend on the nature of the contribution. For example
        \begin{enumerate}
            \item If the contribution is primarily a new algorithm, the paper should make it clear how to reproduce that algorithm.
            \item If the contribution is primarily a new model architecture, the paper should describe the architecture clearly and fully.
            \item If the contribution is a new model (e.g., a large language model), then there should either be a way to access this model for reproducing the results or a way to reproduce the model (e.g., with an open-source dataset or instructions for how to construct the dataset).
            \item We recognize that reproducibility may be tricky in some cases, in which case authors are welcome to describe the particular way they provide for reproducibility. In the case of closed-source models, it may be that access to the model is limited in some way (e.g., to registered users), but it should be possible for other researchers to have some path to reproducing or verifying the results.
        \end{enumerate}
    \end{itemize}

\item {\bf Open access to data and code}
    \item[] Question: Does the paper provide open access to the data and code, with sufficient instructions to faithfully reproduce the main experimental results, as described in supplemental material?
    \item[] Answer: \answerYes{} 
    \item[] Justification: All the datasets are public and can be open accessed. Our codebase will be available in public upon acceptance. 
    

\item {\bf Experimental Setting/Details}
    \item[] Question: Does the paper specify all the training and test details (e.g., data splits, hyperparameters, how they were chosen, type of optimizer, etc.) necessary to understand the results?
    \item[] Answer: \answerYes{} 
    \item[] Justification: We provide these details on Appendix (D, E, F).

\item {\bf Experiment Statistical Significance}
    \item[] Question: Does the paper report error bars suitably and correctly defined or other appropriate information about the statistical significance of the experiments?
    \item[] Answer: \answerYes{} 
    \item[] Justification: Confidence intervals are provided in the experiments.

\item {\bf Experiments Compute Resources}
    \item[] Question: For each experiment, does the paper provide sufficient information on the computer resources (type of compute workers, memory, time of execution) needed to reproduce the experiments?
    \item[] Answer: \answerYes{} 
    \item[] Justification: We provided detailed information about the computation resources in Section 4: a single NVIDIA258
RTX A6000 GPU with 48 GB memory.
    
\item {\bf Code Of Ethics}
    \item[] Question: Does the research conducted in the paper conform, in every respect, with the NeurIPS Code of Ethics \url{https://neurips.cc/public/EthicsGuidelines}?
    \item[] Answer: \answerYes{} 
    \item[] Justification: We have checked NeurIPS Code of Ethics. Our submission satisfies all the requirement.

\item {\bf Broader Impacts}
    \item[] Question: Does the paper discuss both potential positive societal impacts and negative societal impacts of the work performed?
    \item[] Answer: \answerYes{} 
    \item[] Justification: We provide Broader Impacts Section in the beginning of Appendix.
    
\item {\bf Safeguards}
    \item[] Question: Does the paper describe safeguards that have been put in place for responsible release of data or models that have a high risk for misuse (e.g., pretrained language models, image generators, or scraped datasets)?
    \item[] Answer: \answerNA{} 
    \item[] Justification: The paper poses no related risks.

\item {\bf Licenses for existing assets}
    \item[] Question: Are the creators or original owners of assets (e.g., code, data, models), used in the paper, properly credited and are the license and terms of use explicitly mentioned and properly respected?
    \item[] Answer: \answerYes{} 
    \item[] Justification: All used code, data and models in this project are properly cited.

\item {\bf New Assets}
    \item[] Question: Are new assets introduced in the paper well documented and is the documentation provided alongside the assets?
    \item[] Answer: \answerNA{} 
    \item[] Justification: The paper does not release new assets.

\item {\bf Crowdsourcing and Research with Human Subjects}
    \item[] Question: For crowdsourcing experiments and research with human subjects, does the paper include the full text of instructions given to participants and screenshots, if applicable, as well as details about compensation (if any)? 
    \item[] Answer: \answerNA{} 
    \item[] Justification: The paper does not involve crowdsourcing nor research with human subjects.

\item {\bf Institutional Review Board (IRB) Approvals or Equivalent for Research with Human Subjects}
    \item[] Question: Does the paper describe potential risks incurred by study participants, whether such risks were disclosed to the subjects, and whether Institutional Review Board (IRB) approvals (or an equivalent approval/review based on the requirements of your country or institution) were obtained?
    \item[] Answer: \answerNA{} 
    \item[] Justification: The paper does not involve crowdsourcing nor research with human subjects.

\end{enumerate}

\end{document}